\documentclass[lettersize,journal]{IEEEtran}
\usepackage{amsmath,amsfonts}
\usepackage{amsthm}
\usepackage{algorithmic}
\usepackage{algorithm}
\usepackage{array}
\usepackage[caption=false,font=normalsize,labelfont=sf,textfont=sf]{subfig} 
\usepackage{textcomp}
\usepackage{stfloats}
\usepackage{url}
\usepackage{verbatim}
\usepackage{graphicx}
\usepackage{cite}
\usepackage{xcolor}
\usepackage{tikz}
\usetikzlibrary{positioning}

\usepackage{witharrows}
\usepackage{hyperref}       % hyperlinks
\usepackage{url}            % simple URL typesetting
\usepackage{booktabs}       % professional-quality tables
\usepackage{framed}
\usepackage{wrapfig}
\usepackage{url}
\usepackage{booktabs} % for professitemplateonal tables
\usepackage[shortlabels]{enumitem}
\setlist[enumerate]{nosep}
\newtheorem{theorem}{Theorem}

\newtheorem{assumption}{Assumption}[section]

\newtheorem{proposition}{Proposition}
\newtheorem{lemma}{Lemma}
\newtheorem{corollary}{Corollary}%[section]
\usepackage{amsfonts,amssymb}
\usepackage{multirow}

\usepackage{bm}
\usepackage{algorithm}
\usepackage{algorithmic}

\usepackage{tikz}
\usetikzlibrary{matrix}
\usepackage{enumitem}
\usepackage{xr}
\newtheorem{remark}{Remark}
\usepackage{amsmath}
\newcommand{\N}{\mathcal{N}}
\newcommand{\s}{\Sigma}
\newcommand{\e}{\epsilon}
\DeclareMathOperator{\Tr}{Tr}
\newcommand{\norm}[1]{\left\lVert #1 \right\rVert}

\begin{document}
	
	\title{Worst-Case Optimal $\sqrt{\epsilon}$-Rate Stability of KL Divergence under Perturbations of Gaussian Distributions}

	\author{Jialu Pan, Yufeng Zhang, Nan Hu, Zhenbang Chen, Ji Wang, Keqin Li,~\IEEEmembership{Fellow,~IEEE,}
		% <-this % stops a space
		\thanks{Jialu Pan, Yufeng Zhang,  and Nan Hu are with the College of Computer Science and Electronic Engineering, Hunan University, Changsha, China.\protect\\
		E-mail: jialupan@hnu.edu.cn, yufengzhang@hnu.edu.cn, hunan5@hnu.edu.cn. }% <-this % stops a space
		\thanks{Zhenbang Chen and Ji Wang are with the College of Computer, National University of Defense Technology, Changsha, China.\protect\\
		E-mail: zbchen@nudt.edu.cn, wj@nudt.edu.cn.}
		\thanks{Keqin Li is with the Department of Computer Science, State University of New York, New Paltz, New York, USA.\protect\\
		E-mail: lik@newpaltz.edu.}
		\thanks{Jialu Pan and Yufeng Zhang contributed equally to this work and are co-first authors. Yufeng Zhang is the corresponding author.}
	}

	\markboth{Journal of \LaTeX\ Class Files,~Vol.~14, No.~8, August~2021}%
	{Shell \MakeLowercase{\textit{et al.}}: A Sample Article Using IEEEtran.cls for IEEE Journals}
	
	\IEEEpubid{0000--0000/00\$00.00~\copyright~2021 IEEE}
	
	\maketitle
	\hypersetup{hidelinks}
	\begin{abstract}
		We establish a stability bound for the Kullback–Leibler (KL) divergence when its Gaussian reference distribution is perturbed. Existing relaxed triangle inequalities of KL divergence {\color{black}between Gaussian distributions} require all three distributions to be Gaussian, which limits their use in applications such as out-of-distribution detection with flow-based generative models. Let $\mathcal{N}_1$ and $\mathcal{N}_2$ be $d$-dimensional Gaussian distributions, 
			and let $P$ be any $d$-dimensional distribution satisfying $\mathbb{E}_P[\|X\|^2]<\infty$ and $\mathrm{KL}(P\Vert \mathcal{N}_1)<\infty$.
			{\color{black}If $\mathrm{KL}(\mathcal{N}_1\Vert\mathcal{N}_2)\leq\epsilon$ for $0<\epsilon<1/12$, we give an upper bound of $\left|\mathrm{KL}(P\Vert\mathcal{N}_1)-\mathrm{KL}(P\Vert\mathcal{N}_2)\right|$. Under suitable uniform boundedness conditions, the bound is $ \frac{\sqrt{6d\epsilon}}{2}+K\sqrt{\epsilon}+o(\sqrt{\epsilon})$ as $\epsilon\rightarrow 0$ for fixed $d$,
			where $K$ depends on the parameters of $\mathcal{N}_1$ and $\mathcal{N}_2$, and on $\mathbb{E}_P[\|X\|]$ and $\mathbb{E}_P[\|X\|^2]$.} Furthermore, we prove that the $\sqrt{\epsilon}$ dependence is worst-case optimal. The result extends stability guarantees restricted to Gaussian triples to comparisons involving an arbitrary distribution and two nearby Gaussian references. Experiments on synthetic distributions and latent representations from flow-based models support the predicted stability behavior.
		
		%	
		%	We study the stability of Kullback--Leibler (KL) divergence under perturbations of a Gaussian reference distribution.
		%	Existing relaxed triangle inequalities for KL divergence require all participating distributions to be Gaussian,
		%	which limits their applicability to settings such as out-of-distribution detection with flow-based generative models.
		%	We relax this assumption by allowing one of the three distributions to be arbitrary and derive a stability bound under a mild moment condition.
		%	Specifically, {\color{black}let $\mathcal{N}_1$ and $\mathcal{N}_2$ be $d$-dimensional Gaussian distributions, 
			%	and let $P$ be any $d$-dimensional distribution satisfying $\mathbb{E}_P[\|X\|^2]<\infty$ and $\mathrm{KL}(P\Vert \mathcal{N}_1)<\infty$.
			%	If $\mathrm{KL}(\mathcal{N}_1\Vert\mathcal{N}_2)\leq\epsilon$ for $0<\epsilon<1/12$, then
			%	$\left|\mathrm{KL}(P\Vert\mathcal{N}_1)-\mathrm{KL}(P\Vert\mathcal{N}_2)\right|
			%	\leq \frac{\sqrt{6d\epsilon}}{2}+K\sqrt{\epsilon}+o(\sqrt{\epsilon})$,
			%	where $K$ depends on the parameters of $\mathcal{N}_1$ and $\mathcal{N}_2$ and on the first two norm moments of $P$.}
		%	We further prove that the $\sqrt{\epsilon}$ dependence is optimal in the worst-case sense.
		%	Our result therefore extends Gaussian-only stability theory to KL comparisons involving an arbitrary distribution 
		%	and nearby Gaussian reference distributions.
		%	Numerical experiments on synthetic distributions and flow-based model latent representations provide empirical support for the theoretical bound.
	\end{abstract}


	\section{Introduction}
	\IEEEPARstart{K}{ullback--Leibler} (KL) divergence, also known as relative entropy~\cite{kullback1997information}, is a widely used measure of discrepancy between probability distributions. It has broad applications in deep learning~\cite{goodfellow2016deep}, statistics~\cite{Statistical_Inference_divergence}, and information theory~\cite{cover2012elements}. However, KL divergence is asymmetric and does not satisfy the triangle inequality. 
	These limitations complicate KL-based reasoning about relationships among multiple distributions.

	Previous studies have characterized KL divergence for several classes of probability distributions, including multivariate Gaussian distributions~\cite{zhang2023properties}, lattice Gaussian distributions~\cite{nielsen2022kullback}, Gaussian mixture models~\cite{LowerUpperBoundsKLGMM,hershey2007approximating}, and Markov sources~\cite{rached2004kullback}. In particular, Zhang \textit{et al.} have shown that KL divergence between multivariate Gaussian distributions is approximately symmetric and satisfies a relaxed triangle inequality~\cite{zhang2023properties}. Specifically, let $\mathcal{N}_1$, $\mathcal{N}_2$, and $\mathcal{N}_3$ be three $n$-dimensional Gaussian distributions. If $\mathrm{KL}(\mathcal{N}_1\Vert \mathcal{N}_2)\leq \epsilon_1$ and $\mathrm{KL}(\mathcal{N}_2\Vert \mathcal{N}_3)\leq \epsilon_2$ for $\epsilon_1,\epsilon_2\geq 0$, then $\mathrm{KL}(\mathcal{N}_1\Vert \mathcal{N}_3)<3\epsilon_1+3\epsilon_2+2\sqrt{\epsilon_1\epsilon_2}+o(\epsilon_1)+o(\epsilon_2)$. More recently, Xiao \textit{et al.} further sharpened this bound to $\epsilon_1+\epsilon_2+2\sqrt{\epsilon_1\epsilon_2}+o(\epsilon_1)+o(\epsilon_2)$~\cite{xiao2026relaxedtriangleinequalitykullbackleibler}.

	The theoretical results in \cite{zhang2023properties,xiao2026relaxedtriangleinequalitykullbackleibler} have applications in several areas, including out-of-distribution detection, safe reinforcement learning~\cite{liu2022constrained}, and sample-complexity analysis~\cite{JACM-sample-complexity-GMM2020}.
	In \cite{zhang2023outofdistribution}, Zhang \textit{et al.} analyzed the KL divergences among several distributions induced by a flow-based model~\cite{glowopenai}. Using established properties of KL divergence between Gaussian distributions, they explained a counterintuitive phenomenon: a flow-based model may assign comparable or even higher likelihoods to out-of-distribution (OOD) data, yet fail to generate OOD-like samples. The explanation relies on the following key conclusion:
	
	\textit{The KL divergence between the distribution of OOD representations and the prior is greater than that between the distribution of ID representations and the prior.}
	
	Building on this conclusion, Zhang \textit{et al.} developed KLODS, a KL divergence-based OOD detection algorithm. Experimental results demonstrate the effectiveness of the proposed method.
	\IEEEpubidadjcol

	In prior work~\cite{zhang2023outofdistribution}, Zhang \textit{et al.} derived the above conclusion by considering two cases. 
	When the OOD representations follow an approximately Gaussian distribution, they conducted a theoretical analysis based on the approximate symmetry and relaxed triangle inequality of KL divergence between Gaussian distributions \cite{zhang2023properties}. However, this analysis does not apply when the distribution of OOD representations is not assumed to be Gaussian. To address this general case, they imposed the strong assumption that $P_Z\approx P_Z^r$, where $P_Z$ denotes the distribution of in-distribution (ID) representations and $P_Z^r$ denotes the model prior. This assumption is motivated by the training objective of flow-based models, which minimizes $\mathrm{KL}(P_Z\Vert P_Z^r)$. Nevertheless, whether the model-fitting error can be safely neglected remains unclear.

	\textbf{Research Question.} Existing relaxed triangle inequalities for KL divergence {\color{black}between Gaussian distributions} assume that all distributions involved are Gaussian distributions. This assumption fails in important applications such as OOD detection with flow-based models, where data distributions can be highly non-Gaussian. 
	This limitation motivates the following question:
	\textit{{\color{black}Given two Gaussian distributions $\N_1$ and $\N_2$ and an arbitrary distribution $P$, can we bound $|\mathrm{KL}(P\Vert \N_1)-\mathrm{KL}(P\Vert \N_2)|$ when $\mathrm{KL}(\N_1\Vert \N_2)$ is small?}}
	%\textit{{\color{black}Can relaxed triangle inequalities for KL divergence be extended to a setting involving an arbitrary distribution and two Gaussian distributions?}}
	
	We address this question by establishing the optimal $\sqrt{\e}$-rate stability of KL divergence under Gaussian perturbations. 
	Let $\mathcal{N}_1$ and $\mathcal{N}_2$ be $d$-dimensional Gaussian distributions, 
	and let $P$ be any $d$-dimensional distribution satisfying $\mathbb{E}_P[\|X\|^2]<\infty$ and $\mathrm{KL}(P\Vert \mathcal{N}_1)<\infty$.
	{\color{black}If $\mathrm{KL}(\N_1\Vert \N_2)\leq \e$ for $0<\e<1/12$, we give an explicit upper bound $B(\e)$ on  
		$\left|\mathrm{KL}(P\Vert \N_1)-\mathrm{KL}(P\Vert \N_2)\right|$, where $B(\e)$ depends on the parameters of $\N_1$ and $\N_2$ and on $\mathbb{E}_P[\|X\|]$ and $\mathbb{E}_P[\|X\|^2]$. We also give the asymptotic form of $B(\e)$ as $\e\rightarrow 0$ and further show that the $\sqrt{\e}$ dependence is optimal in general.}

	Our bound connects arbitrary distributions with Gaussian families by revealing the stability of KL divergence under Gaussian perturbations. This connection is nontrivial because KL divergence is asymmetric and does not satisfy the triangle inequality. Moreover, our tightness analysis shows that the $\sqrt{\epsilon}$ rate is not caused by the non-Gaussianity of $P$. Instead, it reflects an intrinsic geometric property of KL divergence.
	
	{\color{black}This work establishes an explicit bound on the change in KL divergence from an arbitrary distribution to two nearby Gaussian reference distributions, under the stated moment and finiteness assumptions. It overcomes the restriction in existing studies on the properties of KL divergence between Gaussian distributions \cite{zhang2023properties,xiao2026relaxedtriangleinequalitykullbackleibler} that all distributions involved must be Gaussian distributions.}
	Our result is particularly relevant to the OOD detection setting studied in \cite{zhang2023outofdistribution}, where the distribution of OOD representations \(Q_Z\) need not be Gaussian or close to Gaussian. Under a controlled Gaussian-proxy error and sufficient initial ID–OOD separation, our bound guarantees that \(\mathrm{KL}(Q_Z\Vert P_Z^r)\) remains large, where \(P_Z^r\) denotes the latent prior distribution.\footnote{A corresponding guarantee for the reverse KL divergence \(\mathrm{KL}(P_Z^r\Vert Q_Z)\) requires additional assumptions and is beyond the scope of this work.} More broadly, the result provides a theoretical tool for analyzing the sensitivity of KL divergence to Gaussian reference perturbations in non-Gaussian settings, including applications in deep learning and reinforcement learning.

	%Exsiting research \cite{zhang2023outofdistribution,xiao2026relaxedtriangleinequalitykullbackleibler} reveal that KL divergence between Gaussian distributions follows approximate symmetry and relaxed triangle inequality. However, practical applications such as flow-based OOD detection needs a stronger conclusion where one of the distribution can be non-Gaussian.
	%The theoretical results of this work bridges the gap between the theoretical results on the properties of KL divergence and practical non-Gaussian settings. 
	%In the OOD detection problem studied in \cite{zhang2023outofdistribution}, 
	%the distribution of OOD representations $Q_Z$ may be arbitrary. The theorem in this paper can be used to show that the $\mathrm{KL}(Q_Z\Vert P_Z^r)$ remains large \footnote{A theoretical guarantee for the reverse KL divergence $\mathrm{KL}(P_Z^r\Vert Q_Z)$ needs additional assumption and is beyond the scope of this work.}. 
	%The result also provides theoretical support for KL-based analyses in broader non-Gaussian settings, including applications in deep learning and reinforcement learning.

	The contributions of this paper are as follows.
	\begin{enumerate}
		\item We establish a stability bound for KL divergence involving two Gaussian distributions, $\N_1$ and $\N_2$, 
		and an arbitrary distribution $P$ satisfying $\mathbb{E}_P[\|X\|^2]<\infty$ and $\mathrm{KL}(P\Vert \mathcal{N}_1)<\infty$. 
		Moreover, we derive several corollaries, including one for the case in which $\N_2$ is a standard Gaussian distribution 
		and another for multiple perturbation steps.
		These results reveal an intrinsic stability property of KL divergence 
		that extends beyond fully Gaussian settings.
		
		{\color{black}
		\item We analyze the $\sqrt{\e}$ and dimension dependence of the bound, where $\epsilon$ is the perturbation size. We also prove that the $\sqrt{\e}$ dependence of the bound is optimal in general. 
		}
	
		\item We conduct experiments on synthetic data and real-world latent representations generated by flow-based models.
		The results provide empirical support for the predicted stability behavior of the main theorem.
		
	\end{enumerate}

	The remainder of this paper is organized as follows. 
	Section~\ref{relatedwork} reviews related work. 
	Section~\ref{problem} introduces the necessary background and formulates the research problem. 
	Section~\ref{theory} develops the theoretical results. 
	Section~\ref{sec:exp} reports experiments that illustrate the bound in the main theorem. 
	Finally, Section~\ref{conclusion} concludes the paper.
	Additional results and discussion can be found in the supplementary material.
	
	\section{Related Work}\label{relatedwork}	
	\textbf{Divergence.} KL divergence has broad applications in machine learning~\cite{PRML,goodfellow2016deep,opti-reinforce-KL2010,MDP-KL-cost2014,distance_balance_KL2019}, information theory~\cite{cover2012elements}, and statistics~\cite{Statistical_Inference_divergence}. Its properties have been studied for many classes of probability distributions~\cite{rached2004kullback,LowerUpperBoundsKLGMM,hershey2007approximating,KLD_MGGD,nielsen2021kullback}.
	KL divergence belongs to several broader divergence families. These families include Bregman divergences~\cite{bregman1967relaxation,cluster-bregman-div,functional_bregman_divergence}, $f$-divergences~\cite{ali1966general,JMLR-bound-f-divergence-2021}, R\'{e}nyi divergences~\cite{renyi1961measures}, $\beta$-divergences~\cite{betadiv2021}, and $(f,\Gamma)$-divergences~\cite{f-gamma-divergence}.
	%Bregman divergence defines a class of divergences in vector spaces \cite{cluster-bregman-div}. 
	In particular, KL divergence between multinomial distributions is a special case of the Bregman divergence when the convex function is chosen as $\sum_{i=1}^{n}p_i\log p_i$, where $\bm{p}=(p_1,\ldots,p_n)$ is a probability vector. Similarly, KL divergence between continuous probability distributions can be interpreted as a functional Bregman divergence~\cite{functional_bregman_divergence}.
	
	\textbf{Stability Results of Related Divergences}.
	Several related divergences possess stronger metric or geometric properties. For example, the square root of the Jensen--Shannon divergence defines a metric~\cite{endres2003new}.
	Wasserstein distance is also a metric and therefore satisfies the triangle inequality~\cite{villani2009optimal}. 
	Broader families, such as $f$-divergences and Bregman divergences, 
	have been studied through continuity, projection, and Pythagorean-type properties~\cite{ali1966general,cluster-bregman-div}.
	{\color{black}The R\'enyi weak triangle inequality bounds the divergence between two distributions through an intermediate distribution but introduces divergences of higher orders \cite{2017Renyi_differential}.
	}
	Although these results provide important insights into divergence stability, they do not directly yield a triangle-type bound for KL divergence in this paper.
	
	Most existing results compare KL divergence with other discrepancy measures under additional assumptions. 
	They do not directly bound the change in $\mathrm{KL}(P\Vert \N)$ when the reference distribution $\N$ is perturbed. 
	For example, Pinsker's inequality relates KL divergence to total variation distance~\cite{DBLP:books/cu/CsiszarK11}. 
	Transport-entropy inequalities, such as Talagrand's inequality, relate KL divergence to Wasserstein distance 
	when the reference measure satisfies suitable regularity conditions~\cite{otto2000generalization}. 
	{\color{black}Altschuler \textit{et al.} \cite{altschuler2026shifted} employ a weak triangle inequality involving both KL and $\chi^2$ divergences, which follows from the R\'enyi weak triangle inequality \cite{2017Renyi_differential}.
	Kaji \cite{kaji2026hellinger} characterizes conditions under which the Hellinger distance controls KL divergence and related log-likelihood-ratio moments. }
	
	Recent relaxed triangle inequalities for multivariate Gaussian distributions are more closely related to our setting, 
	but they require all distributions involved to be Gaussian~\cite{zhang2023properties,xiao2026relaxedtriangleinequalitykullbackleibler}. 
	We relax this requirement by allowing $P$ to be arbitrary under a finite second-moment condition. 
	Further discussion of related stability results is provided in Appendix~\ref{sec:appendix-stability-results} of the supplementary material.
	
	\textbf{OOD Detection.} This work is motivated by the theoretical analysis presented in a prior study of OOD detection~\cite{zhang2023outofdistribution}. In separate work, Pan \textit{et al.} also develop an OOD detection method~\cite{PAN2026134478}. In contrast, the present work is primarily theoretical. Its results strengthen the foundation of existing KL-based analyses of OOD detection rather than introducing a new detection algorithm. Accordingly, the experiments are designed to empirically evaluate the proposed bound rather than benchmark OOD detectors.
	
	% \textbf{OOD Detection}.
	% Recent years have witnessed significant progress in deep learning-based OOD detection methods \cite{GeneralizedOOD_Survey_2024, anomaly_detection_survey_2009,deep_anomaly_detection_survey_2021, GAN_anomaly_survey_2020, anomaly_cps_2022}. 
	% In particular, flow-based generative models, which provide exact likelihood estimates, appear to offer a natural solution for OOD detection by using the model likelihood $p(\bm{x})$ as a detection score \cite{GADSurvey2018}. 
	% However, it has been widely observed that likelihood values from such models are often unreliable for distinguishing in-distribution (ID) and OOD data \cite{nalisnick2019detecting}. 
	% In \cite{zhang2023outofdistribution}, Zhang \textit{et al.}  conducted an theoretical analysis and then proposed using groupwise KL divergence from the last scale of flow-based model to detect OOD data. Their analysis relies on a strong assumption when OOD representation do not follow Gaussian-like distribution. The theoretical results presented in this work can provide a rigorous foundation for their analysis and algorithm in a restricted case.

	\section{Background and Research Problem}\label{problem}
	
	Flow-based models are deep generative models that learn explicit and tractable probability densities through invertible mappings between the data and latent spaces. They are typically trained using maximum likelihood estimation (MLE), which is equivalent to minimizing $\mathrm{KL}(P_X\Vert p_{\theta})$~\cite{flow-tpami-2021}. Here, $P_X$ denotes the data distribution, and $p_{\theta}$ is the model density parameterized by $\theta$. Further background is provided in Appendix~\ref{sec:background} of the supplementary material.
	
	Recent studies~\cite{kingma2018glow, chen2019residual, zhang2023outofdistribution} have shown that flow-based models
	often fail to distinguish OOD data reliably from ID data using model likelihood alone. 
	For example, a Glow model~\cite{kingma2018glow} trained on CIFAR-10 assigns higher likelihoods to SVHN samples than to 
	CIFAR-10 training samples~\cite{zhang2023outofdistribution} (see Fig.~\ref{fig:logpx_cifar10_svhn_contrast_glow} in the supplementary material). 
	However, the model cannot generate OOD-like samples. 
	This counterintuitive behavior poses a significant challenge to OOD detection with flow-based models.
	
	In \cite{zhang2023outofdistribution}, Zhang \textit{et al.} conducted a KL-based analysis and proposed an OOD detection algorithm.
	The theoretical explanation is summarized below.
	
	Let $f$ denote the invertible flow transformation and let $z=f(x)$ be the latent representation. Let $P_Z^r$ be the standard Gaussian prior. Let $X_1\sim P_X$ and $X_2\sim Q_X$ represent ID and OOD data, respectively. Their latent representations are denoted by $Z_1=f(X_1)\sim P_Z$ and $Z_2=f(X_2)\sim Q_Z$. Let $P_X^r$ denote the model-induced distribution defined by $Z_r\sim P_Z^r$ and $X_r=f^{-1}(Z_r)\sim P_X^r$. These distributions are represented by bold lines in Fig.~\ref{fig:big_picture}.

	When the ID and OOD data distributions have little overlap, it is natural to adopt the following assumption.
	\begin{assumption}\label{assumption0}
		The divergences $\mathrm{KL}(P_X\Vert Q_X)$ and $\mathrm{KL}(Q_X\Vert P_X)$ are both sufficiently large.
	\end{assumption}

	Because invertible flow transformations preserve KL divergence~\cite{nielsen2020elementary}, we have $\mathrm{KL}(P_Z\Vert Q_Z)=\mathrm{KL}(P_X\Vert Q_X)$. Thus, $\mathrm{KL}(P_Z\Vert Q_Z)$ is also large. On the other hand, flow-based models are typically trained to minimize the forward KL divergence $\mathrm{KL}(P_X\Vert P_X^r)$~\cite{papamakarios2021flow_model_survey,goodfellow2016deep}. KL invariance further gives $\mathrm{KL}(P_X\Vert P_X^r)=\mathrm{KL}(P_Z\Vert P_Z^r)$, and this quantity is expected to be sufficiently small after training.
	
	Normality tests do not reject multivariate normality for $P_Z$ on the datasets considered in \cite{zhang2023outofdistribution}. Moreover, the OOD representation distribution $Q_Z$ is approximately Gaussian for a wide range of OOD datasets. These observations {\color{black}motivate} approximating $P_Z$ and, when appropriate, $Q_Z$ by Gaussian distributions.
	
	The analysis in \cite{zhang2023outofdistribution} considers two settings.
	\begin{enumerate}
		\item \textbf{The Gaussian case.} As shown in Fig.~\ref{fig:big_picture}, $Q_Z$ is approximately Gaussian. Because $\mathrm{KL}(P_Z\Vert Q_Z)$ is large whereas $\mathrm{KL}(P_Z\Vert P_Z^r)$ is small, the relaxed triangle inequality~\cite{zhang2023properties} implies that $\mathrm{KL}(P_Z^r\Vert Q_Z)$ must also be large. The approximate symmetry of KL divergence between Gaussian distributions~\cite{zhang2023properties} then implies that $\mathrm{KL}(Q_Z\Vert P_Z^r)$ is large.
		\item \textbf{The general case.} When $Q_Z$ is not approximately Gaussian, Zhang \textit{et al.} imposed the strong assumption that $P_Z\approx P_Z^r$. Together with Assumption~\ref{assumption0}, this assumption leads to the conclusion that both $\mathrm{KL}(P_Z^r\Vert Q_Z)$ and $\mathrm{KL}(Q_Z\Vert P_Z^r)$ are sufficiently large.
	\end{enumerate}
	
	In summary, both settings yield large values of $\mathrm{KL}(P_Z^r\Vert Q_Z)$ and $\mathrm{KL}(Q_Z\Vert P_Z^r)$. This result explains why sampling from the prior of a flow-based model does not generate OOD-like data.
	
	\textbf{Research Question.} 
	In the analysis of the general case, Zhang \textit{et al.} imposed the strong assumption that $P_Z\approx P_Z^r$, thereby neglecting the model-fitting error. This assumption is needed because the relaxed triangle inequality in \cite{zhang2023properties} applies only to KL divergences among three Gaussian distributions. When the latent distribution of an OOD dataset is substantially non-Gaussian, no existing theoretical result supports an analysis analogous to that used in the Gaussian case.
	
	\begin{figure}[t]
		\centering
		\includegraphics[width=8cm]{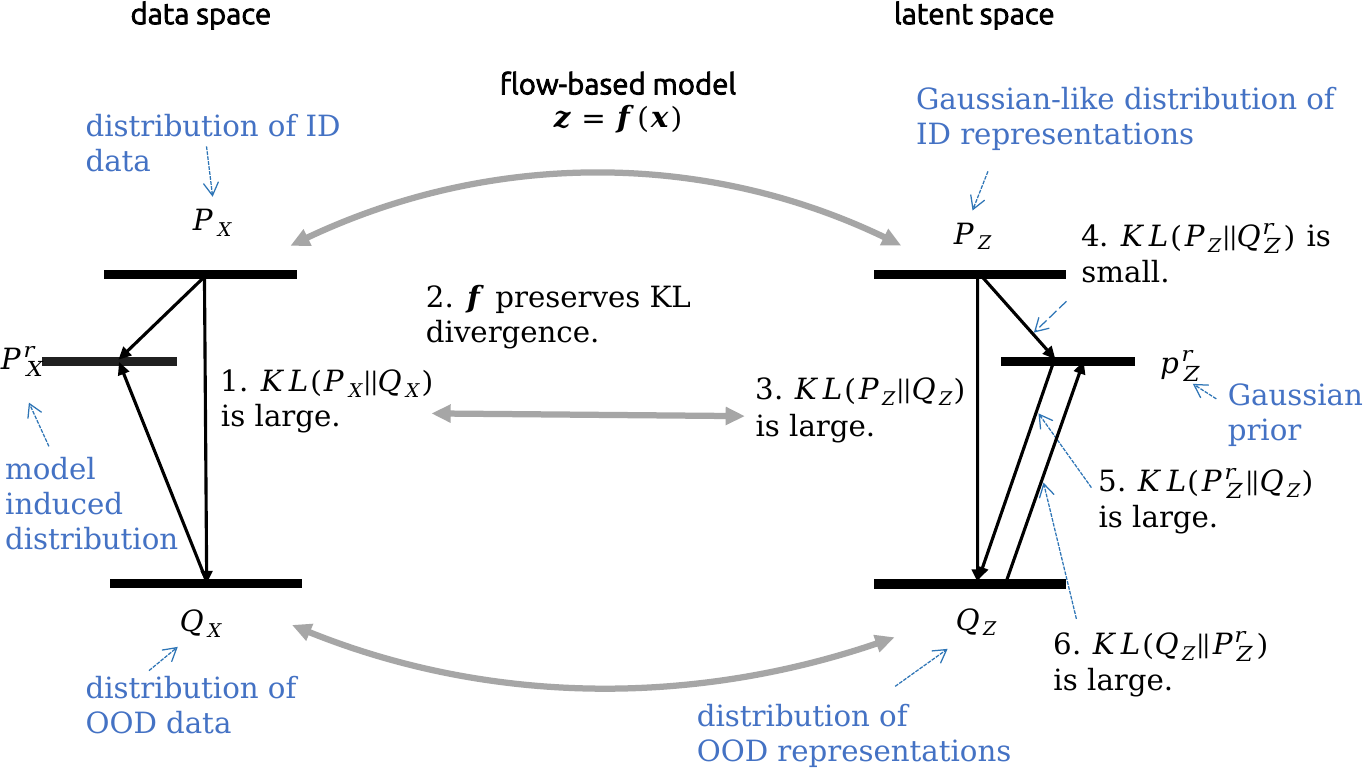}
		%\vspace{-10pt}
		\caption{KL divergence-based analysis of OOD detection for the Gaussian case where $Q_Z$ can be approximated by a Gaussian distribution \cite{zhang2023outofdistribution}.}
		\label{fig:big_picture}
	\end{figure}
	
	{\color{black}
		This limitation leads to the following research question:
		\textit{Given two Gaussian distributions $\N_1$ and $\N_2$ and an arbitrary distribution $P$, if $\mathrm{KL}(\N_1\Vert \N_2)$ is small, can we find a bound on $|\mathrm{KL}(P\Vert \N_1)-\mathrm{KL}(P\Vert \N_2)|$?}
	}
	
	We answer this question by deriving a bound on $|\mathrm{KL}(P\Vert \N_1)-\mathrm{KL}(P\Vert \N_2)|$. 
	Although motivated by flow-based OOD analysis, our result also applies to general problems involving Gaussian reference perturbations.

	\section{Theoretical Result}\label{theory}
	{\color{black}
		\textbf{Overview.} This section first presents several technical lemmas and then establishes the main theorem. 
		An overview of the proof of Theorem~\ref{thm:triangle-P-N1-N2} is provided in Fig.~\ref{fig:proof-roadmap}.
		
		Lemmas~\ref{lemma:f} and \ref{lemma:f-lower-bound} establish a lower bound for $f(x)=x-\log x-1$, 
		a function that appears in the KL divergence between Gaussian distributions. 
		Lemma~\ref{lem:bound-deltamu-deltasigma} then converts a small KL divergence into stability bounds for the corresponding means and covariance matrices.
		
		Lemmas~\ref{lem:pairaverage} and \ref{lemma:bound-sum-log-x} bound the log-determinant term. Combining these results yields the stability bound stated in the main theorem.
	}
	
	\begin{figure*}[t]
		\centering
		\scriptsize
		\begin{tikzpicture}[
			node distance=1.7cm and 1.6cm,
			every node/.style={},
			box/.style={
				draw,
				rounded corners,
				align=center,
				minimum width=2.4cm,
				minimum height=0.7cm
			},
			arrow/.style={->, thick}
			]
			
			%--------------------------------------------------
			% Left column
			%--------------------------------------------------
			\node[box] (L1)
			{Lemma \ref{lemma:f}\\
				Convexity of\\
				$f(x)=x-\log x-1$};
			
			\node[box, below=0.2cm of L1] (L2)
			{Lemma \ref{lemma:f-lower-bound}\\
				Quadratic control\\
				$f(x)\ge \frac{(x-1)^2}{3}$} ;
			
			\node[box, below=0.2cm of L2] (L3)
			{Lemma \ref{lem:pairaverage}\\
				Pair averaging\\
				convergence};
			
			%--------------------------------------------------
			% Middle column
			%--------------------------------------------------
			\node[box, right=1cm of L2] (L5)
			{Lemma \ref{lem:bound-deltamu-deltasigma}\\
				$\mathrm{KL}(\N_1\Vert \N_2)\leq \e$\\
				$\Longrightarrow$
				parameter stability};
			
			\node[box, below=0.2cm of L5] (L4)
			{Lemma \ref{lemma:bound-sum-log-x}\\
				Upper bound on\\
				$\left|\sum_i\log x_i\right|$};

			%--------------------------------------------------
			% three terms
			%--------------------------------------------------
			\node[box, right=1cm of L5] (C)
			{Control of\\
				quadratic term};
				
			\node[box, above=0.2cm of C] (B)
			{Control of\\
				linear term};
				
			\node[box, right=1.1cm of L4] (A)
			{Control of\\
				log-det term
				$|a|$};

			%--------------------------------------------------
			% Main theorem
			%--------------------------------------------------
			\node[box,
			fill=gray!10,
			right=0.8cm of C,
			minimum width=3cm,
			minimum height=0.8cm]
			(T)
			{The main Theorem \ref{thm:triangle-P-N1-N2}
			};
			
			%--------------------------------------------------
			% Arrows
			%--------------------------------------------------
			
			\draw[-{Latex[length=1mm,width=1mm]}]  (L1) -- (L2);
			\draw[-{Latex[length=1mm,width=1mm]}] (L2) -- (L5);
			
			\draw[-{Latex[length=1mm,width=1mm]}](L3) -- (L4);
			
			\draw[-{Latex[length=1mm,width=1mm]}] (L4.east) -- (A.west);
			
			\draw[-{Latex[length=1mm,width=1mm]}] (L5.east) -- (A.west);
			\draw[-{Latex[length=1mm,width=1mm]}] (L5.east) -- (B.west);
			\draw[-{Latex[length=1mm,width=1mm]}] (L5.east) -- (C.west);
			
			\draw[-{Latex[length=1mm,width=1mm]}] (A.east) -- (T.west);
			\draw[-{Latex[length=1mm,width=1mm]}] (B.east) -- (T.west);
			\draw[-{Latex[length=1mm,width=1mm]}] (C.east) -- (T.west);
			
		\end{tikzpicture}
		\caption{
			{\color{black}Proof roadmap of Theorem~\ref{thm:triangle-P-N1-N2}.}
		}
		\label{fig:proof-roadmap}
	\end{figure*}
	
	\subsection{Preliminary Lemmas}\label{sec:lemmas}
	We define $f(x)=x-\log x-1$ for $x\in\mathbb{R}_{>0}$, where $\mathbb{R}_{>0}$ denotes the set of positive real numbers. The following lemma gives the basic properties of $f$.
	\begin{lemma}\label{lemma:f}
		The function $f(x)=x-\log x-1$ is convex on $\mathbb{R}_{>0}$ and attains its minimum at $x=1$, with $f(1)=0$.
	\end{lemma}
	\begin{proof}
		The graph of $f$ and the proof are provided in Appendix~\ref{sec:appendix_proof_lemma_f} of the supplementary material.
	\end{proof}

	The following lemma provides a lower bound for $f(x)=x-\log x-1$ on a restricted interval.
	\begin{lemma}\label{lemma:f-lower-bound}
		For {\color{black}$x\in[0.3,1.7]$},
		$f(x)\geq \frac{(x-1)^2}{3}$.
	\end{lemma}
	\begin{proof}
		%We first prove $f(x)\geq \frac{(x-1)^2}{2max(x,1)}$, then use the interval to bound the right hand side.
		The proof is provided in Appendix~\ref{sec:appendix-proof-lemma-f-lower-bound} of the supplementary material.
	\end{proof}
	
	The following lemma establishes that a pairwise averaging procedure converges to a vector with equal components. It provides a basic tool for Lemma~\ref{lemma:bound-sum-log-x}.
	\begin{lemma}\label{lem:pairaverage}
		Let $\bm{x}^{(1)}=(x_1^{(1)},\dots,x_n^{(1)})\in\mathbb{R}_{\geq 0}^n$ ($n\geq 2$) satisfy
		$\sum_{i=1}^n\left(x_i^{(1)}\right)^2=C>0$. Define a sequence $\{\bm{x}^{(k)}\}_{k\geq1}$ as follows. At each iteration, choose
		\[
		(i_k,j_k) \in \arg\max_{i\neq j}
		\left|\left(x_i^{(k)}\right)^2-\left(x_j^{(k)}\right)^2\right|
		\]
		and set
		\[
		x_{i_k}^{(k+1)}=x_{j_k}^{(k+1)}
		=\sqrt{\frac{\left(x_{i_k}^{(k)}\right)^2+
				\left(x_{j_k}^{(k)}\right)^2}{2}},
		\]
		while leaving all other components unchanged. Then
		\[
		\lim_{k\to\infty}\bm{x}^{(k)}
		=\left(\sqrt{\frac{C}{n}},\dots,\sqrt{\frac{C}{n}}\right).
		\]
	\end{lemma}
	
	\begin{proof}
		See Appendix \ref{sec:appendix-proof-lem:pairaverage} of the supplementary material for details.
	\end{proof}
	
	The following lemma bounds $\left|\sum_{i=1}^n\log x_i\right|$ when the $x_i$ are collectively close to $1$. This result is used directly in the proof of Theorem~\ref{thm:triangle-P-N1-N2}.
	\begin{lemma}\label{lemma:bound-sum-log-x}
		Let $x_1,\dots,x_n$ be positive real numbers. If
		$\sum_{i=1}^n(x_i-1)^2\leq\e$ for $0<\e<1/2$, then
		\[
		\left|\sum_{i=1}^n\log x_i\right|
		\leq -n\log\left(1-\sqrt{\frac{\e}{n}}\right).
		\]
	\end{lemma}

	\begin{proof}
		The proof is provided in Appendix~\ref{sec:proof-lemma:bound-sum-log-x} of the supplementary material.
	\end{proof}

	\subsection{Small KL Divergence Controls Differences in Means and Covariances}\label{sec:small-KL-bound-mean-cov}
	
	The following lemma shows that a small KL divergence between two Gaussian distributions implies small differences in their means and covariance matrices.
	\begin{lemma}\label{lem:bound-deltamu-deltasigma}
		Let $\N_1=\mathcal{N}(\mu_1, \Sigma_1)$ and $\mathcal{N}_2=\N(\mu_2, \Sigma_2)$ be two $d$-dimensional Gaussian distributions. If 
		$\mathrm{KL}(\mathcal{N}_1\Vert \mathcal{N}_2)\leq \epsilon<1/12$, then
		$\|\mu_2-\mu_1\| \leq \|\Sigma_2^{1/2}\|_{\mathrm{op}}\sqrt{2\epsilon}$,
		$\|\Sigma_1-\Sigma_2\|_{F}	\leq \|\Sigma_2\|_{\mathrm{op}}\sqrt{6\epsilon}$, and 
		$\|\Sigma_1^{-1}-\Sigma_2^{-1}\|_{F}\leq 
		\|\Sigma_1^{-1}\|_{\mathrm{op}} \|\Sigma_2\|_{\mathrm{op}} \|\Sigma_2^{-1}\|_{\mathrm{op}}	\sqrt{6\epsilon}$. 
		Here, $\Sigma_2^{1/2}$ denotes the matrix square root of $\Sigma_2$, and 
		$\|\cdot\|_{\mathrm{op}}$ and $\|\cdot\|_{F}$ denote the operator norm and Frobenius norm of a matrix, respectively.
	\end{lemma}
	\begin{proof}
		The proof is provided in Appendix \ref{sec:proof-bound-deltamu-deltasigma} in the supplementary material.
	\end{proof}
	
	\subsection{Optimal $\sqrt{\e}$-Rate Stability Theorem}\label{sec:subsec-main-theorem}
	We now present the main theorem.
	{\color{black}
		\begin{theorem}\label{thm:triangle-P-N1-N2}
			Let $\mathcal{N}_1=\mathcal{N}(\mu_1,\Sigma_1)$ and $\mathcal{N}_2=\mathcal{N}(\mu_2,\Sigma_2)$ be $d$-dimensional Gaussian distributions.
			Let $P$ be any $d$-dimensional distribution satisfying $\mathbb{E}_P[\|X\|^2]<\infty$ and $\mathrm{KL}(P\Vert \mathcal{N}_1)<\infty$.
			If $\mathrm{KL}(\N_1\Vert \N_2) \leq \epsilon$ for $0<\epsilon<1/12$, then 
			$
			|\mathrm{KL}(P\Vert \N_1)-\mathrm{KL}(P\Vert \N_2)|\leq	B(\e),
			$
			where $$B(\e)=-\frac{1}{2}d\log\left(1-\sqrt{\frac{6\epsilon}{d}}\right)
			+C_1\sqrt{\epsilon}
			+C_2\epsilon,$$ 
			and 
			\begin{align}
				C_1
				=&
				\sqrt{2}\|\Sigma_2^{1/2}\|_{\mathrm{op}}
				\|\Sigma_2^{-1}\|_{\mathrm{op}}\|\mu_2\|
				\nonumber\\
				&+
				\frac{\sqrt{6}}{2}
				\|\Sigma_1^{-1}\|_{\mathrm{op}}
				\|\Sigma_2^{-1}\|_{\mathrm{op}}
				\|\Sigma_2\|_{\mathrm{op}}
				\|\mu_1\|^2
				\nonumber\\
				&+
				\sqrt{6}
				\|\Sigma_2^{-1}\|_{\mathrm{op}}
				\|\Sigma_1^{-1}\|_{\mathrm{op}}
				\|\Sigma_2\|_{\mathrm{op}}
				\|\mu_2\|\mathbb{E}_{P}\left[\|X\|\right]
				\nonumber\\
				&+
				\sqrt{2}
				\|\Sigma_2^{-1}\|_{\mathrm{op}}
				\|\Sigma_2^{1/2}\|_{\mathrm{op}}\mathbb{E}_{P}\left[\|X\|\right]
				\nonumber\\
				&+
				\frac{\sqrt{6}}{2}
				\|\Sigma_1^{-1}\|_{\mathrm{op}}
				\|\Sigma_2\|_{\mathrm{op}}
				\|\Sigma_2^{-1}\|_{\mathrm{op}}
				\mathbb{E}_{P}\left[\|X\|^2\right],
			\end{align}
			\begin{align}\nonumber
				C_2
				=&
				\|\Sigma_2^{-1}\|_{\mathrm{op}}
				\|\Sigma_2\|_{\mathrm{op}}\\
				&+
				2\sqrt{3}
				\|\Sigma_2^{-1}\|_{\mathrm{op}}
				\|\Sigma_1^{-1}\|_{\mathrm{op}}
				\|\Sigma_2\|_{\mathrm{op}}^{3/2}\mathbb{E}_{P}\left[\|X\|\right].
			\end{align}
			%The coefficients $C_1$ and $C_2$ depend on the parameters of 
			%$\N_1$ and $\N_2$ and the first two norm moments of $P$ ($\mathbb{E}_P\left[\|X\|\right]$ and $\mathbb{E}_P\left[\|X\|^2\right]$). See Remark~\ref{remark-C} in the supplementary material for the complete expressions.
		\end{theorem}
		
		\begin{corollary}\label{thm:cor-asymptotic-bound}
			For fixed $d$, {\color{black}if the mean norms $\|\mu_i\|$, the covariance operator
			norms $\|\Sigma_i\|_{\mathrm{op}}$, the inverse-covariance
			operator norms $\|\Sigma_i^{-1}\|_{\mathrm{op}}$ ($i=1,2$),
			the second moment $\mathbb{E}_P[\|X\|^2]$, and $\mathrm{KL}(P\Vert \mathcal{N}_1)$ remain uniformly
			bounded as $\epsilon\to0$,} then
			$$B(\e)= \frac{\sqrt{6d\e}}{2} + K\sqrt{\e} + o(\sqrt{\e}),$$ where $K$ depends on the parameters of $\N_1, \N_2$, and the first two norm moments of $P$.
			The complete expression of $K$ is given in Remark~\ref{remark-K} in the supplementary material.
			The dependence on $\sqrt{\epsilon}$ is tight in general; see Proposition \ref{thm:tightness}.
		\end{corollary}

		\begin{remark}
			The $\sqrt{\e}$ rate is a worst-case rate, and a tighter rate may hold in specific settings. 
			For example, when $P=\N_1$, the KL difference equals $\mathrm{KL}(\N_1\Vert\N_2)$, which is at most $\e$ and hence is $O(\e)$.
		\end{remark}

	}
	
	\begin{proof}
		We prove Theorem \ref{thm:triangle-P-N1-N2} and Corollary \ref{thm:cor-asymptotic-bound} together.
		
		\textbf{Step 1. Rewriting KL terms}. 
		
		{\color{black}Let $p$, $n_1$ and $n_2$ denote the densities of $P$, $\N_1$ and $\N_2$, respectively.}
		According to the definition of KL divergence, 
		\begin{align}\label{eq:rewrite-KLP-N2}
			\begin{split}
				\mathrm{KL}(P\Vert \N_2) & = \int p(x) \log \left(\frac{p(x)}{n_2(x)}\right) dx\\
				&=\int p(x) \log \left(\frac{p(x)}{n_1(x)}\frac{n_1(x)}{n_2(x)}\right) dx\\
				&=\int p(x) \log \frac{p(x)}{n_1(x)} dx + \int p(x) \log \frac{n_1(x)}{n_2(x)} dx\\
				&=\mathrm{KL}(P\Vert \N_1) + \mathbb{E}_{X\sim P}\left[\log \frac{n_1(X)}{n_2(X)}\right].
			\end{split}
		\end{align}
		It suffices to show $\left|\mathbb{E}_P\left[\log \frac{n_1(X)}{n_2(X)}\right]\right|$ is bounded.
		
		The log density of $\N(\mu,\Sigma)$ is
		$$
		\log n(x) = -\frac{1}{2}(x-\mu)^\top \Sigma^{-1}(x-\mu)-\frac{1}{2}\log\det(\Sigma)-\frac{d}{2}\log 2\pi.
		$$
		Thus, the log-density ratio $\log\frac{n_1(x)}{n_2(x)}$ has the following closed form.
		{\allowdisplaybreaks
			\begin{equation}
				\begin{aligned}\label{eq:log-N1-N2}
					&\log\frac{n_1(x)}{n_2(x)}\\
					& = \log n_1(x)-\log n_2(x)\\
					& = \frac{1}{2}\log \frac{\det(\Sigma_2)}{\det(\Sigma_1)} +\frac{1}{2}\left[(x-\mu_2)^\top \Sigma_2^{-1}(x-\mu_2)\right.\\
					&\ \ \left. -(x-\mu_1)^\top \Sigma_1^{-1}(x-\mu_1)\right]\\
					& = \frac{1}{2}\log \frac{\det(\Sigma_2)}{\det(\Sigma_1)}  +\frac{1}{2}(\mu_2^\top\Sigma_2^{-1}\mu_2-\mu_1^\top\Sigma_1^{-1}\mu_1)\\
					& \ \ +(\Sigma_1^{-1}\mu_1-\Sigma_2^{-1}\mu_2)^\top x + \frac{1}{2}x^{\top}(\Sigma_2^{-1}-\Sigma_1^{-1})x\\
					&=a + b^{\top} x +x^\top M x 
				\end{aligned}
			\end{equation}
		}
		where 
		\begin{equation}
			\begin{aligned}\label{eq:define-a-b-c}
				a &=  \frac{1}{2}\log \frac{\det(\Sigma_2)}{\det(\Sigma_1)}  +\frac{1}{2}(\mu_2^\top\Sigma_2^{-1}\mu_2-\mu_1^\top\Sigma_1^{-1}\mu_1)\\
				b & = \Sigma_1^{-1}\mu_1-\Sigma_2^{-1}\mu_2\\
				M &= \frac{1}{2}(\Sigma_2^{-1}-\Sigma_1^{-1}).
			\end{aligned}
		\end{equation}
		By \eqref{eq:log-N1-N2}, $\mathbb{E}_P\left[\log \frac{n_1(X)}{n_2(X)}\right]$ can be rewritten as
		\begin{equation}
			\begin{aligned}
				\mathbb{E}_P\left[\log \frac{n_1(X)}{n_2(X)}\right] & = a +\mathbb{E}_P\left[b^\top X\right]+\mathbb{E}_P\left[X^\top M X\right] \\
				& = a + b^\top \mathbb{E}_P\left[X\right] +\mathbb{E}_P\left[X^\top M X\right].
			\end{aligned}
		\end{equation}	
		Thus,
		\begin{align}\label{eq:upperbound-EP-logterm-abc}
			&\left|\mathbb{E}_P\left[\log \frac{n_1(X)}{n_2(X)}\right]\right| \\
			&\leq 
			\left|a\right| + \left|b^\top \mathbb{E}_P\left[X\right]\right| +\left|\mathbb{E}_P\left[X^\top M X\right]\right|\\
			& \leq \left|a\right| +\|b\| \|\mathbb{E}_P\left[X\right]\|+\left|\mathbb{E}_P\left[X^\top M X\right]\right|. \tag{by Cauchy–Schwarz inequality}
		\end{align}

		So, it suffices to prove $\left|a\right| +\|b\| \|\mathbb{E}_P\left[X\right]\|+\left|\mathbb{E}_P\left[X^\top M X\right]\right|$ is bounded.
		In the following \textbf{step 2}, we  bound  $\left|a\right|$, $\|b\| \|\mathbb{E}_P\left[X\right]\|$, and $\left|\mathbb{E}_P\left[X^\top M X\right]\right|$.

		\textbf{Step 2. Bounding $\left|a\right|$, $\|b\| \|\mathbb{E}_P\left[X\right]\|$, and $\left|\mathbb{E}_P\left[X^\top M X\right]\right|$.}
		
		\textbf{Step 2.1. Bounding $\left|a\right|$.}
		
		By definition,
		\begin{equation}
			\begin{aligned}
				|a| &=  \left|\frac{1}{2}\log \frac{\det(\Sigma_2)}{\det(\Sigma_1)} +\frac{1}{2}(\mu_2^\top\Sigma_2^{-1}\mu_2-\mu_1^\top\Sigma_1^{-1}\mu_1)\right|\\
				& \leq  \frac{1}{2}\underbrace{\left|\log \frac{\det(\Sigma_2)}{\det(\Sigma_1)}\right|}_{T_1}  +\frac{1}{2}\underbrace{\left|(\mu_2^\top\Sigma_2^{-1}\mu_2-\mu_1^\top\Sigma_1^{-1}\mu_1)\right|}_{T_2}.\label{eq:decompose-a}
			\end{aligned}
		\end{equation}
		For term $T_1$, 
		\begin{align}\label{eq:T1}
			&T_1 =\left|\log \frac{\det(\Sigma_2)}{\det(\Sigma_1)}\right|  = \left|\log \det(\s_2^{-1}\s_1)\right|
			= \left|\log\det(A)\right| \\
			&= \left|\log \prod_{i=1}^{d}\lambda_i\right|
			= \left|\sum_{i=1}^{d}\log \lambda_i\right|
		\end{align}
		where the second equation follows from Equation \text{\eqref{eq:aboutA}}.
		In Inequality \eqref{eq:sum-quadratic-lambda-bound}, we have shown $\sum_{i=1}^d (\lambda_i-1)^2  \leq 6\epsilon<1/2$. Therefore, we can apply Lemma \ref{lemma:bound-sum-log-x} to Equation \eqref{eq:T1} and obtain 
		\begin{align}
			T_1=\left|\log \frac{\det(\Sigma_2)}{\det(\Sigma_1)}\right| & \leq -d\log\left(1-\sqrt{\frac{6\e}{d}}\right) \label{eq:bound-T1-complete}\\
			&= \sqrt{6d\e}+3\e+O(\e^{3/2})\label{eq:bound-T1}
		\end{align}
		where the last equation is obtained by applying the Taylor expansion of $\log x$ at $x=1$.
		
		It remains to bound the term $T_2$.
		\begin{align}\label{eq:decompose-T2}
			&T_2=\left|\mu_2^\top\Sigma_2^{-1}\mu_2-\mu_1^\top\Sigma_1^{-1}\mu_1\right|\\
			& = \left|\mu_2^\top\Sigma_2^{-1}\mu_2-\mu_1^{\top}\s_2^{-1}\mu_1+\mu_1^{\top}\s_2^{-1}\mu_1-\mu_1^\top\Sigma_1^{-1}\mu_1\right|\nonumber\\
			& = \left|\mu_2^\top\Sigma_2^{-1}\mu_2-\mu_1^{\top}\s_2^{-1}\mu_1 + \mu_1^{\top}(\s_2^{-1}-\s_1^{-1})\mu_1\right|\nonumber\\
			& \leq \underbrace{\left|\mu_2^\top\Sigma_2^{-1}\mu_2-\mu_1^{\top}\s_2^{-1}\mu_1\right|}_{T_{21}} + \underbrace{\left|\mu_1^{\top}(\s_2^{-1}-\s_1^{-1})\mu_1\right|}_{T_{22}}.
		\end{align}	
		It suffices to bound $T_{21}$ and $T_{22}$, respectively.
		
		Using $\Delta \mu=\mu_2-\mu_1$, we can bound $T_{21}$ as follows.
		\begin{align}
			&T_{21} \\
			& = \left|\mu_2^\top\Sigma_2^{-1}\mu_2-(\mu_2-\Delta\mu)^{\top}\s_2^{-1}(\mu_2-\Delta\mu)\right|\nonumber\\
			& = \left|\mu_2^\top\Sigma_2^{-1}\mu_2-(\mu_2^\top\Sigma_2^{-1}\mu_2-\Delta\mu^\top \s_2^{-1}\mu_2-\mu_2^\top\s_2^{-1}\Delta\mu \right. \nonumber\\
			&\left. \ \ \ \   +\Delta\mu^{\top}\s_2^{-1}\Delta\mu)\right|\nonumber\\
			& = \left|2\Delta\mu^\top \s_2^{-1}\mu_2-\Delta\mu^{\top}\s_2^{-1}\Delta\mu\right|\nonumber\\
			& \leq \left|2\Delta\mu^\top \s_2^{-1}\mu_2\right|+\left|\Delta\mu^{\top}\s_2^{-1}\Delta\mu\right|\nonumber\\
			& \leq 2\|\Delta\mu\|\|\s_2^{-1}\|_{\mathrm{op}}\|\mu_2\| + \|\s_2^{-1}\|_{\mathrm{op}}\|\Delta\mu\|^2\nonumber\\
			& \leq 2 \|\Sigma_2^{1/2}\|_{\mathrm{op}}\sqrt{2\epsilon}\|\s_2^{-1}\|_{\mathrm{op}}\|\mu_2\| + \|\s_2^{-1}\|_{\mathrm{op}}(\|\Sigma_2^{1/2}\|_{\mathrm{op}}\sqrt{2\epsilon})^2 \tag{by Lemma \ref{lem:bound-deltamu-deltasigma}}\\
			& = 2\sqrt{2}\|\Sigma_2^{1/2}\|_{\mathrm{op}}\|\s_2^{-1}\|_{\mathrm{op}}\|\mu_2\|\sqrt{\epsilon} + 2\|\s_2^{-1}\|_{\mathrm{op}}\|\Sigma_2\|_{\mathrm{op}}\epsilon \label{eq:bound-T21}\\
			& = O(\sqrt{\e}) 
		\end{align}
		where Equation \eqref{eq:bound-T21} follows from $\|\s_2\|_{\mathrm{op}}=\|\s_2^{1/2}\|_{\mathrm{op}}^2$.
		
		$T_{22}$ can also be bounded by
		\begin{align}
			T_{22} & = \left|\mu_1^{\top}(\s_2^{-1}-\s_1^{-1})\mu_1\right| \nonumber\\
			& \leq \|\s_2^{-1}-\s_1^{-1}\|_{\mathrm{op}}\|\mu_1\|^2\nonumber\\
			& \leq \|\s_2^{-1}-\s_1^{-1}\|_{F}\|\mu_1\|^2\nonumber\\
			%& = \|\s_1^{-1}(\s_1-\s_2)\s_2^{-1}\|_{\mathrm{op}} \|\mu_1\|^2\nonumber\\
			%& \leq \|\s_1^{-1}\|_{\mathrm{op}}\|\s_2^{-1}\|_{\mathrm{op}}\|\s_1-\s_2\|_F \|\mu_1\|^2 %\label{eq:sigmadif-1-diff}\\
			& \leq \sqrt{6}\|\s_1^{-1}\|_{\mathrm{op}}\|\s_2^{-1}\|_{\mathrm{op}}\|\Sigma_2\|_{\mathrm{op}} \|\mu_1\|^2\sqrt{\epsilon} \tag{by Lemma \ref{lem:bound-deltamu-deltasigma} }\\
			%& = F_2 \sqrt{\epsilon}  \\
			& = O(\sqrt{\epsilon} \label{eq:bound-T22})
		\end{align}
		where the second inequality follows from $\|\cdot \|_{\mathrm{op}}\leq \|\cdot\|_{F}$.
		
		Substituting  \eqref{eq:bound-T1}, \eqref{eq:bound-T21}, and \eqref{eq:bound-T22} into \eqref{eq:decompose-a} and \eqref{eq:decompose-T2} yields
		\begin{align}
			|a|& \\
			\leq &\frac{1}{2}T_1+\frac{1}{2}T_{2}\nonumber\\
			\leq &\frac{1}{2}T_1+\frac{1}{2}(T_{21}+T_{22})\nonumber\\
			\leq &\frac{1}{2}(\sqrt{6d\e}+3\e+O(\e^{3/2})) \\
			&+\sqrt{2}\|\Sigma_2^{1/2}\|_{\mathrm{op}}\|\s_2^{-1}\|_{\mathrm{op}}\|\mu_2\|\sqrt{\epsilon} + \|\s_2^{-1}\|_{\mathrm{op}}\|\Sigma_2\|_{\mathrm{op}}\epsilon \nonumber\\
			& + \frac{\sqrt{6}}{2}\|\s_1^{-1}\|_{\mathrm{op}}\|\s_2^{-1}\|_{\mathrm{op}}\|\Sigma_2\|_{\mathrm{op}} \|\mu_1\|^2\sqrt{\epsilon}\nonumber\\
			=& (\frac{3}{2}+\|\s_2^{-1}\|_{\mathrm{op}}\|\Sigma_2\|_{\mathrm{op}})\e + (\sqrt{2}\|\Sigma_2^{1/2}\|_{\mathrm{op}}\|\s_2^{-1}\|_{\mathrm{op}}\|\mu_2\| \\ 
			& +\frac{\sqrt{6}}{2}\|\s_1^{-1}\|_{\mathrm{op}}\|\s_2^{-1}\|_{\mathrm{op}}\|\Sigma_2\|_{\mathrm{op}} \|\mu_1\|^2 + \frac{\sqrt{6d}}{2})\sqrt{\epsilon}\nonumber\\
			& + O(\e^{3/2})\nonumber\\
			=& {\color{black}\frac{\sqrt{6d\e}}{2}+F_1\sqrt{\e}+o(\sqrt{\e})}\label{eq:bound-a-K1K2K3}
		\end{align}
		
		\textbf{Step 2.2. Bounding $\|b\| \|\mathbb{E}_P\left[X\right]\|$.}
		
		Because the norm $\|\cdot\|$ is convex, it follows that
		\begin{align}
			\|\mathbb{E}_P\left[X\right]\| \leq \mathbb{E}_P\left[\|X\|\right].	\label{eq:bound-EP}
		\end{align}
		
		Because $g(x)=\sqrt{x}$ is concave, 
		\begin{align}
			\mathbb{E}_P[\|X\|]=\mathbb{E}_P[\sqrt{\|X\|^2}]\leq 	\sqrt{\mathbb{E}_P[\|X\|^2]}.
		\end{align}
		Combining these two inequalities yields $\|\mathbb{E}_P[X]\| \leq \sqrt{\mathbb{E}_P[\|X\|^2]}$. Because $P$ has a finite second moment, $\|\mathbb{E}_P[X]\|$ is also finite.
		
		From the definition of $b$ (Equation \eqref{eq:define-a-b-c}), we can rewrite $b$ as follows.
		\begin{align}\nonumber
			b & = \Sigma_1^{-1}\mu_1-\Sigma_2^{-1}\mu_2 \nonumber\\
			& = \Sigma_1^{-1}\mu_1-\s_2^{-1}\mu_1+\s_2^{-1}\mu_1-\Sigma_2^{-1}\mu_2 \nonumber\\
			& = (\Sigma_1^{-1}-\s_2^{-1})\mu_1 + \s_2^{-1}(\mu_1-\mu_2) \nonumber\\
			& = (\Sigma_1^{-1}-\s_2^{-1})(\mu_2-\Delta\mu) + \s_2^{-1}(\mu_1-\mu_2) \nonumber\\
			& = (\Sigma_1^{-1}-\s_2^{-1})\mu_2 - (\Sigma_1^{-1}-\s_2^{-1})\Delta\mu+ \s_2^{-1}(\mu_1-\mu_2). \nonumber
		\end{align}
		Therefore, we obtain the following bound
		\begin{align}
			\|b\|  \leq & \underbrace{\|(\Sigma_1^{-1}-\s_2^{-1})\mu_2\|}_{T_3} + \underbrace{\|(\Sigma_1^{-1}-\s_2^{-1})\Delta\mu\|}_{T_4} \nonumber\\
			&+ \underbrace{\|\s_2^{-1}(\mu_1-\mu_2)\|}_{T_5}\label{eq:bound-b-T3T4T5}.
		\end{align}
		We next bound the three terms $T_3$, $T_4$, and $T_5$ separately using Lemma \ref{lem:bound-deltamu-deltasigma}.
		\begin{align}
			T_3 & = \|(\Sigma_1^{-1}-\s_2^{-1})\mu_2\| \nonumber\\
			& \leq \|(\Sigma_1^{-1}-\s_2^{-1})\|_{\mathrm{op}}\|\mu_2\| \nonumber\\
			& \leq \|(\Sigma_1^{-1}-\s_2^{-1})\|_{F}\|\mu_2\| \nonumber\\
			%		& = \|\s_2^{-1}(\s_2-\s_1)\s_1^{-1}\|_{\mathrm{op}}\|\mu_2\| \nonumber\\
			%		& \leq \|\s_2^{-1}\|_{\mathrm{op}}\|\s_1^{-1}\|_{\mathrm{op}}\|\s_2-\s_1\|_F\|\mu_2\| \label{eq:bound-T3-step}\\
			& \leq  \|\s_2^{-1}\|_{\mathrm{op}}\|\s_1^{-1}\|_{\mathrm{op}}\|\Sigma_2\|_{\mathrm{op}}\sqrt{6\epsilon}\|\mu_2\| \label{eq:boundT3_up}  \\
			& = \sqrt{6}\|\s_2^{-1}\|_{\mathrm{op}}\|\s_1^{-1}\|_{\mathrm{op}}\|\Sigma_2\|_{\mathrm{op}}\|\mu_2\|\sqrt{\epsilon}. \label{eq:boundT3}
		\end{align}
		
		Similarly, $T_4$ is bounded by
		\begin{align}
			T_4 & = \|(\Sigma_1^{-1}-\s_2^{-1})\Delta\mu\|\nonumber\\
			& \leq  \|\s_2^{-1}\|_{\mathrm{op}}\|\s_1^{-1}\|_{\mathrm{op}}\|\Sigma_2\|_{\mathrm{op}}\sqrt{6\epsilon}\|\Delta\mu\| \tag{by \eqref{eq:boundT3_up}} \\
			& \leq \|\s_2^{-1}\|_{\mathrm{op}}\|\s_1^{-1}\|_{\mathrm{op}}\|\Sigma_2\|_{\mathrm{op}}\sqrt{6\epsilon}\|\Sigma_2^{1/2}\|_{\mathrm{op}}\sqrt{2\epsilon} \tag{by Lemma \ref{lem:bound-deltamu-deltasigma}} \\
			& = 2\sqrt{3} \|\s_2^{-1}\|_{\mathrm{op}}\|\s_1^{-1}\|_{\mathrm{op}}\|\Sigma_2\|^{3/2}_{\mathrm{op}}\epsilon. \label{eq:boundT4}
		\end{align}
		
		Finally, the last term $T_5$ is bounded by
		\begin{align}
			T_5 & = \|\s_2^{-1}(\mu_1-\mu_2)\| \nonumber\\
			& \leq \|\s_2^{-1}\|_{\mathrm{op}}\|(\mu_1-\mu_2)\| \nonumber\\
			& \leq \|\s_2^{-1}\|_{\mathrm{op}} \|\Sigma_2^{1/2}\|_{\mathrm{op}}\sqrt{2\epsilon}  \tag{by Lemma \ref{lem:bound-deltamu-deltasigma}} \\
			& =  \sqrt{2} \|\s_2^{-1}\|_{\mathrm{op}} \|\Sigma_2^{1/2}\|_{\mathrm{op}}\sqrt{\epsilon}. \label{eq:boundT5}
		\end{align}

		Now combining Inequalities \eqref{eq:bound-EP}, \eqref{eq:bound-b-T3T4T5}, \eqref{eq:boundT3}, \eqref{eq:boundT4}, and \eqref{eq:boundT5}, we obtain a bound on $\|b\|\|\mathbb{E}_P\left[X\right]\|$ as
		\begin{align}
			&\|b\|\|\mathbb{E}_P\left[X\right]\| \\
			&\leq	 (\sqrt{6}\|\s_2^{-1}\|_{\mathrm{op}}\|\s_1^{-1}\|_{\mathrm{op}}\|\Sigma_2\|_{\mathrm{op}}\|\mu_2\| \\
			& + \sqrt{2} \|\s_2^{-1}\|_{\mathrm{op}} \|\Sigma_2^{1/2}\|_{\mathrm{op}})\mathbb{E}_P\left[\|X\|\right]\sqrt{\epsilon} \\
			& +2\sqrt{3} \|\s_2^{-1}\|_{\mathrm{op}}\|\s_1^{-1}\|_{\mathrm{op}}\|\Sigma_2\|^{3/2}_{\mathrm{op}}\mathbb{E}_P\left[\|X\|\right]\epsilon \nonumber\\
			& = F_2 \sqrt{\epsilon}  + F'_2 \e \label{eq:bound-b}\\
			& = O(\sqrt{\epsilon}).
		\end{align}
		
		\textbf{Step 2.3. Bounding $\left|\mathbb{E}_P\left[X^\top M X\right]\right|$.}
		\begin{align} 
			&\left|\mathbb{E}_P\left[X^\top M X\right]\right| \\
			&= 		\left|\mathbb{E}_P\left[X^\top \frac{1}{2}(\Sigma_2^{-1}-\Sigma_1^{-1})  X\right]\right| \tag{by Equation \eqref{eq:define-a-b-c}} \\
			&\leq \frac{1}{2}\left|\mathbb{E}_P\left[ \|(\Sigma_2^{-1}-\Sigma_1^{-1})\|_F \|X\|^2  \right]\right|\\
			&\leq 	\frac{1}{2}\|(\Sigma_2^{-1}-\Sigma_1^{-1})\|_F\left|\mathbb{E}_P\left[  \|X\|^2  \right]\right| \\
			&\leq \frac{\sqrt{6}}{2}\|\s_1^{-1}\|_{\mathrm{op}}\|\Sigma_2\|_{\mathrm{op}}\|\s_2^{-1}\|_{\mathrm{op}}\mathbb{E}_P\left[  \|X\|^2  \right]\sqrt{\epsilon} \tag{by Lemma \ref{lem:bound-deltamu-deltasigma}} \\
			&=  	 F_3 \sqrt{\e}. \label{eq:bound-c}
		\end{align}
		
		Combining Inequalities \eqref{eq:bound-a-K1K2K3}, \eqref{eq:bound-b}, and \eqref{eq:bound-c} with \eqref{eq:upperbound-EP-logterm-abc} yields
		{\color{black}
			\begin{align}\label{eq:upperbound-EP-logterm}
				\begin{split}
					&\left|\mathbb{E}_P\left[\log \frac{n_1(X)}{n_2(X)}\right]\right| \\
					&\leq\left|a\right| +\|b\| \|\mathbb{E}_P\left[X\right]\|+\left|\mathbb{E}_P\left[X^\top M X\right]\right| \\
					&\leq  \frac{\sqrt{6d\e}}{2} + (F_1+F_2+F_3)\sqrt{\e} + o(\sqrt{\e})\\
					&= \frac{\sqrt{6d\e}}{2} + K\sqrt{\e} + o(\sqrt{\e}),
				\end{split}
			\end{align}
			where $K = F_1+F_2+F_3$. A full definition of $F_1$, $F_2$, and $F_3$ is provided in Appendix \ref{sec:explicit-bounds} in the supplementary material.
%			\begin{align}\label{eq:F123}
%				\begin{split}
%					F_1 = &\sqrt{2}\|\Sigma_2^{1/2}\|_{\mathrm{op}}\|\s_2^{-1}\|_{\mathrm{op}}\|\mu_2\|  \\
%					&  +\frac{\sqrt{6}}{2}\|\s_1^{-1}\|_{\mathrm{op}}\|\s_2^{-1}\|_{\mathrm{op}}\|\Sigma_2\|_{\mathrm{op}} \|\mu_1\|^2 	\\
%					F_2 = & (\sqrt{6}\|\s_2^{-1}\|_{\mathrm{op}}\|\s_1^{-1}\|_{\mathrm{op}}\|\Sigma_2\|_{\mathrm{op}}\|\mu_2\|  \\
%					& + \sqrt{2} \|\s_2^{-1}\|_{\mathrm{op}} \|\Sigma_2^{1/2}\|_{\mathrm{op}})\mathbb{E}_P\left[\|X\|\right] \\
%					F_3 = & \frac{\sqrt{6}}{2}\|\s_1^{-1}\|_{\mathrm{op}}\|\Sigma_2\|_{\mathrm{op}}\|\s_2^{-1}\|_{\mathrm{op}}\mathbb{E}_P\left[  \|X\|^2  \right]. 
%				\end{split}
%			\end{align}

			Substituting this into Equality \eqref{eq:rewrite-KLP-N2}, we can obtain
			\begin{align}\nonumber
				|\mathrm{KL}(P\Vert \N_1)-\mathrm{KL}(P\Vert \mathcal{N}_2)|\leq \frac{\sqrt{6d\e}}{2} + K\sqrt{\e} + o(\sqrt{\e}) . 
			\end{align}
		}
		
		Finally, combining the decompositions in \eqref{eq:upperbound-EP-logterm-abc}, \eqref{eq:decompose-a}, and \eqref{eq:decompose-T2} 
		with the bounds in \eqref{eq:bound-T1-complete}, \eqref{eq:bound-T21}, \eqref{eq:bound-T22}, \eqref{eq:bound-b-T3T4T5}--\eqref{eq:boundT5}, 
		and \eqref{eq:bound-c} yields the explicit form of $B(\e)$.
		
		This completes the proof.
	\end{proof}
	{\color{black}
		Fig. \ref{fig:small_picture} shows the overview of Theorem \ref{thm:triangle-P-N1-N2}. 
	}
	
	\textbf{Tightness.} The following proposition shows that the $\sqrt{\epsilon}$ rate in Theorem~\ref{thm:triangle-P-N1-N2} cannot be improved in general.
	\begin{proposition}[Optimality of the $\sqrt{\epsilon}$ rate]\label{thm:tightness}
		The $\sqrt{\epsilon}$ dependence in Theorem~\ref{thm:triangle-P-N1-N2} is optimal in general.
	\end{proposition}
	\begin{proof}
		It suffices to construct an example whose KL deviation is of order $\sqrt{\epsilon}$.
		The example is constructed using $P=\N(t e_1,\mathbf{I}_2)$ for a fixed constant $t>0$, 
		$\N_1=\N(\mathbf{0},\mathbf{I}_2)$ and
		$\N_2=\N(\delta e_1,\mathbf{I}_2)$, where $\delta\to0$ and $e_1=(1,0)^\top$.
		See Appendix \ref{sec:proof_e_tight} in the supplementary material for details.
	\end{proof}

	\begin{figure}[t]
		\centering
		\includegraphics[width=8cm]{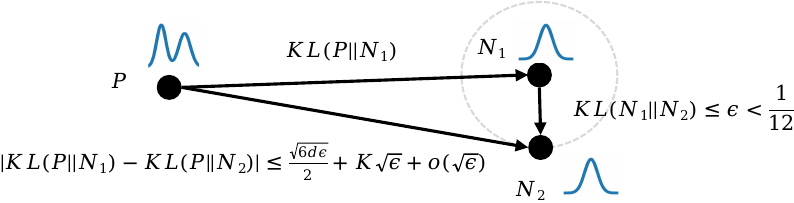}
		%\vspace{-10pt}
		\caption{Stability of KL divergence $\mathrm{KL}(P\Vert \N_1)$ under Gaussian perturbation. Arrows represent KL divergences.}
		\label{fig:small_picture}
	\end{figure}

	%The tightness of the bound shows
	{\color{black}The above example shows}
	 that even when all distributions are Gaussian, one can construct examples where the difference 
	$|\mathrm{KL}(P\Vert \mathcal{N}_2) - \mathrm{KL}(P\Vert \mathcal{N}_1)|$ scales as $\Theta(\sqrt{\epsilon})$.
	Thus, the $\sqrt{\epsilon}$ dependence does not arise from the non-Gaussian extension, 
	but already appears in the fully Gaussian setting.
	This indicates that the $\sqrt{\epsilon}$ term is intrinsic to the geometry of KL divergence.

	In this example, the underlying reason is a mismatch in sensitivity.
	The KL divergence between Gaussian distributions scales quadratically with respect to small perturbations 
	(\textit{i.e.}, $\epsilon = \Theta(\delta^2)$), 
	while the KL divergence $\mathrm{KL}(P\Vert \cdot)$ varies linearly with the perturbation (\textit{i.e.}, $\Theta(\delta)$).
	Combining these yields the $\Theta(\sqrt{\epsilon})$ scaling.

	{\color{black}
		
		\begin{remark}[Dimension dependence and worst-case tightness]
			The explicit $\sqrt{d\epsilon}$ term does not capture the full
			dimension dependence of the bound, because
			$K=F_1+F_2+F_3$ also depends on the Gaussian parameters and
			the moments of $P$.
			In the typical regime where the covariance and inverse-covariance
			operator norms are bounded independently of $d$ and $\epsilon$,
			$\|\mu_i\|=O(\sqrt d)$ for $i=1,2$, and
			$\mathbb E_P\|X\|^2=\Theta(d)$, we have
			$\mathbb E_P\|X\|=O(\sqrt d)$ and $F_1,F_2,F_3=O(d)$.
			Consequently, the explicit bound in Theorem~1 satisfies
			$
			B(\epsilon)=O(d\sqrt{\epsilon})
			$. 			
			This order is tight in the worst case under these conditions.
			Specifically, for any fixed $p\ge1$, the example in Appendix~\ref{sec:example_dimension_dependence} of the supplementary
			material shows a family of examples satisfying
			$\mathbb E_{P_d}\|X\|^2=\Theta(d^p)$ and
			$
			\left|
			\mathrm{KL}(P_d\Vert \N_1)-\mathrm{KL}(P_d\Vert \N_2)
			\right|
			=\Theta(d^p\sqrt{\epsilon})
			$.
		\end{remark}
		
	}
	
	\vspace{-10pt}
	\subsection{Corollaries}
	\subsubsection{Standard Gaussian Case}
	In many applications, the second Gaussian distribution $\N_2$ is chosen as the standard Gaussian $\N(\mathbf{0},\mathbf{I}_d)$. In this case, we obtain a bound that is independent of the parameters of $\N_1$, as stated in the following corollary.
	
	{\color{black}
		\begin{corollary}\label{thm:triangle-P-N1-standardN2}
			Let $\mathcal{N}_1=\mathcal{N}(\mu_1,\Sigma_1)$ be a $d$-dimensional Gaussian distribution, 
			and let $P$ be any $d$-dimensional distribution satisfying $\mathbb{E}_P[\|X\|^2]<\infty$ and $\mathrm{KL}(P\Vert \mathcal{N}_1)<\infty$.
			If $\mathrm{KL}(\N_1\Vert \N(\mathbf{0},\mathbf{I}_d)) \leq \epsilon$ for $0<\epsilon<1/12$, then
			\begin{align}
				\begin{split}
					|\mathrm{KL}(P\Vert \mathcal{N}_1)-\mathrm{KL}(P\Vert \N(\mathbf{0},\mathbf{I}_d))|< \frac{\sqrt{6d\e}}{2} + S\sqrt{\e}+ o(\sqrt{\e}),
				\end{split}
			\end{align}
			where $S=\sqrt{2}\mathbb{E}_P \left[\|X\|\right]+(\sqrt{3}+\sqrt{6})\mathbb{E}_P \left[\|X\|^2\right]$.
		\end{corollary}
	}
	\begin{proof}
		Setting $\N_2=\N(\mathbf{0},\mathbf{I}_d)$ simplifies $F_1$, $F_2$, and $F_3$ and the bounds on the mean and covariance of $\N_1$.
		%simplifies these constants. Moreover, the condition
		%$\mathrm{KL}\bigl(\N_1\Vert\N(\mathbf{0},\mathbf{I}_d)\bigr)\leq\e$
		%yields bounds on the mean and covariance of $\N_1$. 
		%Substituting these bounds into the proof of Theorem~\ref{thm:triangle-P-N1-N2} gives the stated bound on
		%$\left|\mathrm{KL}(P\Vert\N_1)-\mathrm{KL}\bigl(P\Vert\N(\mathbf{0},\mathbf{I}_d)\bigr)\right|$.
		See Appendix~\ref{proof-thm:triangle-P-N1-standardN2} of the supplementary material for details.
	\end{proof}

	{\color{black}
		\subsubsection{Extending to Multiple Perturbation Steps}
		Theorem~\ref{thm:triangle-P-N1-N2} can be generalized to multiple perturbation steps.
		We first extend the existing relaxed triangle inequality for KL divergence between Gaussian distributions~\cite{xiao2026relaxedtriangleinequalitykullbackleibler} to the multistep setting in Corollary~\ref{cor:multi-step-KL-stability}.
		Using this result, we then establish the multistep extension of Theorem~\ref{thm:triangle-P-N1-N2} in Corollary~\ref{thm:multiple-steps-perturbation}.
		
		\begin{corollary}[Multistep KL Stability for Gaussian Chains]
			\label{cor:multi-step-KL-stability}
			For a fixed integer $k\geq 2$, let $\{\mathcal{N}_i\}_{i=1}^{k}$ be $k$ $d$-dimensional Gaussian distributions.
			Suppose that $\mathrm{KL}(\mathcal{N}_i\|\mathcal{N}_{i+1})\leq \epsilon, \quad i=1,\ldots,k-1,$
			where $\epsilon\geq 0$ is a small number.
			Then, the KL divergence between $\mathcal{N}_1$ and $\mathcal{N}_k$ is bounded by
			\[
			\mathrm{KL}(\mathcal{N}_1\Vert \mathcal{N}_k)
			\leq
			(k-1)^2\epsilon+o(\epsilon).
			\]
		\end{corollary}
		
		\begin{proof}
			The proof proceeds by repeatedly applying the relaxed triangle inequality for Gaussian distributions~\cite{xiao2026relaxedtriangleinequalitykullbackleibler}.
			See Appendix \ref{sec:proof-multiple-step-triangle-KL} of the supplementary material.
		\end{proof}
		
		\begin{corollary}\label{thm:multiple-steps-perturbation}
			For fixed $d$ and a fixed integer $k\ge2$, let $\{\mathcal{N}_i\}_{i=1}^{k}$ be $k$ $d$-dimensional
			Gaussian distributions, where
			$\mathcal{N}_i=\mathcal{N}(\mu_i,\Sigma_i)$, satisfying
			$
			\mathrm{KL}(\mathcal{N}_i\|\mathcal{N}_{i+1})
			\leq \epsilon,
			\qquad i=1,\ldots,k-1
			$.
			Let $P$ be a distribution satisfying
			$\mathrm{KL}(P\|\mathcal{N}_1)<\infty$.
			The distributions may depend on $\epsilon$.
			{\color{black}Assume that
			$
			\|\mu_i\|,
			\|\Sigma_i\|_{\mathrm{op}},
			\|\Sigma_i^{-1}\|_{\mathrm{op}}
			\ (i=1,k),
			\mathbb{E}_P[\|X\|^2]
			$
			remain uniformly bounded as $\epsilon\to0$.}
			Then, for sufficiently small $\epsilon>0$,
			\begin{equation}
				\begin{aligned}
					&\left|
					\mathrm{KL}(P\|\mathcal{N}_1)
					-\mathrm{KL}(P\|\mathcal{N}_k)
					\right|\\
					&\qquad\leq
					\frac{(k-1)\sqrt{6d\epsilon}}{2}
					+(k-1)K\sqrt{\epsilon}
					+o(\sqrt{\epsilon}),
				\end{aligned}
			\end{equation}
			where $K$ is the coefficient defined in
			Remark~\ref{remark-K} of the supplementary material,
			evaluated for $P$ and the Gaussian pair
			$(\mathcal{N}_1,\mathcal{N}_k)$.
		\end{corollary}
%		\begin{corollary}\label{thm:multiple-steps-perturbation}
%			Let $\{\mathcal{N}_i\}_{i=1}^{k}$ be $k$ $d$-dimensional Gaussian distributions satisfying $\mathrm{KL}(\mathcal{N}_i\|\mathcal{N}_{i+1})\leq \epsilon, \quad i=1,\ldots,k-1$. 
%			Let $\e\geq 0$ is a small number satisfying $(k-1)^2\e+o(\e)<1/12$. Assume the mean
%			norms, covariance and inverse-covariance operator norms
%			of $\N_1$ and $\N_k$
%			remain uniformly bounded as $\epsilon\to0$.
%			Let $P$ be an arbitrary distribution satisfying $\mathbb{E}_P[\|X\|^2]<\infty$ and $\mathrm{KL}(P\Vert \mathcal{N}_1)<\infty$ uniformly. Then
%			\begin{equation}
%				\small
%				\left|\mathrm{KL}(P\Vert\N_1)-\mathrm{KL}(P\Vert\N_k)\right|
%				\leq \frac{(k-1)\sqrt{6d\epsilon}}{2}+(k-1)K\sqrt{\epsilon}+o(\sqrt{\epsilon}),
%			\end{equation}
%			where $K$ is the constant defined in Remark~\ref{remark-K} in the supplementary material, evaluated for the Gaussian pair $(\N_1,\N_k)$.
%		\end{corollary}
		
		\begin{proof}
			The proof proceeds in two steps. 
			First, we apply Corollary~\ref{cor:multi-step-KL-stability} to $\{\mathcal{N}_i\}_{i=1}^{k}$ to bound $\mathrm{KL}(\mathcal{N}_1\Vert \mathcal{N}_k)$.
			We then apply Theorem~\ref{thm:triangle-P-N1-N2} to $P$, $\N_1$, and $\N_k$. 
			See Appendix~\ref{sec:proof-thm:multiple-steps-perturbation} of the supplementary material.
		\end{proof}
	}
	
	%In this paper, we apply the above conclusions on flow-based model where $P$ represents the distribution of representations of a given OOD dataset. In this setting, the second moment of $P$ is bounded by the model. 

	\section{{\color{black}Experiments}}\label{sec:exp}
	{\color{black}
		\subsection{Overview}
		In this section, we empirically evaluate the stability behavior predicted by Theorem~\ref{thm:triangle-P-N1-N2} through 
		four experiments on synthetic data and flow-based models.
		The source code is publicly available at \url{https://github.com/yuffon/KL-stability-experiments}.
		
		Across the experiments, we construct $P$, $\N_1$, and $\N_2$ under different settings. 
		For the numerical evaluation, we report the following quantities:
		\begin{enumerate}
			\item \textbf{KL Difference $\widehat{\Delta}_{\mathrm{KL}}$}: We define the signed KL difference by
			\(
			\Delta_{\mathrm{KL}}=\mathrm{KL}(P\Vert\N_1)-\mathrm{KL}(P\Vert\N_2)
			\)
			and estimate it by
			\begin{align}\label{equ:diffKL}
				\widehat{\Delta}_{\mathrm{KL}}
				&=\frac{1}{M}\sum_{i=1}^{M}
				\left[\log n_2(X_i)-\log n_1(X_i)\right],
			\end{align}
			where $n_1$ and $n_2$ denote the densities of $\N_1$ and $\N_2$, respectively, and $X_1,\dots,X_M$ are independent samples from $P$. Note that $\widehat{\Delta}_{\mathrm{KL}}$ can be estimated without evaluating the density of $P$. 
			
			\item \textbf{Estimated Theoretical Bound $\widehat{B}(\e)$}: We use $\widehat{B}(\e)$ to denote the estimated version of the theoretical bound $B(\e)$, where $\mathbb{E}_P[\|X\|]$ and $\mathbb{E}_P[\|X\|^2]$ in $B(\e)$ are approximated using samples drawn from $P$.
			%The explicit form of $B(\e)$ is expressed in Equation \eqref{eq:explicit-bound} in the supplementary material.
			
			\item \textbf{Empirical Tightness Ratio $\mathrm{TR}$}: To evaluate the tightness of the proposed bound, we define the empirical tightness ratio as
			\begin{equation}\label{eq:tr}
				\mathrm{TR}=\frac{|\widehat{\Delta}_{\mathrm{KL}}|}{\widehat{B}(\e)}.
			\end{equation}
			Note that $\mathrm{TR}$ measures empirical tightness only in the specified setting and does not reflect theoretical tightness in general.
			
			\item \textbf{Empirical Coefficient $\widehat{K^{\prime}}$ and Theoretical Coefficient $K'$}: 
			We rewrite the bound in 
			%Theorem~\ref{thm:triangle-P-N1-N2} 
			{\color{black}Corollary~\ref{thm:cor-asymptotic-bound}} as
			\begin{equation}\label{eq:bound-B}
				\frac{\sqrt{6d\e}}{2}+K\sqrt{\e}+o(\sqrt{\e})
				=K'\sqrt{\e}+o(\sqrt{\e}),
			\end{equation}
			where $K'=\sqrt{6d}/2+K$ is the theoretical coefficient {\color{black}defined} %computeed
			 using Equation \eqref{eq:thmF123} in the supplementary material.
			We {\color{black}estimate} %define
			 the corresponding signed empirical coefficient as
			\begin{equation}\label{equ:K'}
				\widehat{K^{\prime}}
				=\frac{\widehat{\Delta}_{\mathrm{KL}}}
				{\sqrt{\mathrm{KL}(\N_1\Vert\N_2)}}.
			\end{equation}
			We retain the sign of $\widehat{K^{\prime}}$ to indicate the direction of the signed KL difference.
			
		\end{enumerate}

		\subsection{Experiments on Synthetic Data}\label{sec:exp_synthetic}
		\subsubsection{One-Dimensional Synthetic Experiment}
		
		We first conduct a one-dimensional synthetic experiment to illustrate the proposed KL stability result under Gaussian perturbations. We choose 
		\(
		\N_1=\N(0,1)
		\)
		as the reference Gaussian distribution.
		We construct $P$ as a non-Gaussian mixture of two Gaussian components truncated to the interval
		\(
		[L,U]=[-4,4].
		\)
		The density of $P$ is a weighted sum of the two truncated component densities. We set the mixture weights to $w_1=w_2=0.5$, the component means to $m_1=-2.2$ and $m_2=2.2$, and the corresponding standard deviations to $s_1=0.35$ and $s_2=0.45$. This construction yields a clearly bimodal, non-Gaussian distribution. Moreover, because $P$ has bounded support, it has a finite second moment. Further details of $P$ are provided in Appendix~\ref{sec:detail_1D_synthetic_exp} of the supplementary material.
		
		Starting from $\N_1$, we generate a family of nearby Gaussian distributions
		\(
		\N_2=\N(\mu_2,\sigma_2^2)
		\)
		using several perturbation schemes. We consider perturbations of the mean, the standard deviation, and both parameters simultaneously. The perturbation magnitude is controlled by a scalar parameter $\alpha\in[0,1]$. Specifically, we use the following schemes:
		\begin{enumerate}
			\item \textit{mean-plus}: $\mu_2=\mu_1+0.4\alpha$, $\sigma_2=\sigma_1$;
			\item \textit{mean-minus}: $\mu_2=\mu_1-0.4\alpha$, $\sigma_2=\sigma_1$;
			\item \textit{sigma-expand}: $\mu_2=\mu_1$, $\sigma_2=\sigma_1\exp(0.3\alpha)$;
			\item \textit{sigma-shrink}: $\mu_2=\mu_1$, $\sigma_2=\sigma_1\exp(-0.24\alpha)$;
			\item \textit{both-plus}: $\mu_2=\mu_1+0.25\alpha$, $\sigma_2=\sigma_1\exp(0.2\alpha)$;
			\item \textit{both-minus}: $\mu_2=\mu_1-0.25\alpha$, $\sigma_2=\sigma_1\exp(-0.18\alpha)$.
		\end{enumerate}
		
		Fig.~\ref{fig:1D_P_N1_N2} in the supplementary material visualizes the mixture distribution $P$, the reference Gaussian $\N_1$, and a representative perturbed Gaussian $\N_2$ obtained using the \textit{sigma-expand} scheme with $\alpha=1$.
		
		For each perturbed Gaussian $\N_2$, we compute $\mathrm{KL}(\N_1\Vert\N_2)$ using the closed-form expression in~\eqref{eq:kl_n1_n2}.
		As shown in Fig.~\ref{fig:1D_KL_eps}\subref{fig:1d_kl_vs_perturbation} in the supplementary material, the perturbation magnitudes are chosen such that
		\(
		\mathrm{KL}(\N_1\Vert\N_2)<1/12,
		\)
		thereby satisfying the small perturbation condition of Theorem~\ref{thm:triangle-P-N1-N2}.
		
		We estimate $\mathrm{KL}(P\Vert\N_1)$ and $\mathrm{KL}(P\Vert\N_2)$ using Monte Carlo samples from $P$. Given independent samples $X_1,\ldots,X_M\sim P$, the KL divergence from $P$ to a Gaussian distribution $\N$ with density $n(x)$ is estimated as
		$
		\widehat{\mathrm{KL}}(P\Vert\N)
		=\frac{1}{M}\sum_{i=1}^{M}
		\left[\log p(X_i)-\log n(X_i)\right]$.
		We use $M=300{,}000$  Monte Carlo samples. To reduce the variance of the estimated KL differences, 
		we reuse the same samples from $P$ across all perturbation settings.
		The resulting estimate of $\mathrm{KL}(P\Vert\N_1)$ is approximately $2.2300$.

		Fig.~\ref{fig:1D_KL_eps}\subref{fig:1D_kl_difference_vs_sqrt_epsilon} in the supplementary material plots
		$\widehat{\Delta}_{\mathrm{KL}}$
		against
		\(
		\sqrt{\mathrm{KL}(\N_1\Vert\N_2)}
		\), \textit{i.e.}, $\sqrt{\e}$. Each dot on a curve corresponds to a different value of the perturbation factor $\alpha$.
		The results show that the magnitude of $\widehat{\Delta}_{\mathrm{KL}}$
		depends on the perturbation magnitude. 
		Increasing the variance of $\N_2$ decreases $\mathrm{KL}(P\Vert\N_2)$ because the bimodal distribution $P$ places more probability mass away from the origin than the standard Gaussian distribution. Nevertheless, along all perturbation paths, the KL difference remains controlled when $\N_2$ is close to $\N_1$.

		\begin{table}[!t]
			\centering
			\scriptsize
			\caption{Results of the one-dimensional synthetic experiment at \(\alpha=1\). MC s.e. denotes the Monte Carlo standard error in estimating $\widehat{\Delta}_{KL}$.}
			\label{tbl:1d_results}
			\begin{tabular}{lrrrrrrr}
				\toprule
				Perturbation & $\e$ & $\widehat{\Delta}_{KL}$ & $\widehat{B}(\e)$ & $\mathrm{TR}$ & $\widehat{K'}$ & $K'$ & MC s.e. \\ 
				\midrule
				mean-plus & 0.08  & -0.08  & 4.66  & 0.02  & -0.28  & 13.18  & 0.0016  \\ 
				mean-minus & 0.08  & -0.08  & 4.66  & 0.02  & -0.29  & 13.18  & 0.0016  \\ 
				sigma-expand & 0.07  & 0.83  & 3.69  & 0.22  & 3.03  & 9.65  & 0.0007  \\ 
				sigma-shrink & 0.07  & -1.30  & 3.61  & 0.36  & -4.98  & 11.30  & 0.0010  \\ 
				both-plus & 0.06  & 0.60  & 3.45  & 0.17  & 2.55  & 11.53  & 0.0009  \\ 
				both-minus & 0.08  & -0.95  & 4.52  & 0.21  & -3.32  & 12.84  & 0.0017 \\ 
				\bottomrule
			\end{tabular}
		\end{table}

		Table~\ref{tbl:1d_results} summarizes the quantitative comparison at the largest perturbation level ($\alpha=1$). 
		Fig. \ref{fig:bound-against-e-1d} in the supplementary material plots $\widehat{B}(\e)$ as a function of $\e$.
		First, all perturbation cases satisfy
		$|\widehat{\Delta}_{KL}| < \widehat{B}(\epsilon)$,
		which is consistent with the proposed bound in the one-dimensional setting. 
		The results also show that the theoretical bound is generally conservative. 
		Specifically, the tightness ratio ($\mathrm{TR}$) indicates that the observed KL difference is approximately $2\%$ -- $36\%$ 
		of the derived upper bound. 
		Such conservativeness is expected because the bound is established for arbitrary distributions with finite second moments and 
		involves several worst-case inequalities, such as the operator norm-based bounds in Equation~\eqref{eq:thmF123}. 
		Therefore, the theoretical constant $K'$ may overestimate the actual sensitivity for a specific distribution $P$.
		
		Furthermore, the values of $|\widehat{K'}|$ are consistently smaller than the corresponding theoretical coefficient $K'$. 
		For example, the sigma-shrink perturbation yields the largest $|\widehat{K'}|=4.98$ and $\mathrm{TR}=0.36$.
		These results are consistent with the bound established by the theorem.

		Overall, this one-dimensional experiment provides empirical support for the stability behavior predicted by Theorem~\ref{thm:triangle-P-N1-N2}.

		\subsubsection{Two-Dimensional Synthetic Experiment}

		We next conduct a two-dimensional synthetic experiment to illustrate the stability of KL divergence under Gaussian perturbations. We choose the reference distribution to be the standard Gaussian
		$\N_1=\N(\mathbf{0},\mathbf{I}_2)$, where $\mathbf{I}_2$ denotes the $2\times2$ identity matrix. We denote the density of $\N_1$ by $n_1(x)$ for $x=(x_1,x_2)^\top\in\mathbb{R}^2$.
		
		We define the non-Gaussian distribution $P$ as a mixture of two Gaussian components truncated to the rectangular region
		\(
		A=[L,U]\times[L,U],
		\)
		where
		\(
		L=-4,\qquad U=4.
		\)
		For the $k$th component, we use an axis-aligned Gaussian distribution with mean
		\(
		m_k=(m_{k1},m_{k2})^\top
		\)
		and coordinatewise standard deviations
		\(
		s_k=(s_{k1},s_{k2})^\top.
		\)
		We set the mixture weights to $w_1=0.55$ and $w_2=0.45$, and choose the component parameters as
		\begin{equation}
			\begin{aligned}
				m_1&=(-2.0,-1.4)^\top, &
				s_1&=(0.45,0.75)^\top,\\
				m_2&=(2.1,1.5)^\top, &
				s_2&=(0.65,0.40)^\top.
			\end{aligned}
		\end{equation}
		Because $P$ is supported on the bounded region $A$, it has a finite second moment. Further details of its density are provided in Appendix~\ref{sec:detail_2D_synthetic_exp} of the supplementary material.
		
		Starting from \(\N_1\), we generate perturbed Gaussian distributions $\N_2=\mathcal{N}(\mu_2,\Sigma_2)$ using the perturbations summarized in Table~\ref{tab:2d_perturbation}, 
		where the scalar parameter $\alpha\in[0,1]$ controls the perturbation strength. 
		Formal definitions of these perturbations are provided in Appendix~\ref{sec:detail_2D_synthetic_exp} of the supplementary material.
		Fig.~\ref{fig:2D_P_N1_N2} in the supplementary material visualizes the density contours of $P$, $\N_1$, and a representative perturbed Gaussian $\N_2$.
		\begin{table}[!t]
			\centering
			\scriptsize
			\caption{Perturbations to $\N_1=\mathcal{N}(\mathbf{0},\mathbf{I}_2)$.}
			\label{tab:2d_perturbation}
			\begin{tabular}{ll}
				\toprule
				Perturbation & Description\\ 
				\midrule
				mean-x & shift rightward along the $x$-axis\\ 
				mean-diag & shift diagonally downward and rightward\\ 
				cov-x-expand & expand along the $x$-axis\\ 
				cov-y-shrink & contract along the $y$-axis\\ 
				cov-diag-expand & expand along a diagonal direction\\ 
				cov-rot-anisotropic & apply rotated anisotropic expansion and contraction\\ 
				both-mean-cov-rot & jointly perturb the mean and rotated covariance\\ 
				\bottomrule
			\end{tabular}
		\end{table}

		As shown in Fig.~\ref{fig:2D_KL_eps}\subref{fig:2d_kl_vs_perturbation} in the supplementary material, $\mathrm{KL}(\N_1\Vert\N_2)$ remains below $1/12$ along all perturbation paths, 
		with a maximum value of $0.08$. 
		Thus, all experimental settings satisfy the small perturbation condition in Theorem~\ref{thm:triangle-P-N1-N2}. 
		The \textit{cov-x-expand} and \textit{cov-diag-expand} curves overlap because $\N_1$ is isotropic and $\mathrm{KL}(\N_1\Vert\N_2)$ is invariant to orthogonal rotations of $\N_2$.

		Similarly, we use $M=300{,}000$ Monte Carlo samples to estimate $\mathrm{KL}(P\Vert\N_1)$ and $\mathrm{KL}(P\Vert\N_2)$. 
		%We reuse the same samples across all perturbation settings to reduce the variance of the estimated KL differences. 
		The resulting estimate of $\mathrm{KL}(P\Vert\N_1)$ is approximately $2.9953$.
		
		We report the same evaluation quantities as in the one-dimensional synthetic experiment.
		Fig.~\ref{fig:2D_KL_eps}\subref{fig:2D_kl_difference_vs_sqrt_epsilon} in the supplementary material plots
		\(
		\widehat{\mathrm{KL}}(P\Vert\N_1)-\widehat{\mathrm{KL}}(P\Vert\N_2)
		\)
		against $\sqrt{\mathrm{KL}(\N_1\Vert\N_2)}$. Each curve varies smoothly with the perturbation magnitude $\alpha$. 
		The sign of the KL change depends on the perturbation scheme.
		
		\begin{table}[!t]
			\scriptsize
			\caption{Results of the two-dimensional synthetic experiment at \(\alpha=1\). MC s.e. is the Monte Carlo standard error in estimating $\widehat{\Delta}_{KL}$.}
			\label{tab:2d_synthetic_results}
			\begin{tabular}{lcrcrccc}
				\toprule
				Perturbation & $\e$ & \(\widehat{\Delta}_{KL}\) & $\widehat{B}(\e)$ & $\mathrm{TR}$  & $\widehat{K'}$ & $K'$ & MC s.e. \\ 
				\midrule
				mean-x & 0.08  & -0.14  & 5.78  & 0.02  & -0.50  & {\color{black}16.97}  & 0.0015  \\ 
				mean-diag & 0.08  & -0.10  & 5.70  & 0.02  & -0.34  & {\color{black}16.94}  & 0.0005  \\ 
				cov-x-expand & 0.03  & 0.49  & 3.50  & 0.14  & 2.94  & {\color{black}18.15}  & 0.0006  \\ 
				cov-y-shrink & 0.02  & -0.28  & 3.24  & 0.09  & -1.79  & {\color{black}18.14}  & 0.0006  \\ 
				cov-diag-expand & 0.03  & 0.77  & 3.50  & 0.22  & 4.68  & {\color{black}18.15}  & 0.0008  \\ 
				cov-rot-anisotropic & 0.03  & 0.72  & 4.61  & 0.16  & 3.95  & {\color{black}21.37}  & 0.0007  \\ 
				both-mean-cov-rot & 0.06  & 0.53  & 6.78  & 0.08  & 2.09  & {\color{black}22.10}  & 0.0008 \\ 
				\bottomrule
			\end{tabular}
		\end{table}
		
		Table~\ref{tab:2d_synthetic_results} summarizes the quantitative results of the two-dimensional synthetic experiment at $\alpha=1$. 
		Fig. \ref{fig:bound-against-e-2d} in the supplementary material plots $\widehat{B}(\e)$ as a function of $\e$.
		Similar to the one-dimensional case, all perturbation settings satisfy
		$|\widehat{\Delta}_{KL}| < B(\epsilon)$.
		Specifically, the $\mathrm{TR}$ values range from $0.02$ to $0.22$.
		This conservativeness is expected because the constant $K'$ is derived for arbitrary distributions with finite second moments and involves several worst-case norm inequalities. 
		Consequently, $K'$ can be substantially larger than the empirical  $|\widehat{K'}|$ for specific perturbation directions.

		Furthermore, the empirical coefficients $\widehat{K'}$ reveal that the sensitivity of $\mathrm{KL}(P\Vert\mathcal{N}_2)$ 
		depends on the perturbation type. 
		For example, the cov-diag-expand perturbation achieves the largest
		$\widehat{K'}=4.68$ and the highest $\mathrm{TR}=0.22$, suggesting that changes in the covariance structure have a stronger impact on the KL difference among these settings.
		
		\subsection{Validation on Two-Dimensional Flow-Based Models}\label{sec:exp_flow}
		This experiment demonstrates that the proposed stability bound remains meaningful when $P$ is obtained from real deep generative models 
		rather than analytically constructed distributions.
		\subsubsection{Theoretical Analysis}\label{sec:analysis_in_flow}
		Let $P_X$ and $Q_X$ denote the ID and OOD data distributions, respectively. Let $P_Z$ and $Q_Z$ denote the corresponding latent distributions induced by an invertible flow, and let $P_Z^r$ denote the standard Gaussian prior.
		
		In typical OOD detection scenarios, ID and OOD data distributions have little or no overlap. This motivates the following assumption.
		\begin{assumption}\label{assumption1}
			\textit{$\mathrm{KL}(Q_X\Vert P_X)$ is sufficiently large.}
		\end{assumption}
		This assumption is weaker than Assumption~\ref{assumption0}, which requires large KL divergences in both directions. Because an invertible transformation preserves KL divergence, we have
		$\mathrm{KL}(Q_Z\Vert P_Z)=\mathrm{KL}(Q_X\Vert P_X)$.
		Therefore, $\mathrm{KL}(Q_Z\Vert P_Z)$ is also sufficiently large.
		
		The training objective of a flow-based model minimizes $\mathrm{KL}(P_Z\Vert P_Z^r)$~\cite{flow-tpami-2021}. 
		In the experiments below, the normality tests do not reject multivariate normality of the ID representations.
		This observation motivates the use of a fitted Gaussian distribution $\widehat{P}_Z$ as a proxy for $P_Z$.
		Thus, we establish the following assumption.
		\begin{assumption}\label{assumption2}
			\textit{Let $P_Z$ denote the ID latent distribution and $\widehat{P}_Z$ its fitted Gaussian proxy. 
				For some small $\eta\geq 0$, assume that
				$|\mathrm{KL}(Q_Z\Vert P_Z)-\mathrm{KL}(Q_Z\Vert \widehat{P}_Z)|\leq \eta$.}
		\end{assumption}
		
		%In the experiments below, we fit the Gaussian proxy $\widehat{P}_Z$ to the ID latent representations and find that
		%$\mathrm{KL}(\widehat{P}_Z\Vert P_Z^r)<1/12$.

		In addition, inputs to flow-based models are typically normalized to a bounded interval, such as $[-0.5,0.5]^d$. 
		Because the model consists of finitely many layers, each with a finite Lipschitz constant, 
		the overall flow transformation is also Lipschitz continuous. 
		It therefore maps the bounded input domain to a bounded latent domain.
		This implies that the latent distribution has a finite second moment. 
		A detailed derivation is given in Proposition~\ref{thm:pro_flow_finite_moment} of Appendix~\ref{sec:proof_second_moment} of the supplementary material.

		Under Assumptions~\ref{assumption1} and~\ref{assumption2}, 
		we apply Theorem~\ref{thm:triangle-P-N1-N2} with $P=Q_Z$, $\N_1=\widehat{P}_Z$, 
		and $\N_2=P_Z^r$. This yields
		\begin{align}\label{eq:oneside-KL-bound-flow}
			\mathrm{KL}(Q_Z\Vert P_Z^r) &\geq \mathrm{KL}(Q_Z\Vert \widehat{P}_Z)-B(\e)\\
			& \geq \mathrm{KL}(Q_Z\Vert P_Z)-\eta -B(\e)
		\end{align}
		Thus, if $\mathrm{KL}(Q_Z\Vert P_Z)$ is sufficiently large, $\mathrm{KL}(Q_Z\Vert P_Z^r)$ remains large.
		
		The result provides a one-sided stability guarantee for the KL score
		$\mathrm{KL}(Q_Z\Vert P_Z^r)$ used in previous KL-based OOD analysis \cite{zhang2023outofdistribution}.
		Importantly, this result relaxes the strong assumption $P_Z\approx P_Z^r$ used in ~\cite{zhang2023outofdistribution}, 
		where the model-fitting error is ignored.

		\textbf{Note.}
		The directions of the KL divergences considered in this subsection differ from those 
		in the motivating problem presented in Section~\ref{problem}. 
		A theoretical guarantee for the reverse direction $\mathrm{KL}(P_Z^r\Vert Q_Z)$ requires additional assumptions and is outside the scope of the theorem.
		Nevertheless, our theoretical results remain applicable to the OOD detection analysis. 
		Further details are provided in Appendix~\ref{sec:appendix-exp-nsf-details} of the supplementary material.

		\subsubsection{Validation on Two-Dimensional Neural Spline Flows}
		\label{sec:exp-2d-flow}
		
		We further evaluate Theorem~\ref{thm:triangle-P-N1-N2} empirically using two-dimensional Neural Spline Flow (NSF) models~\cite{NSF}.

		\textbf{Datasets}. 
		We use four two-dimensional datasets: \textit{abs}, \textit{crescent}, \textit{gaussian}, and \textit{sine\_wave}. 
		Fig.~\ref{fig:all-four-datasets} in the supplementary material shows these datasets.
		Each dataset contains $10^6$ training samples and $10^5$ test samples.
		Their distinct geometric structures suggest substantial differences among the underlying distributions.

		\textbf{Model}. 
		We use a two-dimensional NSF model~\cite{NSF} with the standard Gaussian prior $\N_2=\N(\mathbf{0},\mathbf{I}_2)$. 
		Our implementation is based on the official implementation released in~\cite{NSF}. 
		For each dataset, we train a separate NSF model with that dataset as the ID dataset and treat the remaining three as OOD datasets. 
		The details of the model and training process are presented in Appendix \ref{sec:appendix-exp-nsf-details} of the supplementary material.

		\textbf{Results and Analysis.}
		The learned densities qualitatively reproduce the main geometric structures of the corresponding training distributions.
		For example, Fig.~\ref{fig:nsf_abs} in the supplementary material shows the density learned by the flow trained on \textit{abs}, 
		together with samples generated from the model. 
		Both the learned density and the generated samples reflect the characteristic structure of the \textit{abs} dataset.
		
		Fig.~\ref{fig:nsf_abs_vs_others_representations} in the supplementary material shows the latent representations of all 
		four datasets under the flow trained on \textit{abs}. The ID samples are mapped to an approximately Gaussian distribution, 
		whereas latents from the other datasets concentrate near one-dimensional curves and form visibly non-Gaussian structures. 
		The separation between the ID and OOD distributions in both the data and latent spaces provides qualitative support 
		for Assumption~\ref{assumption1} for each ID--OOD pair.
		
		As shown in Table~\ref{tbl:ID_2d_flow} in the supplementary material, the Henze--Zirkler normality tests \cite{henze1990class} yield $p>0.05$ for the latent representations 
		of each ID dataset. Thus, the null hypothesis of multivariate normality is not rejected, providing empirical support 
		for approximating each aggregated ID latent distribution by a multivariate Gaussian distribution. 
		We then fit a Gaussian distribution $\widehat{P_Z}$ to the ID latent representations and compute its KL divergence 
		from the standard Gaussian prior. 
		For all ID datasets, $\mathrm{KL}(\widehat{P_Z}\Vert \mathcal{N}(\mathbf{0},\mathbf{I}_2))<10^{-4}$. 
		These results show that the small perturbation condition in Theorem \ref{thm:triangle-P-N1-N2} is applicable in the experiment.
		They also provide a proxy-based diagnostic for $P_Z$, 
		as they establish an empirical lower bound on $\mathrm{KL}(P_Z\Vert P_Z^r)$ \cite{Statistical_Inference_divergence}.
		
		In the experiments, the density of $Q_Z$ is unknown. Nevertheless, the difference
		$\widehat{\Delta}_{\mathrm{KL}}$
		can be estimated without explicitly estimating this density, 
		because the terms involving the density of $Q_Z$ cancel when the two KL divergences are subtracted; see Equation~\eqref{equ:diffKL}.

		Table~\ref{tab:2d-flow-ood} summarizes the estimated quantities $\widehat{\Delta}_{\mathrm{KL}}$, $\widehat{B}(\e)$, 
		$\mathrm{TR}$, $\widehat{K'}$, and $K'$ defined in Equations \eqref{equ:diffKL} -- \eqref{equ:K'}.
		Since the OOD latent distribution can be arbitrary, it is difficult to estimate $\mathrm{KL}(Q_Z\Vert \N(\mathbf{0},\mathbf{I}_2))$ directly.
		Nevertheless, we fit a Gaussian distribution $\widehat{Q_Z}$ to the OOD representations and compute $\mathrm{KL}(\widehat{Q_Z}\Vert \N(\mathbf{0},\mathbf{I}_2))$ using the closed form. 
		Theoretically, $\mathrm{KL}(\widehat{Q_Z}\Vert \N(\mathbf{0},\mathbf{I}_2))$ provides a lower bound on $\mathrm{KL}(Q_Z\Vert \N(\mathbf{0},\mathbf{I}_2))$ \cite{zhang2023outofdistribution,Statistical_Inference_divergence}.
		The last column of Table~\ref{tab:2d-flow-ood} shows $\mathrm{KL}(\widehat{Q_Z}\Vert \N(\mathbf{0},\mathbf{I}_2))$.
		
		Moreover, the $p$-values from the Henze--Zirkler normality tests of the OOD representations are below $10^{-10}$ 
		in all cases, providing strong evidence for non-Gaussianity. 
		The Monte Carlo standard errors in estimating $\widehat{\Delta}_{\mathrm{KL}}$ are less than $10^{-4}$ for every case.

		The results show that the observed KL differences satisfy
		$|\widehat{\Delta}_{KL}| < \widehat{B}(\epsilon)$ in all evaluated cases, providing empirical support for the proposed stability guarantee beyond fully Gaussian settings. 
		The tightness ratio ($\mathrm{TR}$) varies across different ID distributions. 
		For the \textit{abs} and \textit{crescent} distributions, the $\mathrm{TR}$ values are larger, ranging from $0.26$ to $0.36$, 
		indicating that the theoretical bound provides a tighter approximation in these cases. 
		In contrast, when the ID distribution is \textit{gaussian} or \textit{sine\_wave}, 
		the $\mathrm{TR}$ values are generally smaller (below $0.06$), suggesting that the corresponding theoretical bounds are more conservative.
		This difference can be partially attributed to the scale of $\widehat{\Delta}_{KL}$. 
		Across the four ID datasets, the theoretical constants $K'$ are of the same order of magnitude. 
		However, the magnitudes of $\widehat{\Delta}_{KL}$ are considerably smaller when \textit{gaussian} or \textit{sine\_wave} is used as the ID dataset. 
		Consequently, these two ID datasets yield smaller $\mathrm{TR}$ values.
		
		Additionally, the values of $\mathrm{KL}(\widehat{Q}_Z\Vert \N(\mathbf{0},\mathbf{I}_2))$ are several orders of magnitude larger than the corresponding values of $\mathrm{KL}(\widehat{P}_Z\Vert \N(\mathbf{0},\mathbf{I}_2))$ reported in Table~\ref{tbl:ID_2d_flow} in the supplementary material. 
		Note that this quantity provides a lower bound on the unknown $\mathrm{KL}(Q_Z\Vert \N(\mathbf{0},\mathbf{I}_2))$. 
		These results are consistent with the conclusion drawn from the theoretical analysis (see Equation~\eqref{eq:oneside-KL-bound-flow}).
		
		\begin{table}[!t]
			\centering
			\scriptsize
			\caption{Results on a two-dimensional NSF. All MC s.e. values are less than $10^{-4}$. $\widehat{\mathrm{KL}}_{\widehat{Q}}$ is $\mathrm{KL}(\widehat{Q}_Z\Vert\N(\mathbf{0},\mathbf{I}_2))$.}
			\label{tab:2d-flow-ood}
			\begin{tabular}{llrccrcc}
				\toprule
				ID & OOD & $\widehat{\Delta}_{KL}$ & $\widehat{B}(\e)$ & $\mathrm{TR}$ & $\widehat{K'}$ & $K'$  & $\widehat{\mathrm{KL}}_{\widehat{Q}}$ \\ 
				\midrule
				\textit{abs} & \textit{crescent} & 0.0144  & 0.04  & 0.33  & 5.25  & {\color{black}15.79}  & 2.18  \\ 
				& \textit{gaussian} & 0.0088  & 0.03  & 0.26  & 3.23  & {\color{black}12.44}  & 1.10  \\ 
				& \textit{sine\_wave} & 0.0142  & 0.04  & 0.32  & 5.18  & {\color{black}15.92}  & 2.03  \\ 
				\midrule
				\textit{crescent} & \textit{abs} & -0.0294  & 0.09  & 0.34  & -4.02  & {\color{black}11.71}  & 1.53  \\ 
				& \textit{gaussian} & -0.0315  & 0.09  & 0.36  & -4.31  & {\color{black}11.97}  & 1.79  \\ 
				& \textit{sine\_wave} & -0.0258  & 0.09  & 0.30  & -3.53  & {\color{black}11.78}  & 1.22  \\ 
				\midrule
				\textit{gaussian} & \textit{abs} & -0.0004  & 0.03  & 0.01  & -0.09  & {\color{black}7.29}  & 0.15  \\ 
				& \textit{crescent} & 0.0013  & 0.05  & 0.03  & 0.28  & {\color{black}9.97}  & 0.84  \\ 
				& \textit{sine\_wave} & 0.0004  & 0.04  & 0.01  & 0.09  & {\color{black}8.51}  & 0.30  \\ 
				\midrule
				\textit{sine\_wave} & \textit{abs} & 0.0008  & 0.02  & 0.03  & 0.18  & {\color{black}5.58}  & 0.07  \\ 
				& \textit{crescent} & -0.0011  & 0.03  & 0.03  & -0.25  & {\color{black}7.61}  & 0.32  \\ 
				& \textit{gaussian} & 0.0013  & 0.02  & 0.06  & 0.31  & {\color{black}4.83}  & 0.19 \\ 
				\bottomrule
			\end{tabular}
		\end{table}
		
		Overall, the results obtained with the two-dimensional flow-based models provide empirical support for Theorem~\ref{thm:triangle-P-N1-N2}.

		\subsection{Groupwise Validation on Higher-Dimensional Flow-Based Models}
		This experiment provides a groupwise empirical validation of the proposed KL stability guarantee in high-dimensional flow-based models.
		
		The theoretical analysis in Section~\ref{sec:exp_flow} requires $\mathrm{KL}(\widehat{P}_Z\Vert P_Z^r)\leq 1/12$. However, modern flow-based models are often trained on data with thousands of dimensions. For a $d$-dimensional distribution $P$, $\mathrm{KL}(P\Vert \N(\mathbf{0},\mathbf{I}_d))$ can be decomposed as~\cite{zhang2023outofdistribution}
		%\begin{align}\label{equ:decompose_KL_basic}
		$\mathrm{KL}(P\Vert \prod_{i=1}^{d} P_i)
		+
		\sum_{i=1}^{d}
		\mathrm{KL}(P_i\Vert \N(0,1))$,
		%\end{align}
		where $P_i$ denotes the marginal distribution of the $i$th dimension, $\mathrm{KL}(P\Vert \prod_{i=1}^d P_i)$ is the total correlation, and $\mathrm{KL}(P_i\Vert \N(0,1))$ is the marginal KL divergence.
		Because all terms in this decomposition are nonnegative, the condition $\mathrm{KL}(P\Vert \N(\mathbf{0},\mathbf{I}_d))<1/12$ requires
		$\sum_{i=1}^d \mathrm{KL}(P_i\Vert \N(0,1))<1/12$.
		Consequently, the condition becomes increasingly restrictive as $d$ grows.
		
		Nevertheless, we can partition the dimensions into groups and illustrate the numerical behavior predicted by Theorem \ref{thm:triangle-P-N1-N2} separately for each group.
		Specifically, suppose that the $d$ dimensions are partitioned into $k$ disjoint groups, each containing $d/k$ dimensions. 
		The KL divergence $\mathrm{KL}(P\Vert \N(\mathbf{0},\mathbf{I}_d))$ can be decomposed as 
		$
		%\mathrm{KL}(P\Vert \N(\mathbf{0},\mathbf{I}_d)))
		%=
		\mathrm{KL}(P\Vert \prod_{i=1}^{k}P_i^{(g)})
		+
		\sum_{i=1}^{k}
		\mathrm{KL}(P_i^{(g)}\Vert \N(\mathbf{0},\mathbf{I}_{d/k})),
		$
		where $P_i^{(g)}$ denotes the marginal distribution of the $i$-th dimension group, $\mathrm{KL}(P\Vert \prod_{i=1}^{k}P_i^{(g)})$ is the generalized total correlation, and
		$\mathrm{KL}(P_i^{(g)}\Vert\N(\mathbf{0},\mathbf{I}_{d/k}))$ is the groupwise KL divergence.
		Although the total KL divergence $\mathrm{KL}(P\Vert \N(\mathbf{0},\mathbf{I}_d))$ may exceed $1/12$, 
		the KL divergence for an individual group may still be less than $1/12$. 
		Accordingly, we can evaluate the bound predicted by Theorem \ref{thm:triangle-P-N1-N2} on each group individually.

		\textbf{Datasets.} We use three widely adopted datasets, CelebA32~\cite{celeba}, SVHN~\cite{svhn}, and CIFAR-10~\cite{cifar}, 
		as the ID training datasets, and ImageNet32~\cite{imagenet32}, CIFAR-100~\cite{cifar}, and LSUN~\cite{LSUN} as the OOD datasets. 
		All images have shape $32\times32\times3$. We use the official training and test splits and select 10,000 samples from each test set.
		
		\textbf{Model.} We use the official Glow architecture~\cite{glowopenai}, 
		which adopts a multiscale architecture with three stages and 32 flow steps per stage. 
		Details of the model are presented in Appendix \ref{sec:detail_higher_d_flow} of the supplementary material.
		
		\textbf{Results and Analysis.}
		We partition each latent representation into four-dimensional groups corresponding to $2\times2$ blocks. 
		Each latent representation has $32\times32\times3$ dimensions, so this partition yields 768 groups.
		For each group $g$, we apply the Henze--Zirkler multivariate normality test to the ID representations. 
		We then fit a Gaussian distribution, denoted by $\widehat{P}_Z^{(g)}$, to the representations in that group and 
		compute $\mathrm{KL}(\widehat{P}_Z^{(g)}\Vert \N(\mathbf{0},\mathbf{I}_4))$. 
		Experimental results show that, for  %more than
		{\color{black}$91.1\%$--$94.4\%$} of all groups, the null hypothesis of multivariate normality is not rejected ($p>0.05$).
		For {\color{black}$98.6\%$--$99.3\%$} of all groups the fitted Gaussian satisfies $\mathrm{KL}(\widehat{P}_Z^{(g)}\Vert \N(\mathbf{0},\mathbf{I}_4))<1/12$.
		See Table~\ref{tbl:ratio_group} in Appendix \ref{sec:detail_higher_d_flow} of the supplementary material for details.

		We then evaluate each trained model on the OOD datasets and collect the corresponding OOD latent representations. For each group $g$ satisfying
		$\mathrm{KL}(\widehat{P}_Z^{(g)}\Vert \N(\mathbf{0},\mathbf{I}_4))<1/12$,
		let $Q_Z^{(g)}$ denote the marginal distribution of the OOD representations in that group, let
		$\N_1=\widehat{P}_Z^{(g)}$ denote the Gaussian distribution fitted to the corresponding ID representations, and let
		$\N_2=\N(\mathbf{0},\mathbf{I}_4)$ denote the groupwise standard Gaussian prior. 
		Using Equations~\eqref{equ:diffKL} -- ~\eqref{equ:K'}, we estimate
		$\widehat{\Delta}_{\mathrm{KL}}=\widehat{\mathrm{KL}}(Q_Z^{(g)}\Vert\N_1)-\widehat{\mathrm{KL}}(Q_Z^{(g)}\Vert\N_2)$, the bound $\widehat{B}(\e)$,
		the tightness ratio $\mathrm{TR}$, and the corresponding empirical coefficient $\widehat{K^{\prime}}$ for each dimension group.

		\begin{figure}[t!]
			\scriptsize
			\centering
			\subfloat[\textit{{\scriptsize ID: CelebA32}}]{
				\begin{minipage}[t]{8cm}
					\centering
					\includegraphics[width=8cm]{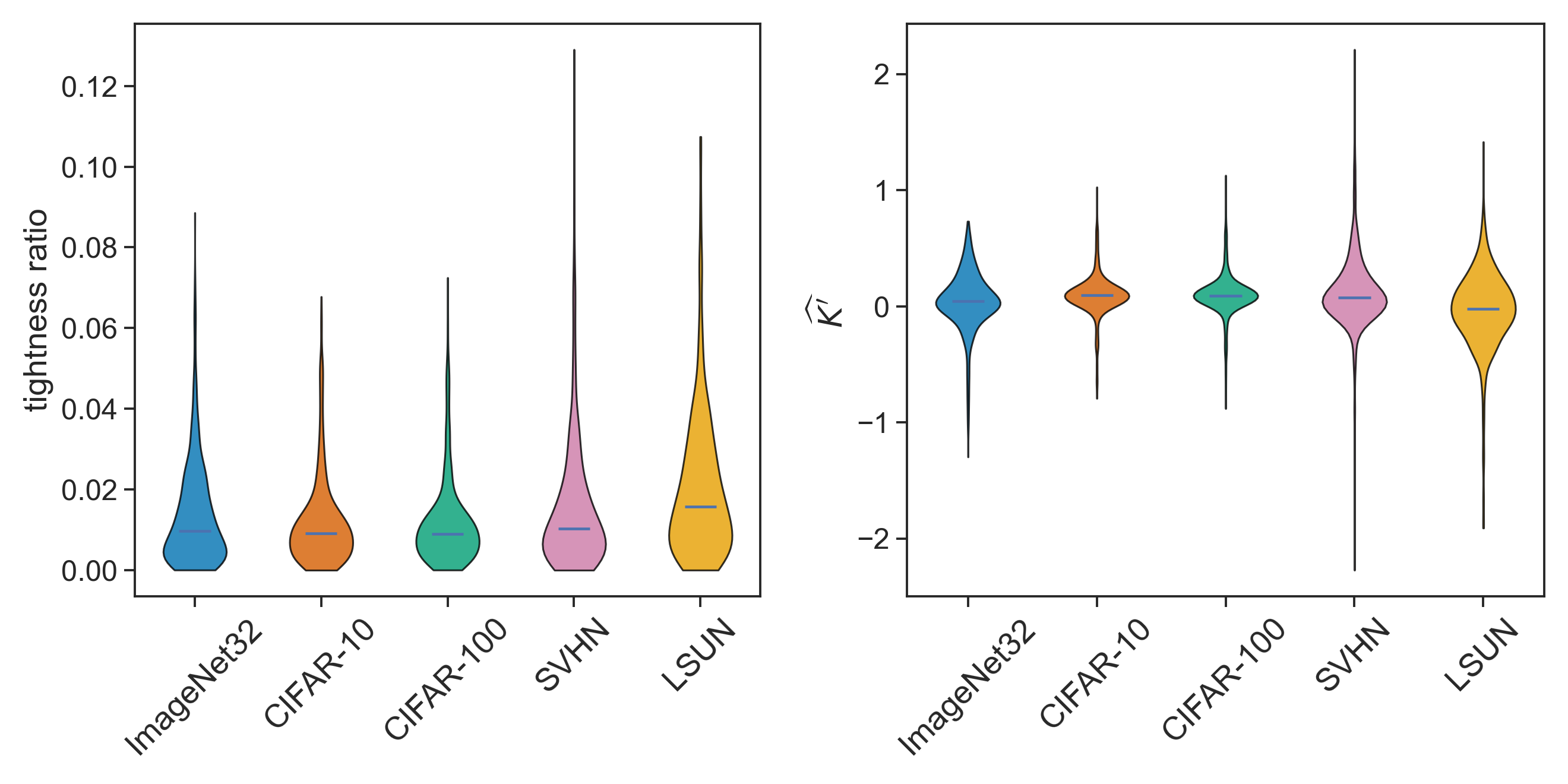}
					\vspace{-5pt}
					\label{fig:glow_celeba32_merged_K_tightness}
				\end{minipage}
			}	
			
			\subfloat[\textit{{\scriptsize ID: SVHN}}]{
				\begin{minipage}[t]{8cm}
					\centering
					\includegraphics[width=8cm]{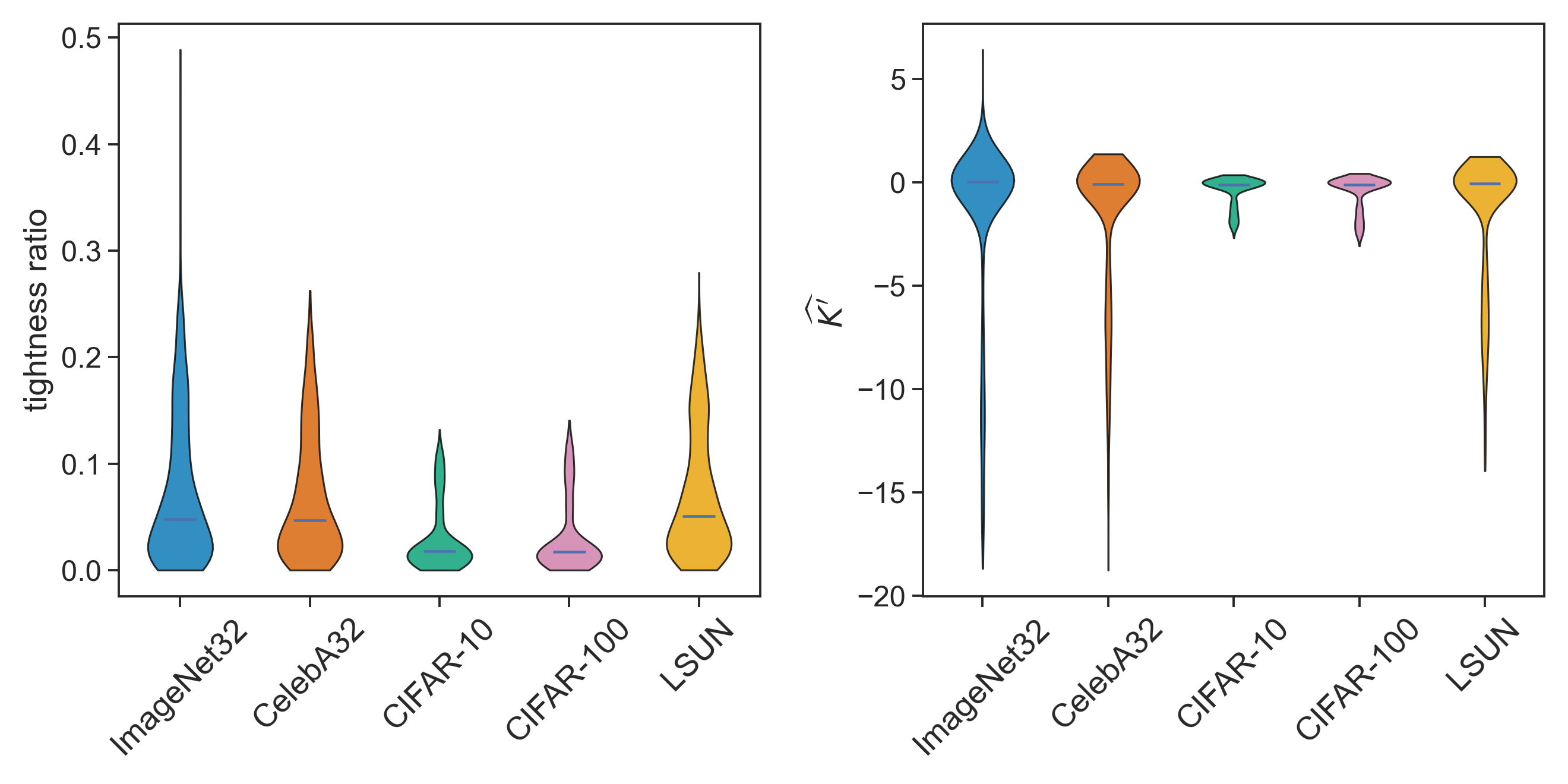}
					\vspace{-5pt}
					\label{fig:glow_svhn_merged_K_tightness}
				\end{minipage}
			}
			
			\subfloat[\textit{{\scriptsize ID: CIFAR-10}}]{
				\begin{minipage}[t]{8cm}
					\centering
					\includegraphics[width=8cm]{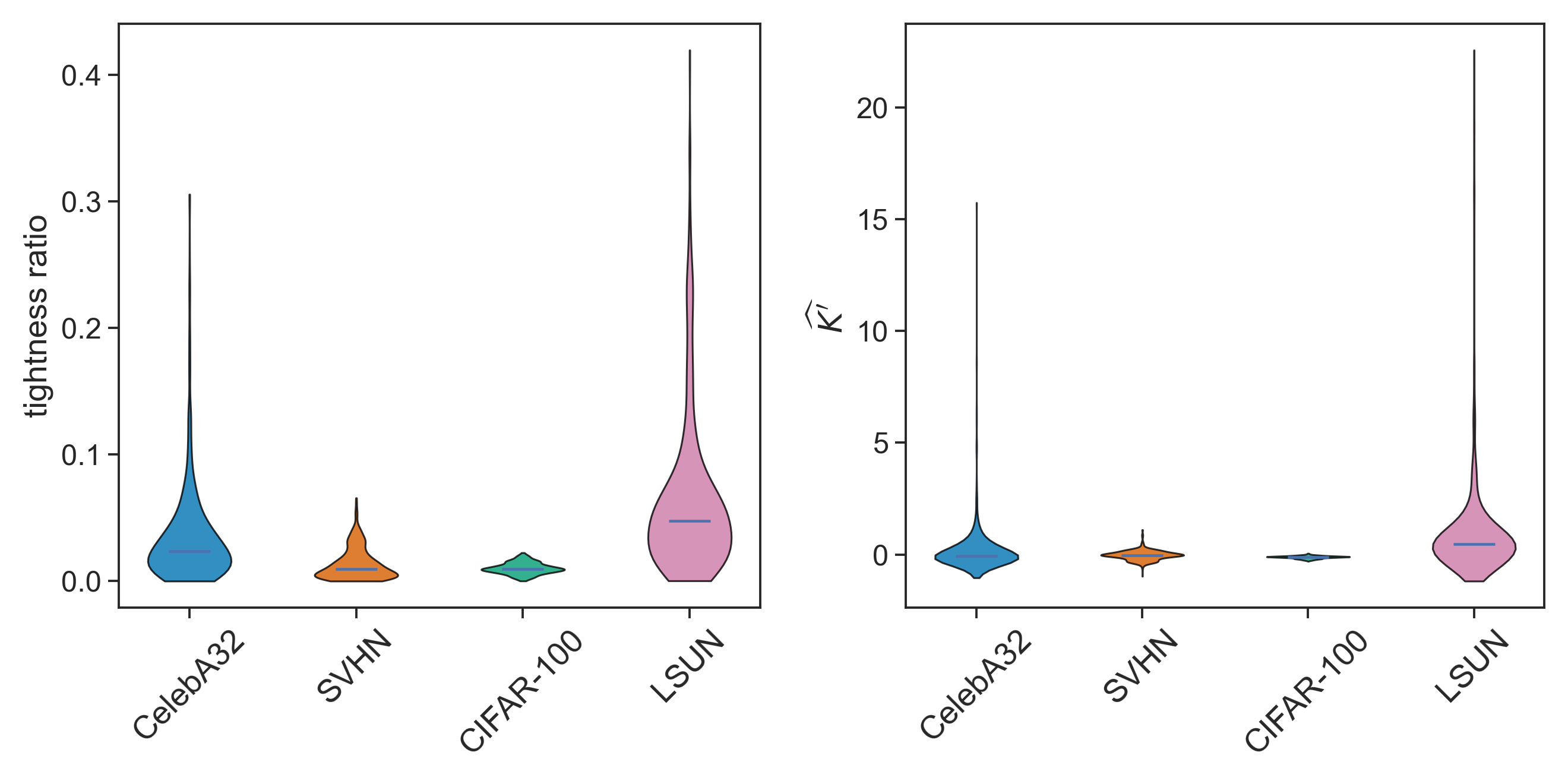}
					\vspace{-5pt}
					\label{fig:glow_cifar10_merged_K_tightness}
				\end{minipage}
			}
			\caption{Distributions of the tightness ratio $\mathrm{TR}$ and empirical coefficient $\widehat{K^{\prime}}$ across groups of Glow latent dimensions.
				Dimension groups satisfying $\mathbb{E}[\|Z\|^2]>50$ are excluded from the violin plots for readability.}
			\label{fig:glow_K_tightness}
		\end{figure}

		Fig.~\ref{fig:glow_K_tightness}  presents the distributions of the tightness ratio $\mathrm{TR}$ and the empirical coefficient $\widehat{K'}$ 
		across different dimension groups for various ID--OOD pairs. 
		Specifically, all $\mathrm{TR}$ values, including those for the groups omitted from the violin plots, are below 1,
		indicating that the empirical KL variations remain within the theoretical bound.

		The tightness of the bound varies across different datasets. 
		For the Glow model trained on CelebA32, most $\mathrm{TR}$ values are below 0.06, 
		suggesting that the theoretical bound is relatively conservative in this setting. 
		In contrast, for the Glow model trained on SVHN, the tightness ratios are generally larger, 
		with most values below 0.3 and a maximum of 0.49. 
		This difference indicates that the tightness of the bound depends on the underlying latent distributions 
		and the corresponding OOD shifts.
		
		For most dimension groups, the empirical coefficients $\widehat{K'}$ are concentrated in a relatively narrow interval around zero. 
		However, a small number of groups exhibit significantly larger absolute values. 
		To mitigate distortion in the violin plots caused by these extreme cases, we exclude any group with an estimated second moment $\mathbb{E}[\|Z\|^2]>50$, where  $Z$ denotes the OOD representations on the dimension group.
		We further investigate the source of these extreme values by analyzing the components of the theoretical constant $K'$. 
		As shown in Fig.~\ref{fig:glow_Ex2_F3_K'_ratio} in the supplementary material, 
		ID latent representations have relatively stable second moments, typically within the range of $[3,5]$. 
		In contrast, several OOD datasets contain a small proportion of dimension groups with substantially larger second moments. 
		Since the term $F_3$ in Equation~\eqref{eq:thmF123} in the supplementary material directly depends on $\mathbb{E}[\|Z\|^2]$, 
		these unusually large second moments can dominate the theoretical coefficient $K'$ and are empirically associated with larger empirical coefficients.
		
		Table~\ref{tbl:bigEX} in Appendix \ref{sec:detail_higher_d_flow} of the supplementary material further quantifies the relationship between large second moments and $\widehat{K'}$. 
		The proportion of dimension groups satisfying $\mathbb{E}[\|Z\|^2]>50$ is relatively small ($\leq 13.8\%$) for all datasets, 
		indicating that such extreme groups are rare. 
		Nevertheless, the Pearson correlation coefficients between $\widehat{K'}$ and $\mathbb{E}[\|Z\|^2]$ have large absolute values (close to 1), 
		with statistically significant $p$-values, demonstrating a strong linear association between the empirical coefficient and the second moment. 
		%These results confirm that unusually large latent magnitudes are the primary factor contributing to extreme $\widehat{K'}$ values.
		
		Note that although some dimension groups exhibit relatively large {\color{black}absolute} values of $\widehat{K'}$  and, consequently, large KL differences, 
		the theoretical bound remains valid and characterizes the magnitude of the KL variation in these cases.
		Furthermore, these large KL differences do not necessarily obscure the separation between ID and OOD latent representations. 
		For each of these groups, we fit a Gaussian distribution $\widehat{Q}_Z^{(g)}$ to the OOD representations and 
		compare $\mathrm{KL}(\widehat{Q}_Z^{(g)}\Vert\mathcal{N}(\mathbf{0},\mathbf{I}_4))$ with $\mathrm{KL}(\widehat{P}_Z^{(g)}\Vert\mathcal{N}(\mathbf{0},\mathbf{I}_4))$. 
		See Fig.~\ref{fig:glow_KL_on_large_Ex2_groups} in the supplementary material for details.
		The experimental results show that the former is consistently larger than the latter. 
		Note that $\mathrm{KL}(\widehat{Q}_Z^{(g)}\Vert\mathcal{N}(\mathbf{0},\mathbf{I}_4))$  serves only as an empirical lower bound of $\mathrm{KL}(Q_Z^{(g)}\Vert\mathcal{N}(\mathbf{0},\mathbf{I}_4))$ \cite{zhang2023outofdistribution, Statistical_Inference_divergence}.

		In summary, the experiments provide groupwise empirical evidence  for the predicted stability behavior in high-dimensional flow representations.
		Further discussion of the analysis and results is presented in Appendix \ref{sec:detail_higher_d_flow} of the supplementary material.
		
	}%end color red before experiment section
	
	\vspace{-5pt}
	{\color{black}
		\subsection{Discussion}\label{discuss}
		In the experiments on high-dimensional flow-based models, we approximate each ID latent distribution $P_Z$ (or marginal $P_Z^{(g)}$) 
		by a fitted Gaussian distribution $\widehat{P}_Z$ (or $\widehat{P}_Z^{(g)}$) 
		and assume that the corresponding approximation error is small.
		Accordingly, we assume $|\mathrm{KL}(Q_Z\Vert P_Z)-\mathrm{KL}(Q_Z\Vert \widehat{P}_Z)|\leq \eta$ for some small $\eta$.
		The Henze--Zirkler normality tests do not reject Gaussianity ($p>0.05$), 
		providing empirical motivation for this assumption but not establishing exact Gaussianity.
		Unlike previous work~\cite{zhang2023outofdistribution}, 
		this work replaces the direct approximation $P_Z\approx P_Z^r$ with an explicit Gaussian-proxy approximation error $\eta$, 
		thereby separating the proxy error from the prior-matching error (see Assumption \ref{assumption2}).
		Nevertheless, this approximation error need not vanish, and the normality tests alone do not quantify $\eta$.
		Deriving theoretical bounds on $\eta$ for deep flow-based models remains an important direction for future work.
		
		Further discussions of computational complexity and potential applications are provided in Appendix \ref{sec:appendix-discussion} of the supplementary material.
	}

	\section{Conclusion}\label{conclusion}
	
	In this paper, we establish a stability bound with the optimal worst-case $\sqrt{\e}$-rate for KL divergence between an arbitrary distribution and two nearby Gaussian reference distributions.
	Our result extends existing relaxed triangle inequalities from fully Gaussian settings to settings in which one distribution is arbitrary, 
	revealing an intrinsic stability property of KL divergence under Gaussian perturbations. 
	Experiments on synthetic data and flow-based models provide empirical support for the predicted stability behavior. 
	Future work will investigate whether analogous stability guarantees hold under weaker assumptions, such as non-Gaussian reference families or alternative divergence directions.

	\bibliographystyle{IEEEtran}
	\bibliography{main}

\setcounter{equation}{0}                      % 重置公式计数器，让编号从 1 开始
\renewcommand{\theequation}{S\arabic{equation}} % 加前缀 S

\setcounter{figure}{0}
\renewcommand{\thefigure}{S\arabic{figure}}

\setcounter{table}{0}
\renewcommand{\thetable}{S\arabic{table}}
	
\appendices
\section{Background}\label{sec:background}

\subsection{Flow-Based Model}
Flow-based models are a class of deep generative models that learn an explicit, tractable probability density through an invertible mapping between the data and latent spaces. % with a simple prior distribution, typically a multivariate Gaussian.
Let $z=f_{\theta}(x)$ denote a flow transformation with parameters $\theta$.
The transformation is typically constructed as a composition of bijective, differentiable mappings implemented by neural networks:
\begin{align}\nonumber
	&x \stackrel{f_1}{\longleftrightarrow} h_1 \stackrel{f_2}{\longleftrightarrow} h_2 \dots \stackrel{f_n}{\longleftrightarrow} z.
\end{align}
The resulting bijection is
$f=f_n\circ f_{n-1}\circ\cdots\circ f_1$.
Here, $x$ denotes an input data point, such as an image, and
$z$ is its latent representation.

By the change-of-variables formula, the data likelihood can be computed exactly as
\begin{align}
	\log p_X(x)
	=\log p^r_Z(f(x))
	+\sum_{i=1}^n\log\left|\det\left(\frac{\partial h_i}{\partial h_{i-1}}\right)\right|.
	\label{equ:change_of_var}
\end{align}
Here, $h_0=x$, $h_n=z$, and $\partial h_i/\partial h_{i-1}$ is the Jacobian matrix of $f_i$. The latent prior density $p^r_Z$ is typically chosen to be simple and tractable, with the $d$-dimensional standard Gaussian distribution $\N(\mathbf{0},\mathbf{I}_d)$ being the most common choice.

Flow-based models are typically trained by maximum likelihood estimation (MLE).
Given a dataset $\{x_j\}_{j=1}^N$, the objective is to maximize
$\sum_{j=1}^{N}\log p_{\theta}(x_j)$, where $p_{\theta}$ is the density of the model distribution $P_{\theta}$. This is equivalent to minimizing $\mathrm{KL}(P_X\Vert P_{\theta})$~\cite{flow-tpami-2021}.

Sampling from a flow-based model is straightforward and efficient. One can first draw a latent variable $z$ from the prior distribution, 
and then obtain a data sample by applying the inverse transformation $x=f^{-1}(z)$.

Representative architectures include RealNVP~\cite{dinh2016realnvp}, Glow~\cite{glowopenai}, Residual Flows~\cite{chen2019residual}, Neural Spline Flows~\cite{NSF}, and TarFlow~\cite{tarflow}. Owing to their exact likelihood evaluation and invertible latent representations, flow-based models have been widely studied for density estimation and out-of-distribution detection.

Recent studies reveal that flow-based models may assign higher model likelihood to OOD data than ID data. 
Fig. \ref{fig:logpx_cifar10_svhn_contrast_glow} shows the distribution of $\log p(x)$ of different datasets under Glow trained on CIFAR-10.
SVHN samples have higher likelihoods than CIFAR-10. However, SVHN samples with contrast adjusted by a factor of 2.0 have likelihoods comparable to those of CIFAR-10 samples. 
These results demonstrate the limitations of likelihood-based OOD detection using flow-based models.

\begin{figure}[t]
	\centering
	\includegraphics[width=6.65cm]{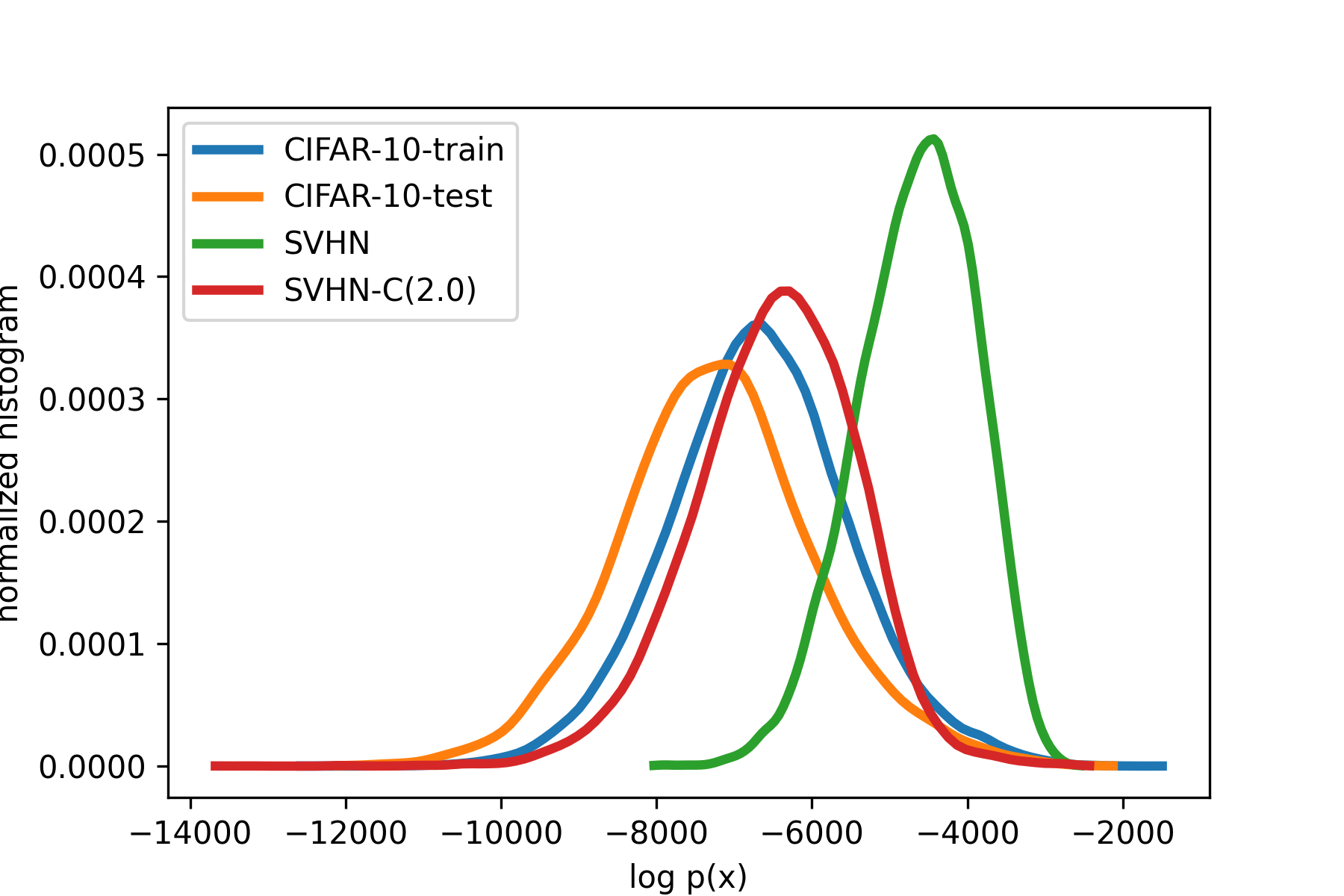}
	\caption{Distribution of model log-likelihood. Glow model trained on CIFAR-10 assigns higher likelihoods to SVHN than CIFAR-10 and comparable likelihoods to SVHN with an adjusted contrast of factor 2.0 (SVHN-C(2.0)).}
	\label{fig:logpx_cifar10_svhn_contrast_glow}
\end{figure}

\IEEEpubidadjcol

\section{Proofs and Additional Theoretical Results}\label{sec:proof-lemmas}
\subsection{Proof of Lemma \ref{lemma:f}}\label{sec:appendix_proof_lemma_f}

\begin{proof}
	Since $f''(x)=x^{-2}>0$ for all $x\in\mathbb{R}_{>0}$, the function $f$ is strictly convex. Moreover, $f'(x)=1-\frac{1}{x}$ takes the value 0 at $x=1$. Therefore, $f$ attains its unique minimum at $x=1$, with $f(1)=0$.
	The graph of $f$ is shown in Fig.~\ref{fig:f}.
\end{proof}
\begin{figure}
	\centering	
	\includegraphics[width=7cm]{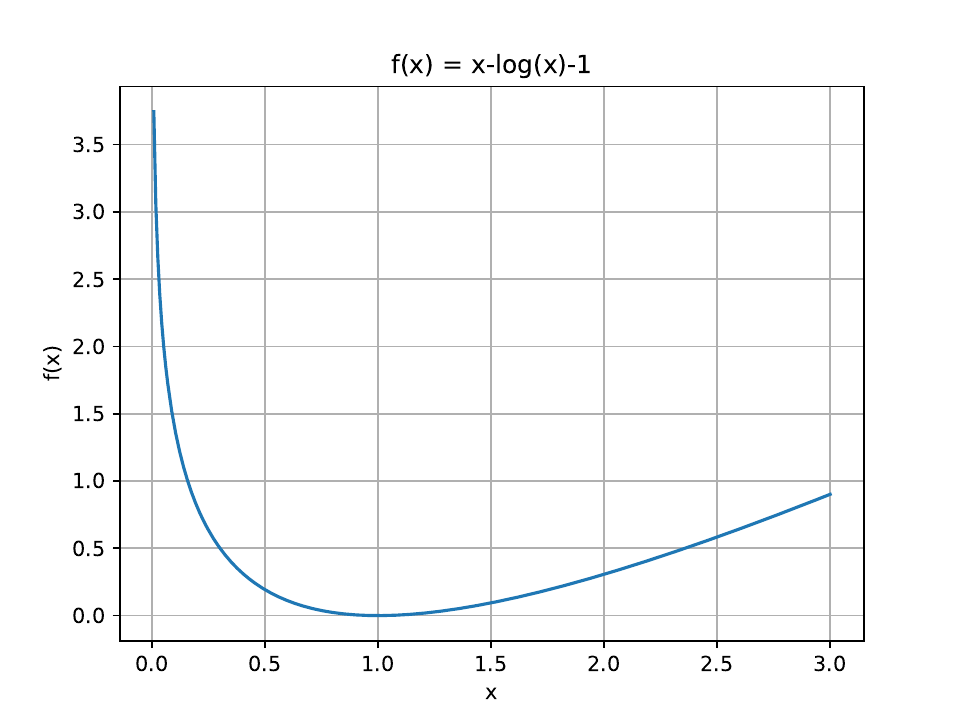}
	\caption{$f(x)=x-\log x-1$.}
	\label{fig:f}
\end{figure}
\subsection{Proof of Lemma \ref{lemma:f-lower-bound}}\label{sec:appendix-proof-lemma-f-lower-bound}
\begin{proof}
	Let $H(x)=f(x)-\frac{(x-1)^2}{3}$. It follows that $H(1)=0$.  The first-order derivative $H'(x)=-\dfrac{1}{3x}(x-1)(2x-3)$. 
	When $x\in[0.3,1]$, $H'(x)\leq 0$. Thus, $H(x)\geq 0$ for $x\in [0.3, 1]$. 
	$H'(x)=0$ has two solutions $x=1$ and $1.5$.
	When $x\in [1,1.5]$, $H'(x)\geq 0$, indicating $H(x)$ is increasing and $H(x)\geq 0$ in this region. 
	For $x\in (1.5,1.7]$, $H'(x) < 0$, indicating $H(x)$ is decreasing. 
	Together with $H(1.7)\approx 0.006$, this implies $H(x)> 0$ for $x\in (1.5,1.7]$. 
	Therefore, we obtain $f(x)\geq \frac{(x-1)^2}{3}$ for $x\in [0.3,1.7]$.	
\end{proof}

\subsection{Proof of Lemma \ref{lem:pairaverage}}\label{sec:appendix-proof-lem:pairaverage}
\begin{proof}
	\textbf{Step 1.}
	
	Let $y_i^{(k)}=(x_i^{(k)})^2$ and $\bm{y}^{(k)}=((x_1^{(k)})^2,\dots,(x_n^{(k)})^2)$. Then $y_i^{(k)}\geq 0$ and $\sum_{i=1}^n y_i^{(k)}=C$.
	The update step is
	$$
	y_{i_k}^{(k+1)}=y_{j_k}^{(k+1)}=\frac{y_{i_k}^{(k)}+y_{j_k}^{(k)}}{2}.
	$$
	At each step, $\sum_{i=1}^n y_i^{(k)}=C$ is preserved.
	
	Define $\mu=\frac{C}{n}$ and  $$\Phi(\bm{y}^{(k)})=\sum_{i=1}^{n}\left(y_i^{(k)}-\mu\right)^2.$$
	Then $\Phi \geq 0$, $\Phi=0$ if and only if $y_i^{(k)}=\mu$ for each $i$.

	Suppose $y_{i_k}^{(k)}$ and $y_{j_k}^{(k)}$ are chosen in the $k$-th step.  
	The change of $\Phi$ is 
	\begin{align}
		\Delta_k & = \Phi(\bm{y}^{(k)})-\Phi(\bm{y}^{(k+1)})\\
		& = (y_{i_k}^{(k)}-\mu)^2+(y_{j_k}^{(k)}-\mu)^2-2\left(\frac{y_{i_k}^{(k)}+y_{j_k}^{(k)}}{2}-\mu\right)^2\\
		&=\frac{1}{2}(y_{i_k}^{(k)}-y_{j_k}^{(k)})^2 \label{eq:deltaphi}\\ 
		& \geq 0.
	\end{align}
	Thus, in each step $\Phi(\bm{y}^{(k)})\geq \Phi(\bm{y}^{(k+1)})$.
	When $y_{i_k}^{(k)}\neq y_{j_k}^{(k)}$, the inequality is strict. Since $\Phi(\bm{y}^{(k)})\geq 0$ and $\Phi(\bm{y}^{(k)})$ is non-increasing,  $\lim\limits_{k\rightarrow \infty}\Phi(\bm{y}^{(k)})=\Phi^*\geq 0$.

	\textbf{Step 2.}
	
	Summing up $\Delta_1$ to $\Delta_k$ yields
	\begin{align}
		\sum_{i=1}^{k} \Delta_i & = \sum_{l=1}^{k}\frac{1}{2}(y_{i_l}^{(l)}-y_{j_l}^{(l)})^2\\
		& =\sum_{i=1}^{k} (\Phi(\bm{y}^{(i)})-\Phi(\bm{y}^{(i+1)}))\\
		&=\Phi(\bm{y}^{(1)})-\Phi(\bm{y}^{(k+1)}).
	\end{align}
	We have
	\begin{align}
		\lim\limits_{k\rightarrow \infty}\sum_{i=1}^{k} \Delta_i&=\Phi(\bm{y}^{(1)})-\lim\limits_{k\rightarrow \infty} \Phi(\bm{y}^{(k+1)})\\
		&=\Phi(\bm{y}^{(1)})-\Phi^*\\
		& < \infty.
	\end{align}
	
	Thus, $\lim\limits_{k\rightarrow \infty}\sum_{l=1}^{k}\frac{1}{2}(y_{i_l}^{(l)}-y_{j_l}^{(l)})^2<\infty$. This implies $\lim\limits_{k\rightarrow \infty}\frac{1}{2}(y_{i_k}^{(k)}-y_{j_k}^{(k)})^2=0$. 
	Since $(i_k, j_k)\in \arg\max\limits_{i,j} |y_i^{(k)}-y_j^{(k)}|$, 
	we obtain $\lim\limits_{k\rightarrow \infty}\frac{1}{2}(y_{i}^{(k)}-y_{j}^{(k)})^2=0$ for any $i\neq j$. 
	Combining the condition $\sum_{i=1}^n y_i^{(k)}=C$ yields $\lim\limits_{k\rightarrow \infty}y_{i}^{(k)}=\frac{C}{n}$ for each $i$.
	
	Finally, we can obtain $\lim\limits_{k\rightarrow \infty}x_i^{(k)}=\sqrt{\frac{C}{n}}$ and $\lim\limits_{k\rightarrow \infty} \bm{x}^{(k)}=\left(\sqrt{\frac{C}{n}},\dots,\sqrt{\frac{C}{n}}\right)$.
\end{proof}

\subsection{Proof of Lemma \ref{lemma:bound-sum-log-x}}\label{sec:proof-lemma:bound-sum-log-x}
{\color{black}
	\textbf{Proof Sketch}.
	The key idea is to characterize the largest possible value of $\left|\sum_{i=1}^n\log x_i\right|$ under the constraint $\sum_{i=1}^n(x_i-1)^2\leq \e$. 
	
	After centering the variables by $y_i=x_i-1$, the problem becomes an extremal problem for $ S(\bm{y})=\sum_{i=1}^n\log (y_i+1)$ over a Euclidean ball.
	
	For the maximum of $S(\bm{y})$, the concavity and symmetry of the logarithm imply that the optimal configuration distributes the perturbation energy equally among all coordinates, yielding
	$$\mathrm{max}\ S(\bm{y})=n\log(1+\sqrt{\e/n}).$$
	
	For the minimum, we first show in the two-dimensional case that, for fixed
	$y_i^2 + y_j^2=c$, the quantity $\log(1+y_i)+\log (1+y_j)$ is minimized when $y_i=y_j<0$. Repeatedly applying this pairwise averaging argument as discussed in Lemma \ref{lem:pairaverage} decreases $S(\bm{y})$ while preserving $\sum_{i=1}^n y_i^2$. The process converges to the uniform negative configuration 
	$
	y_1=\cdots y_n=-\sqrt{\dfrac{\epsilon}{n}}
	$,
	which gives 
	$$\mathrm{min}\ S(\bm{y}) = n\log (1-\sqrt{\dfrac{\epsilon}{n}}).$$
	
	Finally, comparing $\mathrm{max}\ S(\bm{y})$ and $\mathrm{min}\ S(\bm{y})$ yields
	$$\left|\sum_{i=1}^n\log x_i\right| \leq -n\log\left(1-\sqrt{\frac{\e}{n}}\right).$$
}
\begin{proof}
	We center all variables by setting $y_i=x_i-1$. Then the problem becomes
	\begin{align}
		\mathrm{max} &\ \left| S(\bm{y})\right |=\left|\sum_{i=1}^ns(y_i)\right|=\left|\sum_{i=1}^n\log (y_i+1)\right|\\
		\mathrm{s.t.} &\ \sum_{i=1}^n y_i^2 \leq \e
	\end{align}
	where $\bm{y}=(y_1,\dots,y_n)$ and $s(y)=\log (y+1)$.
	
	We remove the absolute value in the objective function by solving $\mathrm{max}\ S(\bm{y})$ in \textbf{Step 1} and $\mathrm{min}\  S(\bm{y})$ in \textbf{Step 2}, and then take the larger absolute value in \textbf{Step 3}.
	
	\textbf{Step 1. Solving $\mathrm{max}\ S(\bm{y})$.} 
	
	The feasible set is the closed Euclidean ball of radius $\sqrt{\e}$ centered at the origin.
	$S(\bm{y})$ is continuous and hence has a maximum on the compact feasible set.
	The function $s(y)=\log(y+1)$ is strictly concave and increasing. Thus, $S(\bm{y})=\sum_{i=1}^n s(y_i)$ is also strictly concave and increasing.
	Therefore, the maximum is attained at some point $\bm{y}^*=(y^*_1,\dots,y^*_n)$ on the sphere $\Sigma_{i=1}^n y_i^2 = \e$ and satisfying $y^*_i>0$. 
	
	Suppose $S(\bm{y})$ attains its maximum on some point $\bm{y}=(y_1,\dots,y_i,\dots,y_j,\dots,y_n)$.
	Since $S(\bm{y})$ is invariant under permutations of the coordinates $y_i$, it follows that $S(\bm{y}')=S(\bm{y})$ for $\bm{y}'=(y_1,\dots,y_j,\dots,y_i,\dots,y_n)$ and is also the maximum.
	
	Furthermore, $S(\bm{y})$ is concave, so we have $$S(\bm{y})=\frac{S(\bm{y})+S(\bm{y}')}{2}\leq S\left(\frac{\bm{y}+\bm{y}'}{2}\right).$$
	Since the feasible set $\{(y_1,\dots,y_n)|\sum_{i=1}^n y_i^2 \leq \e\}$ is convex, 
	$\frac{\bm{y}+\bm{y}'}{2}=(y_1,\dots,\frac{y_i+y_j}{2},\dots,\frac{y_i+y_j}{2},\dots,y_n)$ lies in the feasible set. 
	This means we can find a point $\frac{\bm{y}+\bm{y}'}{2}$ in the feasible set where $S(\frac{\bm{y}+\bm{y}'}{2})$ 
	takes a value not less than the maximum $S(\bm{y})$. 
	Therefore, the maximum must be attained at the point $\bm{y}^*=(y^*_1,\dots,y^*_n)$ where $y_i=y_j$ for all $1\leq i,j\leq n$.
	
	Combining $\Sigma_{i=1}^n y_i^2 = \e$ yields $y_i=\sqrt{\e/n}$ for each $1\leq i\leq n$. So the maximum of $S(\bm{y})$ is $n\log(1+\sqrt{\e/n})$ when $\e$ is distributed equally among each $y_i$.
	
	\textbf{Step 2. Solving $\mathrm{min}\ S(\bm{y})$.}
	Since $s(y)=\log(y+1)$ is strictly increasing, replacing any positive coordinate $y_i$ by $-y_i$ preserves feasibility and decreases $S(\bm{y})$. Hence, it suffices to minimize $S$ over the nonpositive orthant, where $y_i\leq 0$ for all $i$.
	We first resolve the two-dimensional case in \textbf{Step 2.1} and then generalize to the high-dimensional problem in \textbf{Step 2.2}. 
	
	\textbf{Step 2.1. Resolving the two-dimensional problem.}
	
	Take any two $y_i$ and $y_j$ and note $S_{ij}(y_i,y_j)=\log(1+y_i)+\log(1+y_j)=\log(1+y_i)(1+y_j)$. 
	We want to solve the following problem.
	\begin{align}
		min\ &\ S_{ij}(y_i,y_j)=\log(1+y_i)+\log (1+y_j)\\
		s.t.\ &\ y_i^2 + y_j^2=c\\
		&\ 0 < c < 1/2\\
		&\ y_i,y_j \leq 0
	\end{align}
	
	Since $\log(x)$ is monotonically increasing, it suffices to minimize  $S_{ij}'(y_i,y_j)=(1+y_i)(1+y_j)$,
	and the minimum must be attained on the boundary of the circle for $y_i,y_j \leq 0$.
	Using the standard parameterization of the circle, the problem is transformed to 
	\begin{align}
		min\ &\ F(\theta)=(1+\sqrt{c}\cos \theta)(1+\sqrt{c}\sin \theta)\\
		s.t.\ &\ \pi \leq \theta \leq \dfrac{3\pi}{2}\\
		&\ 0< c < 1/2
	\end{align}
	
	We now show that $F(\theta)$ has only one stationary point in $\theta\in [\pi, \frac{3\pi}{2}]$.
	The first-order derivative of $F(\theta)$ is 
	\begin{align}
		&F'(\theta)\\
		&=\sqrt{c}(\cos \theta-\sin \theta)+c\cos 2\theta\\
		&=\sqrt{c}(\cos \theta-\sin \theta)+c(\cos \theta-\sin \theta)(\cos \theta+\sin \theta)\\
		&=(\cos \theta-\sin \theta)(\sqrt{c}+c(\cos \theta+\sin \theta)).
	\end{align}
	
	To make $F'(\theta)=0$, either $\cos \theta-\sin \theta$ or $\sqrt{c}+c(\cos \theta+\sin \theta)$ must be 0.
	
	In the former case, $\theta=\frac{5\pi}{4}$ makes $\cos \theta=\sin \theta=-\frac{\sqrt{2}}{2}$ and $\cos \theta-\sin \theta=0$.
	
	In the latter case,  we show $\sqrt{c}+c(\cos \theta+\sin \theta)=0$ has no solution. In other words,
	\begin{align}
		\cos \theta+\sin \theta=\dfrac{-1}{\sqrt{c}} \label{eq:TC}
	\end{align}
	has no solution.
	
	Let $T(\theta)=\cos \theta+\sin \theta\ (\theta \in [\pi, \frac{3\pi}{2}])$. The second-order derivative $T''(\theta)=-\cos \theta-\sin \theta> 0$ on $[\pi, \frac{3\pi}{2}]$, so $T(\theta)$ is strictly convex and has only one minimum on  $[\pi, \frac{3\pi}{2}]$. Solving $T'(\theta)=\cos \theta-\sin \theta=0$ yields $\theta=5\pi/4$. So the minimum of $T(\theta)$ is $T(\frac{5\pi}{4})=-\sqrt{2}$. However, from the condition $c<1/2$, we obtain $\frac{-1}{\sqrt{c}}<-\sqrt{2}$. This implies $\cos \theta+\sin \theta> \frac{-1}{\sqrt{c}}$ for $\theta \in [\pi, \frac{3\pi}{2}]$.
	So $\sqrt{c}+c(\cos \theta+\sin \theta)=0$ has no solution on $[\pi, \frac{3\pi}{2}]$. 
	
	Therefore, $F(\theta)$ has only one stationary point at $\theta=\frac{5\pi}{4}$. Since $F(\theta)$ is continuously differentiable, the minimum is attained at the endpoints of the feasible interval $[\pi, \frac{3\pi}{2}]$ or the single stationary point.
	
	At the endpoints, $F(\pi)=F(\frac{3\pi}{2})=1-\sqrt{c}$. At the stationary point, $F(\frac{5\pi}{4})=(1-\sqrt{c/2})^2$. 
	It remains to compare $1-\sqrt{c}$ and $(1-\sqrt{c/2})^2$.
	
	Now we show $H(c)=(1-\sqrt{c/2})^2-(1-\sqrt{c})<0$ for $0<c<1/2$. It follows that $H''(c)=\dfrac{1}{4}(\sqrt{2}-1)c^{-3/2}>0$ for $0<c<1/2$. So $H(c)$ is convex. Solving $H'(c)=0$ yields $c=(\sqrt{2}-1)^2$. The minimum of $H(c)$ is $H((\sqrt{2}-1)^2)=-\frac{1}{2}(\sqrt{2}-1)^2<0$. It is also easy to verify that $H(0)=0$ and $H(1/2)=\frac{2\sqrt{2}-3}{4}<0$. 
	Therefore, $H(c)<0$ for $0<c<1/2$. 
	
	Therefore, the minimum of $F(\theta)$ is $(1-\sqrt{c/2})^2$ at $\theta=\frac{5\pi}{4}$.
	
	Accordingly, $S_{ij}$ attains its minimum $2\log(1-\sqrt{c/2})$ when $y_i=y_j$.
	
	\textbf{Step 2.2. Extending to the higher-dimensional problem.}
	
	Now we extend the above step to high-dimensional problem. 
	
	Since $\log(x)$ is monotonically increasing, the minimum is attained on the boundary satisfying $y_i<0$ for each $i$ and $\sum_{i=1}^n y_i^2=\e$.
	
	Given any point $\bm{y}^{(1)}=(y_1^{(1)},\dots,y_n^{(1)})$ on the boundary 
	and $S(\bm{y}^{(1)})=\sum_{k=1}^n \log(1+y_k^{(1)})$, we choose $i_1,j_1=\arg\max\limits_{i,j} |(y_i^{(1)})^2-(y_j^{(1)})^2|$, then average $y_{i_1}$ and $y_{j_1}$ by 
	$$
	y_{i_1}^{(2)}=y_{j_1}^{(2)}=-\sqrt{\frac{\left(y_{i_1}^{(1)}\right)^2+\left(y_{j_1}^{(1)}\right)^2}{2}}
	$$
	without affecting the sum $\left(y_{i_1}^{(1)}\right)^2+\left(y_{j_1}^{(1)}\right)^2=\left(y_{i_1}^{(2)}\right)^2+\left(y_{j_1}^{(2)}\right)^2=c\leq \e <1/2$ . Keeping the other coordinates unchanged, we obtain $\bm{y}^{(2)}=(y_1^{(2)},\dots,y_n^{(2)})=(y_1^{(1)},\dots,y_{i_1}^{(2)},\dots,y_{j_1}^{(2)},\dots, y_n^{(1)})$. 
	
	From \textbf{Step 2.1}, we obtain
	\begin{align}\nonumber
		&S(\bm{y}^{(1)})\nonumber\\
		=&\log(1+y_{i_1}^{(1)})+\log(1+y_{j_1}^{(1)})+\sum_{k\neq i_1,j_1}\log(1+y_k^{(1)})\nonumber\\
		\geq& \log(1+y_{i_1}^{(2)})+\log(1+y_{j_1}^{(2)}) +\sum_{k\neq i_1,j_1}\log(1+y_k^{(1)})\nonumber\\
		=& S(\bm{y}^{(2)}).\nonumber
	\end{align}
	
	Applying this step iteratively, we can obtain a sequence $\bm{y}^{(1)},\bm{y}^{(2)},\dots, \bm{y}^{(k)}$ satisfying $S(\bm{y}^{(i)})\geq S(\bm{y}^{(i+1)})\ (1\leq i \leq k-1)$. 	
	Substituting $\bm{x}^{(k)}=-\bm{y}^{(k)}$ into Lemma \ref{lem:pairaverage} in the main text, we can obtain $\lim\limits_{k\rightarrow \infty}\{\bm{y^{(k)}}\}=(y_1^*,\dots,y_n^*)$ where $y_i^*=-\sqrt{\e/n}$ for each $i$. 
	This indicates $\lim\limits_{k\rightarrow \infty}S(\bm{y}^{(k)})=n\log (1-\sqrt{\e/n})$.
	
	In summary, from any given point $\bm{y}^{(1)}=(y_1^{(1)},\dots,y_n^{(1)})$ on the boundary, we can find a sequence $\{\bm{y}^{(k)}\}$ such that $\lim\limits_{k\rightarrow \infty}S(\bm{y}^{(k)})=n\log (1-\sqrt{\e/n}) \leq S(\bm{y}^{(1)})$. 
	
	Therefore, the minimum of $S(\bm{y})$ is $n\log (1-\sqrt{\e/n})$.

	\textbf{Step 3. Comparing absolute values of min and max.}
	
	Now we have obtained $\max S(\bm{y})=n\log(1+\sqrt{\e/n})$ and $\min S(\bm{y})=n\log (1-\sqrt{\e/n})$. It remains to compare their absolute values.
	
	It suffices to show $|n\log(1+\sqrt{\e/n})|<|n\log(1-\sqrt{\e/n})|$.
	Since $-\log(1-x)-\log(1+x)=-\log(1-x^2)>0$ for $0<x<1$, it follows that $|n\log(1-\sqrt{\e/n})|>|n\log(1+\sqrt{\e/n})|$. 
	
	Therefore, the absolute value of the minimum is larger than the maximum. This yields $|S(\bm{y})|\leq -n\log(1-\sqrt{\e/n})$ and completes the proof.
\end{proof}

\subsection{Proof of Lemma \ref{lem:bound-deltamu-deltasigma}}\label{sec:proof-bound-deltamu-deltasigma}
\begin{proof}
	$\mathrm{KL}(\N_1\Vert \N_2)$ has the following closed form.
	\begin{equation}
		\begin{aligned}\label{eq:kl_n1_n2}
			& \frac{1}{2}\left[(\mu_2-\mu_1)^\top \Sigma_2^{-1}(\mu_2-\mu_1)+\Tr(\Sigma_2^{-1}\Sigma_1)\right.\\
			& \left. -\log \det(\Sigma_2^{-1}\Sigma_1)-d\right]
		\end{aligned}
	\end{equation}
	where $\Tr$ and $\det$ denote the trace and determinant of a matrix, respectively.
	
	Let $\Delta \mu=\mu_2-\mu_1$, $A=\Sigma_2^{-1/2}\Sigma_1\Sigma_2^{-1/2}$. 
	The matrix $A$ is positive definite, and 
	\begin{equation}
		\begin{aligned}\label{eq:aboutA}
			\Tr(A)&=\Tr(\Sigma_2^{-1/2}\Sigma_1\Sigma_2^{-1/2})=\Tr(\Sigma_2^{-1}\Sigma_1)\\
			\det(A)&=\det(\Sigma_2^{-1/2}\Sigma_1\Sigma_2^{-1/2})=\det(\Sigma_2^{-1}\Sigma_1).
		\end{aligned}
	\end{equation}

	The condition $\mathrm{KL}(\N_1\Vert \N_2)\leq \epsilon$ can be expressed as
	\begin{equation}\label{eq:KLN1N2-small}
		\begin{aligned}
			\frac{1}{2}\left[\Delta\mu^{\top}\Sigma_2^{-1}\Delta\mu+\Tr(A)-\log\det(A)-d\right]
			&\leq \epsilon.
		\end{aligned}
	\end{equation}
	Let $\lambda_i\ (1\leq i \leq d)$ be the eigenvalues of $A$. Since $A$ is positive definite, $\lambda_i>0$. Therefore,
	\begin{equation}
		\begin{aligned} \nonumber
			&\Tr(A)-\log\det(A)-d=\sum_{i=1}^d \lambda_i-\log\prod_{i=1}^d\lambda_i-d\\
			&=\sum_{i=1}^d(\lambda_i-\log\lambda_i-1)=\sum_{i=1}^d f(\lambda_i)\geq 0.
		\end{aligned}
	\end{equation}
	Since $\Sigma_2^{-1}$ is positive definite, the quadratic term satisfies $\Delta\mu^{\top}\Sigma_2^{-1}\Delta\mu\geq 0$. The condition in Inequality \eqref{eq:KLN1N2-small} implies the following two inequalities
	\begin{align}
		\Delta\mu^{\top}\Sigma_2^{-1}\Delta\mu\leq  2\epsilon \label{eq:quadratic-bound} \\
		\sum_{i=1}^d(\lambda_i-\log\lambda_i-1)=\sum_{i=1}^d f(\lambda_i)\leq 2\epsilon.  \label{eq:sumf-bound}
	\end{align}
	
	In \textbf{Steps 1} and \textbf{2} below, we derive bounds on $\norm{\mu_2-\mu_1}$ and $\|\Sigma_1-\Sigma_2\|_{F}$, respectively. The bound on $\|\Sigma_1^{-1}-\Sigma_2^{-1}\|_{F}$ then follows from the bound on $\|\Sigma_1-\Sigma_2\|_{F}$.
	
	\textbf{Step 1. Bounding  $\norm{\mu_2-\mu_1}$.}
	
	Let $\lambda'_{min}$ and $\lambda'_{max}$ denote the minimum and maximum eigenvalues of $\Sigma_2^{-1}$, respectively, and let $\lambda_{max}$ denote the maximum eigenvalue of $\Sigma_2$. Then the following inequality holds:
	$$
	\lambda'_{min}\norm{\Delta \mu}^2 \leq \Delta\mu^{\top}\Sigma_2^{-1}\Delta\mu \leq \lambda'_{max}\norm{\Delta \mu}^2.
	$$
	Combining this eigenvalue bound with \eqref{eq:quadratic-bound} yields
	\begin{align}
		\norm{\Delta\mu}^2\leq \frac{2\epsilon}{\lambda'_{min}}.
	\end{align}
	
	Since $\Sigma_2$ is positive definite, the eigenvalues of $\Sigma_2^{-1}$ are the reciprocals of those of $\Sigma_2$, hence $\frac{1}{\lambda'_{min}}=\lambda_{max}=\norm{\Sigma_2}_{\mathrm{op}}=\|\Sigma_2^{1/2}\|^2_{\mathrm{op}}$. Therefore,
	\begin{align}\label{eq:bound-delta-mu}
		\norm{\Delta\mu} \leq \|\Sigma_2^{1/2}\|_{\mathrm{op}}\sqrt{2\epsilon}.
	\end{align}

	\textbf{Step 2. Bounding  $\|\Sigma_1-\Sigma_2\|_{F}$.}
	
	Since both $\Sigma_1$ and $\Sigma_2$ are positive definite, the following equality holds:
	\begin{align}
		\|\Sigma_1-\Sigma_2\|_{F} &=\|\s_2^{1/2}(\s_2^{-1/2}\s_1\s_2^{-1/2}-\mathbf{I}_d)\s_2^{1/2}\|_F\nonumber\\
		&=\|\s_2^{1/2}(A-\mathbf{I}_d)\s_2^{1/2}\|_F\nonumber\\
		&\leq \|\s_2^{1/2}\|_{\mathrm{op}}\|(A-\mathbf{I}_d)\|_F\|\s_2^{1/2}\|_{\mathrm{op}} \label{eq:ABC-F-op}\\
		&=\|\s_2\|_{\mathrm{op}}\|(A-\mathbf{I}_d)\|_F \label{eq:s1-s2-bound-A-I}
	\end{align}
	where Inequality \eqref{eq:ABC-F-op} follows from $\|ABC\|_F\leq \|A\|_{\mathrm{op}}\| B\|_F \|C\|_{\mathrm{op}}$ for any three matrices \cite{bhatia2013matrix}.
	
	It remains to bound $\|A-\mathbf{I}_d\|_F$ in terms of $\e$.
	In the following, we show $\|A-\mathbf{I}_d\|^2_F=\sum_{i=1}^d(\lambda_i-1)^2$.
	
	By the definition $A=\s_2^{-1/2}\s_1\s_2^{-1/2}$, $A$ is positive definite, and therefore can be orthogonally decomposed as 
	$A=U\Lambda U^\top$, where $U$ is an orthogonal matrix and $\Lambda=diag(\lambda_1,\dots,\lambda_d)$. Thus,
	\begin{align}
		A-\mathbf{I}_d & = U\Lambda U^\top-U\mathbf{I}_d U^\top=U(\Lambda-\mathbf{I}_d)U^{\top}.
	\end{align}
	Therefore,
	\begin{align}
		&\|A-\mathbf{I}_d\|^2_F \nonumber\\
		& = \|U(\Lambda-\mathbf{I}_d)U^{\top}\|^2_F\nonumber\\
		& = \Tr((U(\Lambda-\mathbf{I}_d)U^{\top})^{\top}U(\Lambda-\mathbf{I}_d)U^{\top})\tag{by the definition of $\|\cdot\|_F$ and trace}\\
		& = \Tr(U(\Lambda-\mathbf{I}_d)^{\top}U^{\top}U(\Lambda-\mathbf{I}_d)U^{\top})\nonumber\\
		& = \Tr(U(\Lambda-\mathbf{I}_d)^{\top}(\Lambda-\mathbf{I}_d)U^{\top})\nonumber\\
		& = \Tr((\Lambda-\mathbf{I}_d)^{\top}(\Lambda-\mathbf{I}_d))\nonumber\\
		%& = \Tr(diag(\lambda_1-1,\dots,\lambda_d-1)^{\top}diag(\lambda_1-1,\dots,\lambda_d-1))\nonumber\\
		& = \sum_{i=1}^d(\lambda_i-1)^2 \label{eq:A-I-sum}
	\end{align}
	where the last equation follows from $\Lambda-\mathbf{I}_d=diag(\lambda_1-1,\dots,\lambda_d-1)$.
	
	{\color{black}
		By Lemma \ref{lemma:f}, $f(x)$ is convex and attains its minimum $f(1)=0$. 
		Hence, Inequality \eqref{eq:sumf-bound}, namely $\sum_{i=1}^d f(\lambda_i)\leq 2\epsilon$, 
		implies that $f(\lambda_i)\leq 2\epsilon$ for every $i$. 
		Moreover, $f(0.3)\approx 0.504$ and $f(1.7)\approx 0.169$. Because $f$ decreases on $(0,1)$ and increases on $(1,+\infty)$, 
		the condition $2\epsilon<2\times \frac{1}{12}\approx 0.167<f(1.7)<f(0.3)$ implies $\lambda_i\in(0.3,1.7)$.
	}
	
	Therefore, we can apply Lemma \ref{lemma:f-lower-bound} to Inequality \eqref{eq:sumf-bound} and obtain 
	\begin{align}
		&\sum_{i=1}^d \frac{(\lambda_i-1)^2}{3} \leq \sum_{i=1}^d f(\lambda_i)\leq 2\epsilon \nonumber\\
		\Longrightarrow & \sum_{i=1}^d (\lambda_i-1)^2  \leq 6\epsilon. \label{eq:sum-quadratic-lambda-bound}
	\end{align}
	
	Combining \eqref{eq:A-I-sum} and  \eqref{eq:sum-quadratic-lambda-bound} yields
	\begin{align}
		\|A-\mathbf{I}_d\|_F \leq \sqrt{6\e}. \label{eq:A-I-6e}
	\end{align}
	Substituting \eqref{eq:A-I-6e} into \eqref{eq:s1-s2-bound-A-I}, we obtain
	\begin{align}
		\|\Sigma_1-\Sigma_2\|_{F} & \leq \|\s_2\|_{\mathrm{op}}\|(A-\mathbf{I}_d)\|_F
		\leq \|\s_2\|_{\mathrm{op}}\sqrt{6\e}. \label{eq:s1-s2-bound}
	\end{align}
	
	Finally, the bound on $\|\Sigma_1^{-1}-\s_2^{-1}\|_{F}$ follows from
	\begin{align}
		\|\Sigma_1^{-1}-\s_2^{-1}\|_{F} & =\|\s_2^{-1}(\s_2-\s_1)\s_1^{-1}\|_{F}\nonumber\\
		%& \leq \|\s_2^{-1}(\s_2-\s_1)\s_1^{-1}\|_{F} \tag{by $\| \cdot \|_{\mathrm{op}} \leq \|\cdot\|_F$} \\
		& \leq \|\s_2^{-1}\|_{\mathrm{op}}\|\s_1^{-1}\|_{\mathrm{op}}\|\s_2-\s_1\|_F \label{eq:ABC_F_bound}\\
		& \leq \|\s_2^{-1}\|_{\mathrm{op}}\|\s_1^{-1}\|_{\mathrm{op}}\|\s_2\|_{\mathrm{op}}\sqrt{6\e},
	\end{align}
	where Inequality \eqref{eq:ABC_F_bound} follows from $\|ABC\|_F\leq \|A\|_{\mathrm{op}}\|B\|_F\|C\|_{\mathrm{op}}$ for any three matrices \cite{bhatia2013matrix}.
\end{proof}

\subsection{Explicit Expressions of the Bound}\label{sec:explicit-bounds}
\begin{remark}(Coefficient in the asymptotic bound)
	\label{remark-K}
	The coefficient $K=F_1+F_2+F_3$ consists of the following three terms:
	\begin{equation}
		\begin{aligned}\label{eq:thmF123}
			F_1 = &\sqrt{2}\|\Sigma_2^{1/2}\|_{\mathrm{op}}\|\s_2^{-1}\|_{\mathrm{op}}\|\mu_2\|  \\
			&+\frac{1}{2}\sqrt{6}\|\s_1^{-1}\|_{\mathrm{op}}\|\s_2^{-1}\|_{\mathrm{op}}\|\Sigma_2\|_{\mathrm{op}} \|\mu_1\|^2	\\
			F_2 = & (\sqrt{6}\|\s_2^{-1}\|_{\mathrm{op}}\|\s_1^{-1}\|_{\mathrm{op}}\|\Sigma_2\|_{\mathrm{op}}\|\mu_2\|  \\
			& + \sqrt{2} \|\s_2^{-1}\|_{\mathrm{op}} \|\Sigma_2^{1/2}\|_{\mathrm{op}})\mathbb{E}_P\left[\|X\|\right] \\
			F_3 = & \frac{\sqrt{6}}{2}\|\s_1^{-1}\|_{\mathrm{op}}\|\Sigma_2\|_{\mathrm{op}}\|\s_2^{-1}\|_{\mathrm{op}}\mathbb{E}_P\left[  \|X\|^2  \right]. 
		\end{aligned}
	\end{equation}
\end{remark}

\subsection{{\color{black}Proof of Proposition \ref{thm:tightness}}}\label{sec:proof_e_tight}
\begin{proof}
	It suffices to construct an example whose KL deviation is of order $\sqrt{\epsilon}$. Consider the two-dimensional Gaussian distributions
	$\N_1=\N(\mathbf{0},\mathbf{I}_2)$ and
	$\N_2=\N(\delta e_1,\mathbf{I}_2)$, where $\delta\to0$ and $e_1=(1,0)^\top$. Then
	\[
	\mathrm{KL}(\N_1\Vert\N_2)=\frac{\delta^2}{2}=\epsilon,
	\qquad \delta=\pm\sqrt{2\epsilon}.
	\]
	
	Let $P=\N(t e_1,\mathbf{I}_2)$ for a fixed constant $t>0$. Since $\mathrm{KL}(P\Vert\N_1)=t^2/2$, we have
	\[
	\mathrm{KL}(P\Vert\N_2)-\mathrm{KL}(P\Vert\N_1)
	=-t\delta+\frac{\delta^2}{2}
	=\pm t\sqrt{2\epsilon}+\epsilon.
	\]
	Consequently,
	$\left|\mathrm{KL}(P\Vert\N_2)-\mathrm{KL}(P\Vert\N_1)\right|=\Theta(\sqrt{\epsilon})$.
	Thus, the $\sqrt{\epsilon}$ rate is worst-case optimal. This example characterizes the optimal worst-case sensitivity of KL divergence under Gaussian perturbations.
\end{proof}

{\color{black}
	\subsection{Dimension Dependence of the Bound}\label{sec:example_dimension_dependence}
	\textbf{Example}.
	Let $p\ge 1$ be fixed, let $d\ge2$ and $0<t<1/5$, and consider
	\[
	\begin{aligned}
		P_d
		&=\mathcal \N\!\left(
		0,\operatorname{diag}(d^p,1,\ldots,1)
		\right),\\
		N_1
		&=\mathcal \N\!\left(
		0,\operatorname{diag}(1-t,1+t,1,\ldots,1)
		\right),\\
		N_2
		&=\mathcal \N\!\left(
		0,\operatorname{diag}(1+t,1-t,1,\ldots,1)
		\right).
	\end{aligned}
	\]
	Note that $P_d$ is a distribution whose parameters depend on $d$.
	Denote the covariance matrices of $N_1$ and $N_2$ by
	$\Sigma_1$ and $\Sigma_2$, respectively. The second moment
	of $P_d$ is
	\[
	\begin{aligned}
		\mathbb E_{P_d}\|X\|^2
		=\operatorname{tr}\!\left(
		\operatorname{diag}(d^p,1,\ldots,1)
		\right)
		=d^p+d-1
		=\Theta(d^p).
	\end{aligned}
	\]
	The eigenvalues of both reference covariance matrices lie
	in $(4/5,6/5)$, independently of $d$. In particular,
	\[
	\|\Sigma_1^{-1}\|_{\mathrm{op}}=\frac{1}{1-t},\ 
	\|\Sigma_2\|_{\mathrm{op}}=1+t,\ 
	\|\Sigma_2^{-1}\|_{\mathrm{op}}=\frac{1}{1-t}.
	\]
	Consequently,
	\[
	F_3
	=\frac{\sqrt6}{2}
	\frac{1+t}{(1-t)^2}(d^p+d-1)
	=\Theta(d^p),
	\]
	uniformly over $0<t<1/5$.
	
	In the following, we show the KL difference $\mathrm{\mathrm{KL}}(P_d\|\N_1)-\mathrm{KL}(P_d\|\N_2)$ scales as $\Theta(d^p\sqrt{\epsilon})$.
	Observe that
	$
	\det\Sigma_1=\det\Sigma_2=1-t^2
	$. 
	The logarithmic determinant term in the Gaussian KL
	formula therefore vanishes, giving
	\begin{equation}
		\label{eq:example_KL_N1N2}
		\begin{aligned}
			\epsilon
			:=\mathrm{KL}(\N_1\|\N_2)
			=\frac12\left(
			\frac{1-t}{1+t}
			+\frac{1+t}{1-t}-2
			\right)
			&=\frac{2t^2}{1-t^2}.
		\end{aligned}
	\end{equation}
	The restriction $0<t<1/5$ ensures that
	$0<\epsilon<1/12$.
	
	Let $C_d=\operatorname{diag}(d^p,1,\ldots,1)$.
	Since all three distributions have zero means and
	$\det\Sigma_1=\det\Sigma_2$, the KL difference satisfies
	\[
	\begin{aligned}
		\Delta_{\mathrm{KL}}
		=\mathrm{\mathrm{KL}}(P_d\|\N_1)-\mathrm{KL}(P_d\|\N_2)
		=\frac12\operatorname{tr}\!\left(
		(\Sigma_1^{-1}-\Sigma_2^{-1})C_d
		\right).
	\end{aligned}
	\]
	Moreover,
	\[
	\Sigma_1^{-1}-\Sigma_2^{-1}
	=\operatorname{diag}\!\left(
	\frac{2t}{1-t^2},
	-\frac{2t}{1-t^2},
	0,\ldots,0
	\right).
	\]
	Thus,
	\begin{equation}
		\label{eq:delta_example}
		\begin{aligned}
			\Delta_{\mathrm{KL}}
			=\frac12\left(
			\frac{2t}{1-t^2}d^p-\frac{2t}{1-t^2}
			\right)
			=\frac{t}{1-t^2}(d^p-1).
		\end{aligned}
	\end{equation}
	
	Solving Equation \eqref{eq:example_KL_N1N2} for $t^2$ yields
	$$
	t^2=\frac{\epsilon}{\epsilon+2},
	\qquad
	1-t^2=\frac{2}{\epsilon+2}.
	$$
	Substituting these identities into
	\eqref{eq:delta_example} gives
	\[
	\Delta_{\mathrm{KL}}
	=\frac{d^p-1}{2}\sqrt{\epsilon(\epsilon+2)}.
	\]
	Since $d\ge2$, $p\ge 1$, and $0<\epsilon<1/12$, we have
	$
	\frac{\sqrt2}{4}\,d^p\sqrt{\epsilon}
	\le\Delta_{\mathrm{KL}}
	\le\frac{5}{4\sqrt3}\,d^p\sqrt{\epsilon}
	$. 
	Hence,
	$
	\Delta_{\mathrm{KL}}=\Theta(d^p\sqrt{\epsilon})
	$
	with constants independent of $d$ and $\epsilon$.

	In particular, setting $p=1$ gives
	$\mathbb E_{P}[\|X\|^2]=2d-1=\Theta(d)$ and
	$\Delta_{\mathrm{KL}}=\Theta(d\sqrt{\epsilon})$.
	This example shows that the $O(d^p\sqrt{\epsilon})$
	order is attainable when the second moment grows as
	$\Theta(d^p)$. Thus, the dependence on $d$ cannot generally be
	improved.
}

\subsection{Proof of Corollary \ref{thm:triangle-P-N1-standardN2}}\label{proof-thm:triangle-P-N1-standardN2}
\begin{proof}
	From the condition $\N_2=\N(\mathbf{0},\mathbf{I}_d)$, we have $\mu_2=0$, $\|\Sigma_2\|_{\mathrm{op}}=\|\Sigma^{1/2}_2\|_{\mathrm{op}}=\|\Sigma^{-1}_2\|_{\mathrm{op}}=1$. Combining the definition of constants $F_1$, $F_2$, and $F_3$ (Equation~\eqref{eq:thmF123}) yields
	\begin{align}\label{eq:K1-K6}
		\begin{split}
			F_1 &=\frac{\sqrt{6}}{2}(\|\Sigma_1^{-1}\|_{\mathrm{op}}\|\mu_1\|^2) \\
			F_2 &= \sqrt{2} \mathbb{E}_P \left[\|X\|\right]\\
			F_3 &= \frac{\sqrt{6}}{2}\|\Sigma_1^{-1}\|_{\mathrm{op}}\mathbb{E}_{P}\left[\|X\|^2\right].
		\end{split}
	\end{align}
	
	Substituting $\mu_2=\bm{0}$ and $\Sigma_2=\mathbf{I}_d$ into Lemma \ref{lem:bound-deltamu-deltasigma} yields 
	\begin{align}
		\norm{\mu_1}\leq&  \sqrt{2\epsilon} \label{eq:bound-mu1}\\
		\|\Sigma_1-\mathbf{I}_d\|_{F}\leq& \sqrt{6\epsilon}.
	\end{align}
	Let $\lambda_i$ be the eigenvalues of $\Sigma_1$ and $1/\lambda_i$ be the eigenvalues of $\Sigma_1^{-1}$. Let $\lambda_{min}$ be the smallest eigenvalue of $\Sigma_1$. 
	It follows that $\|\Sigma_1^{-1}\|_{\mathrm{op}}=1/\lambda_{min}$. We have 
	\begin{align}\nonumber
		& \|\Sigma_1-\mathbf{I}_d\|_{F}\leq \sqrt{6\epsilon} \nonumber\\ 
		\Longrightarrow & \sum_{i=1}^{d}{(\lambda_i-1)}^2 \leq 6\e \ \tag{by \eqref{eq:A-I-sum}} \nonumber\\
		\Longrightarrow & (\lambda_{min}-1)^2 \leq 6\e   \nonumber\\
		\Longrightarrow & 	\lambda_{min} \geq 1-\sqrt{6\e}  \nonumber\\
		\Longrightarrow & \|\Sigma_1^{-1}\|_{\mathrm{op}}=1/\lambda_{min} \leq \frac{1}{1-\sqrt{6\e}}<2+\sqrt{2}, \label{eq:bound-S-1op}
	\end{align}
	where the last inequality follows from $\e<1/12$.
	
	Substituting Inequalities \eqref{eq:bound-mu1} and \eqref{eq:bound-S-1op} into Equations \eqref{eq:K1-K6}, we obtain
	\begin{align}\label{eq:bound-K2K5K6}
		\begin{split}
			F_1 &< 2(\sqrt{3}+\sqrt{6})\sqrt{\e}\\
			F_3 &< (\sqrt{3}+\sqrt{6}) \mathbb{E}_P\left[\|X\|^2\right].
		\end{split}
	\end{align}
	Combining Inequalities \eqref{eq:upperbound-EP-logterm}, \eqref{eq:K1-K6}, and \eqref{eq:bound-K2K5K6} yields the following bounds independent of the parameters of $\N_1$:
	\begin{align}
		\begin{split}
			&\left|\mathbb{E}_P\left[\log \frac{n_1(X)}{n_2(X)}\right]\right| \\
			&\leq \frac{\sqrt{6d\e}}{2} + (F_1+F_2+F_3)\sqrt{\e} + o(\sqrt{\e})\\
			&<\frac{\sqrt{6d\e}}{2} + S\sqrt{\e}+ o(\sqrt{\e}),
		\end{split}
	\end{align}
	where $S=\sqrt{2}\mathbb{E}_P \left[\|X\|\right]+(\sqrt{3}+\sqrt{6})\mathbb{E}_P \left[\|X\|^2\right]$ and 
	only depends on the first and second norm moments of $P$.
	
	Combining this bound with Equation \eqref{eq:rewrite-KLP-N2} in the main text, we can obtain the following bound
	\begin{align}
		\begin{split}
			|\mathrm{KL}(P\Vert \mathcal{N}_1)-\mathrm{KL}(P\Vert \mathcal{N}_2)|< \frac{\sqrt{6d\e}}{2} + S\sqrt{\e}+ o(\sqrt{\e}).\label{eq:corollary-std-bound}
		\end{split}
	\end{align}
	Note that Equation~\eqref{eq:bound-S-1op} is a strict inequality. Thus, the resulting overall bound is also strict.
\end{proof}
{\color{black}

	\subsection {Proof of Corollary~\ref{cor:multi-step-KL-stability}}\label{sec:proof-multiple-step-triangle-KL}
	\begin{proof}
		We prove the result by induction on $k$.
		
		For $k=2$, the result directly follows from the assumption $\mathrm{KL}(\mathcal{N}_1\|\mathcal{N}_2)\leq \epsilon$.
		Hence,
		\[
		\mathrm{KL}(\mathcal{N}_1\|\mathcal{N}_2)
		\leq
		(2-1)^2\epsilon+o(\epsilon).
		\]
		
		Assume that the result holds for $k=m$, \textit{i.e.},
		\[
		\mathrm{KL}(\mathcal{N}_1\|\mathcal{N}_m)
		\leq
		(m-1)^2\epsilon+o(\epsilon).
		\]
		We next prove that it also holds for $k=m+1$.
		
		By applying the two-step KL stability result to the three Gaussian distributions
		$\mathcal{N}_1$, $\mathcal{N}_m$, and $\mathcal{N}_{m+1}$, we have
		\begin{align}\nonumber
			&\mathrm{KL}(\mathcal{N}_1\|\mathcal{N}_{m+1})\\
			&\leq
			\left(
			\sqrt{\mathrm{KL}(\mathcal{N}_1\|\mathcal{N}_m)}
			+
			\sqrt{\mathrm{KL}(\mathcal{N}_m\|\mathcal{N}_{m+1})}
			\right)^2
			+
			o(\epsilon).\nonumber
		\end{align}

		By the induction hypothesis and the assumption,
		$
		\mathrm{KL}(\mathcal{N}_1\|\mathcal{N}_m)
		\leq
		(m-1)^2\epsilon+o(\epsilon)
		$,
		and
		$
		\mathrm{KL}(\mathcal{N}_m\|\mathcal{N}_{m+1})
		\leq
		\epsilon
		$.
		Therefore,
		$
		\sqrt{\mathrm{KL}(\mathcal{N}_1\|\mathcal{N}_m)}
		\leq
		(m-1)\sqrt{\epsilon}+o(\sqrt{\epsilon}),
		$
		and thus
		\[
		\begin{aligned}
			\mathrm{KL}(\mathcal{N}_1\|\mathcal{N}_{m+1})
			&\leq
			\left(
			(m-1)\sqrt{\epsilon}
			+
			\sqrt{\epsilon}
			+
			o(\sqrt{\epsilon})
			\right)^2
			+
			o(\epsilon)
			\\
			&=
			m^2\epsilon+o(\epsilon).
		\end{aligned}
		\]
		
		Hence,
		$
		\mathrm{KL}(\mathcal{N}_1\|\mathcal{N}_{m+1})
		\leq
		((m+1)-1)^2\epsilon+o(\epsilon)
		$.
		
		By induction, the result holds for all $k\geq2$.
	\end{proof}
	\subsection{Proof of Corollary~\ref{thm:multiple-steps-perturbation}}\label{sec:proof-thm:multiple-steps-perturbation}
	\begin{proof}
		Let $\delta_k=\mathrm{KL}(\mathcal{N}_1\Vert \mathcal{N}_k)$.
		By Corollary~\ref{cor:multi-step-KL-stability}, we have
		$\delta_k\leq (k-1)^2\epsilon+o(\epsilon)$.
		For sufficiently small $\epsilon$, $\delta_k<1/12$.
		Therefore, we can apply Theorem~\ref{thm:triangle-P-N1-N2} to the
		Gaussian pair $(\mathcal{N}_1,\mathcal{N}_k)$ with
		$\epsilon$ replaced by $\delta_k$. It follows that
		\[
		\left|
		\mathrm{KL}(P\Vert\mathcal{N}_1)
		-
		\mathrm{KL}(P\Vert\mathcal{N}_k)
		\right|
		\leq
		\frac{\sqrt{6d\delta_k}}{2}
		+
		K\sqrt{\delta_k}
		+
		o(\sqrt{\delta_k}),
		\]
		where $K$ is the constant defined in Remark~\ref{remark-K}, evaluated
		for the Gaussian pair $(\mathcal{N}_1,\mathcal{N}_k)$.
		
		Next, using the bound on $\delta_k$, we have
		\[
		\sqrt{\delta_k}
		\leq
		\sqrt{(k-1)^2\epsilon+o(\epsilon)}
		=
		(k-1)\sqrt{\epsilon}+o(\sqrt{\epsilon}),
		\]
		where $k$ is fixed. Similarly,
		\[
		\sqrt{6d\delta_k}
		\leq
		(k-1)\sqrt{6d\epsilon}+o(\sqrt{\epsilon}).
		\]
		Substituting these inequalities into the above bound yields
		\[
		\begin{aligned}
			&|\mathrm{KL}(P\Vert\mathcal{N}_1)-\mathrm{KL}(P\Vert\mathcal{N}_k)|\\
			&\leq
			\frac{(k-1)\sqrt{6d\epsilon}}{2}
			+
			(k-1)K\sqrt{\epsilon}
			+
			o(\sqrt{\epsilon}).
		\end{aligned}
		\]
		This completes the proof.
	\end{proof}
}

\section{Proof of the Finite Second Moment of the Latent Distribution}\label{sec:proof_second_moment}
\begin{proposition}\label{thm:pro_flow_finite_moment}
	Let $X\sim P_X$ satisfy $\mathbb{E}[\|X\|^2]<\infty$, and let $f$ be a globally Lipschitz flow transformation. Then the latent variable $Z=f(X)$ satisfies $\mathbb{E}[\|Z\|^2]<\infty$.
\end{proposition}
\begin{remark}
	In common image applications, normalization to a bounded domain, such as $[-0.5,0.5]^d$, ensures that the input distribution has a finite second moment.
\end{remark}

\begin{proof}
	Let $C$ be the global Lipschitz constant of $f$. That is,
	$\|f(x)-f(y)\|\leq C\|x-y\|$ for all $x,y$.
	
	By the Lipschitz property and the triangle inequality,
	\begin{align}
		\|Z\|^2
		&=\|f(X)\|^2 \nonumber\\
		&\leq \bigl(\|f(\mathbf{0})\|+\|f(X)-f(\mathbf{0})\|\bigr)^2 \nonumber\\
		&\leq \bigl(\|f(\mathbf{0})\|+C\|X\|\bigr)^2 \nonumber\\
		&\leq 2\|f(\mathbf{0})\|^2+2C^2\|X\|^2.
	\end{align}
	Taking expectations gives
	\begin{align}
		\mathbb{E}[\|Z\|^2]
		\leq 2\|f(\mathbf{0})\|^2+2C^2\mathbb{E}[\|X\|^2]
		<\infty.
		\label{eq:finiteZ}
	\end{align}
	Therefore, the latent distribution $P_Z$ has a finite second moment.
\end{proof}

{\color{black}
	\section{Experimental Details}\label{sec:more-exp-details}
	
	\subsection{Details of One-Dimensional Synthetic Experiment}\label{sec:detail_1D_synthetic_exp}
	\textbf{Density of $P$}.
	In this experiment, we construct \(P\) as a non-Gaussian two-component truncated Gaussian mixture. The truncation interval is
	\(
	[L,U]=[-4,4].
	\)
	Let \(\varphi\) and \(\Phi\) denote the probability density function and cumulative distribution function of the standard Gaussian distribution, respectively. For the \(k\)-th component, define the untruncated Gaussian density as
	\[
	g_k(x)
	=
	\frac{1}{s_k}
	\varphi\left(\frac{x-m_k}{s_k}\right),
	\]
	where $m_k$ and $s_k$ are the mean and standard deviation of the $k$-th component, respectively.
	
	After truncation to \([L,U]\), the normalizing constant of the \(k\)-th component is
	\[
	Z_k
	=
	\Phi\left(\frac{U-m_k}{s_k}\right)
	-
	\Phi\left(\frac{L-m_k}{s_k}\right).
	\]
	Thus, the density of \(P\) is given by the weighted sum of two component-wise truncated Gaussian distributions:
	\[
	p(x)
	=
	\sum_{k=1}^{2}
	w_k
	\frac{g_k(x)}{Z_k}
	\mathbf{1}_{[L,U]}(x),
	\]
	where \(w_k>0\) and \(\sum_{k=1}^{2}w_k=1\). In the experiment, we set
	\[
	w_1=w_2=0.5,\ 
	m_1=-2.2,\ 
	m_2=2.2,\ 
	s_1=0.35,\ 
	s_2=0.45.
	\]
	This construction yields a clearly bimodal and non-Gaussian distribution. Since  $P$ is supported on a bounded interval, it has a finite second moment.
	
	\textbf{Additional Results}.	
	\begin{figure}[t!]
		\centering
		\begin{minipage}[t]{6.5cm}
			\centering
			\includegraphics[width=6.5cm]{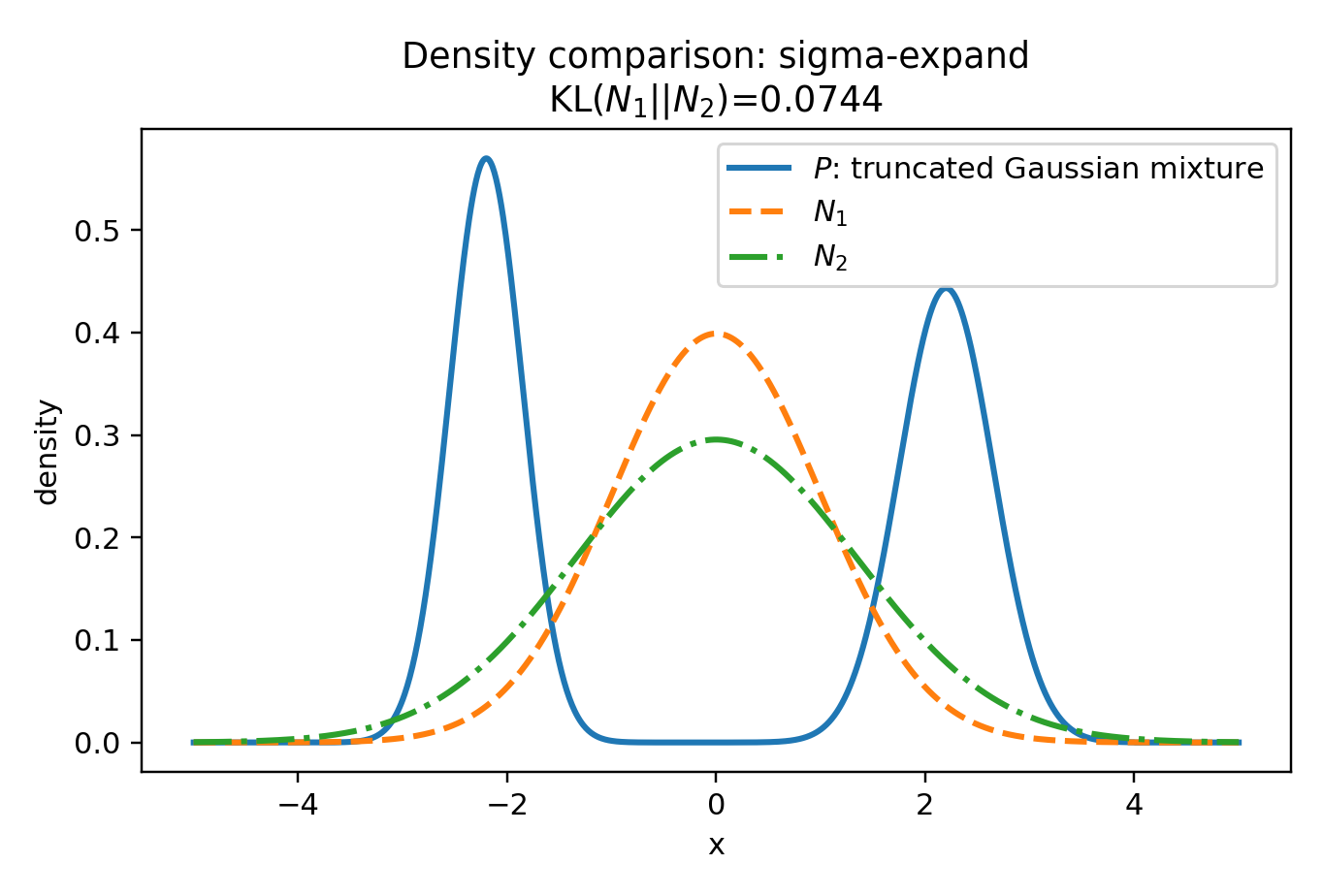}
			\vspace{-15pt}
		\end{minipage}
		\caption{$P$ is a two-component truncated Gaussian mixture, $\N_1$ is the standard Gaussian distribution, and $\N_2$ is obtained by increasing the standard deviation of $\N_1$.}
		\label{fig:1D_P_N1_N2}
	\end{figure}
	Fig.~\ref{fig:1D_P_N1_N2} shows $P$, $\N_1$, and a representative $\N_2$ obtained by increasing the standard deviation of $\N_1$.

	\begin{figure}[t!]
		\scriptsize
		\centering
		\scriptsize
		\subfloat[]{
			\begin{minipage}[t]{6.5cm}
				\centering
				\includegraphics[width=6.5cm]{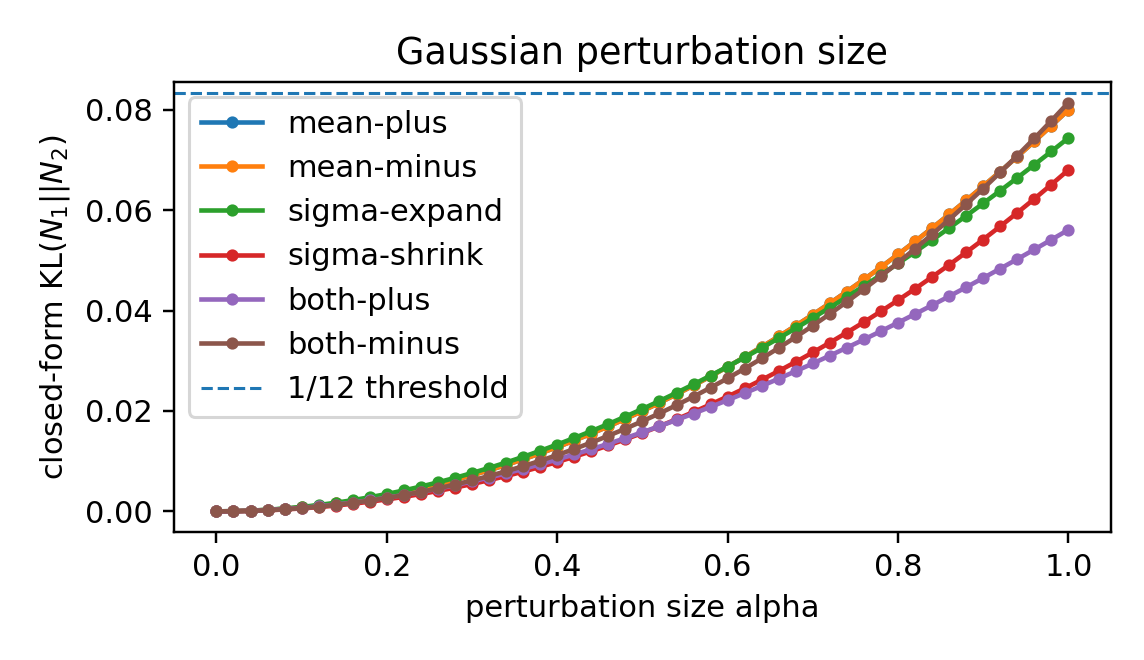}
				\vspace{-20pt}
				\label{fig:1d_kl_vs_perturbation}
			\end{minipage}
		}
		
		%\vspace{-10pt}
		\subfloat[]{
			\begin{minipage}[t]{6.5cm}
				\centering
				\includegraphics[width=6.5cm]{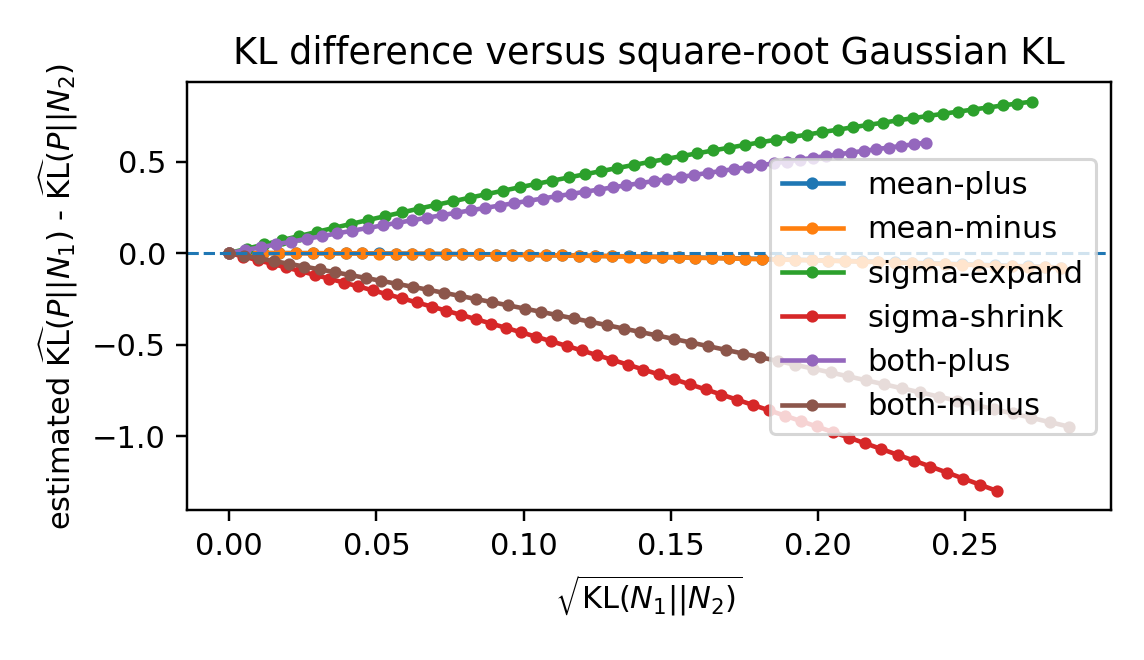}
				\vspace{-20pt}
				\label{fig:1D_kl_difference_vs_sqrt_epsilon}
			\end{minipage}
		}
		\caption{KL behavior under one-dimensional Gaussian perturbations. 
			(a) \(\mathrm{KL}(\N_1\|\N_2)\) versus the perturbation strength \(\alpha\). 
			(b) \(\mathrm{KL}(P\|\N_1)-\mathrm{KL}(P\|\N_2)\) versus \(\sqrt{\mathrm{KL}(\N_1\|\N_2)}\). 
			The curve of \textit{mean-plus} almost overlaps with that of \textit{mean-minus}.}
		\label{fig:1D_KL_eps}
	\end{figure}
	
	Fig. \ref{fig:1D_KL_eps} shows \(\mathrm{KL}(\N_1\|\N_2)\) versus the perturbation strength \(\alpha\) and \(\mathrm{KL}(P\|\N_1)-\mathrm{KL}(P\|\N_2)\) versus \(\sqrt{\mathrm{KL}(\N_1\|\N_2)}\).

	Fig. \ref{fig:bound-against-e-1d} plots the estimated theoretical bound $\widehat{B}(\e)$ against $\e$ in the one-dimensional synthetic experiment.
	\begin{figure}[t!]
		\centering
		\begin{minipage}[t]{7.5cm}
			\centering
			\includegraphics[width=6.5cm]{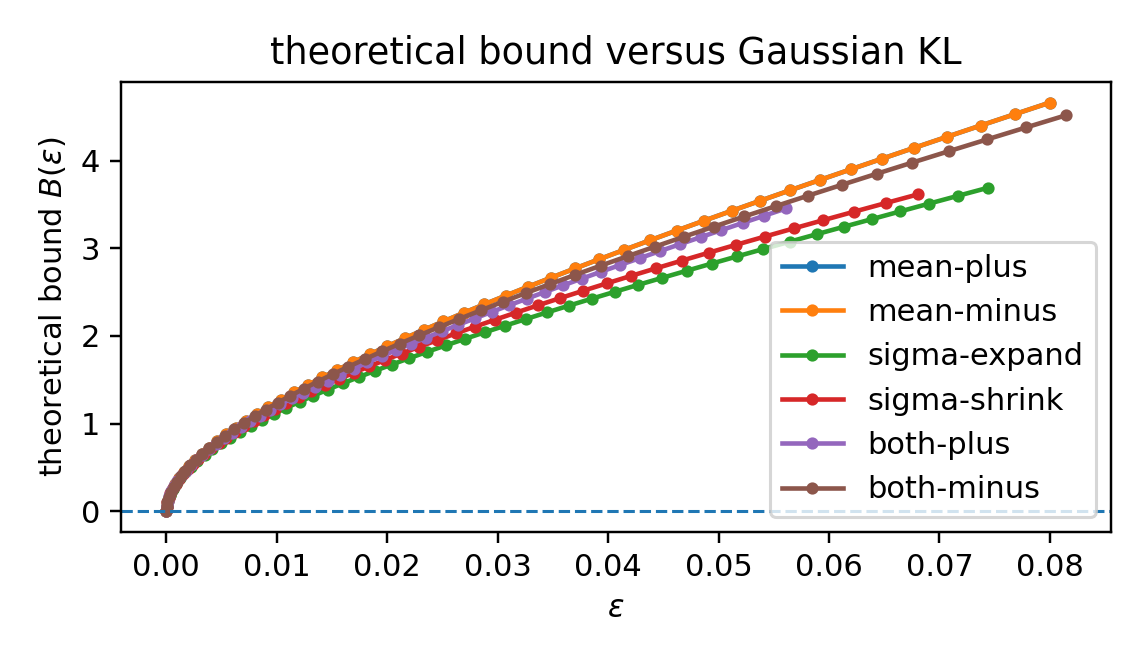}
		\end{minipage}
		\caption{The estimated theoretical bound $\widehat{B}(\e)$ in the one-dimensional synthetic experiment.
			Each marker corresponds to a value of $\alpha$.}
		\label{fig:bound-against-e-1d}
	\end{figure}

	\subsection{Details of Two-Dimensional Synthetic Experiment}\label{sec:detail_2D_synthetic_exp}
	\textbf{Density of $P$}.
	We define \(P\) as a two-component truncated Gaussian mixture. The truncation region is the rectangle
	\(
	A=[L,U]\times [L,U],
	\)
	where
	\(
	L=-4,\ U=4.
	\)
	Let \(\varphi\) and \(\Phi\) denote the probability density function and cumulative distribution function of the standard one-dimensional Gaussian distribution, respectively. For the \(k\)-th component, we use an axis-aligned Gaussian distribution with mean
	\(
	m_k=(m_{k1},m_{k2})^\top
	\)
	and coordinate-wise standard deviations
	\(
	s_k=(s_{k1},s_{k2})^\top.
	\)
	The untruncated density of the \(k\)-th component is
	\[
	g_k(x)=\prod_{j=1}^{2}
	\frac{1}{s_{kj}}
	\varphi\left(\frac{x_j-m_{kj}}{s_{kj}}\right).
	\]
	Since each component is truncated to the rectangle \(A\), its normalizing constant is
	\[
	Z_k=\prod_{j=1}^{2}
	\left[
	\Phi\left(\frac{U-m_{kj}}{s_{kj}}\right)
	-
	\Phi\left(\frac{L-m_{kj}}{s_{kj}}\right)
	\right].
	\]
	Therefore, the density of the truncated Gaussian mixture \(P\) is
	\[
	p(x)=\sum_{k=1}^{2}
	w_k
	\frac{g_k(x)}{Z_k}
	\mathbf{1}_{A}(x),
	\]
	where \(w_k>0\) and \(\sum_{k=1}^{2}w_k=1\). 
	
	In the experiment, we set
	\[
	w_1=0.55,\qquad w_2=0.45,
	\]
	and
	\begin{align}
		m_1&=(-2.0,-1.4)^\top
		&
		s_1&=(0.45,0.75)^\top \nonumber\\
		m_2&=(2.1,1.5)^\top
		&
		s_2&=(0.65,0.40)^\top. \nonumber
	\end{align}
	This construction yields a bimodal and non-Gaussian distribution. Moreover, since \(P\) is supported on the bounded region \(A\), it has a finite second moment.
	
	\textbf{Perturbations}. In the two-dimensional synthetic experiment, we use the following perturbations on $\N_1=\N(\mathbf{0}, \mathbf{I}_2)$.
	\begin{enumerate}
		\item \textit{mean-x}: $\mu_2=(0.4\alpha, 0)^{\top}, \Sigma_2=\mathbf{I}_2$;
		\item \textit{mean-diag}: $\mu_2=(0.28\alpha, -0.28\alpha)^{\top}, \Sigma_2=\mathbf{I}_2$;
		\item \textit{cov-x-expand}: $\mu_2=0, \Sigma_2=
		\begin{pmatrix}
			\exp(0.35\alpha) & 0\\
			0 & 1
		\end{pmatrix}$;
		\item \textit{cov-y-shrink}: $$\mu_2=\bm{0}, \Sigma_2= 
		\begin{pmatrix}
			1 & 0\\
			0 & \exp(-0.3\alpha)
		\end{pmatrix};$$
		\item \textit{cov-diag-expand}:$$\mu_2=0, \Sigma_2
		=
		R_{\pi/4}
		\begin{pmatrix}
			\exp(0.35\alpha) & 0\\
			0 & 1
		\end{pmatrix}
		R_{\pi/4}^{\top},$$
		where \(R_\theta\) denotes the two-dimensional rotation matrix
		\[
		R_\theta
		=
		\begin{pmatrix}
			\cos\theta & -\sin\theta\\
			\sin\theta & \cos\theta
		\end{pmatrix};
		\]
		\item \textit{cov-rot-anisotropic}: a rotated anisotropic covariance perturbation
		$$
		\mu_2=0,
		\Sigma_2
		=
		R_{\pi/6}
		\begin{pmatrix}
			\exp(0.28\alpha) & 0\\
			0 & \exp(-0.24\alpha)
		\end{pmatrix}
		R_{\pi/6}^{\top};$$
		\item \textit{both-mean-cov-rot}: a joint mean-covariance perturbation
		\begin{align}
			&\mu_2=(0.22\alpha,-0.18\alpha)^\top,\nonumber\\
			&\Sigma_2
			=
			R_{\pi/5}
			\begin{pmatrix}
				\exp(0.22\alpha) & 0\\
				0 & \exp(-0.16\alpha)
			\end{pmatrix}
			R_{\pi/5}^{\top}.\nonumber
		\end{align}
		
	\end{enumerate}
	
	Fig. \ref{fig:2D_P_N1_N2} shows $P$, $\N_1$, and one representative $\N_2$ that is obtained by perturbing \(\N_1\) with covariance expansion along the $x$-axis.
	\begin{figure}[t!]
		\centering
		\begin{minipage}[t]{6.5cm}
			\centering
			\includegraphics[width=6.5cm]{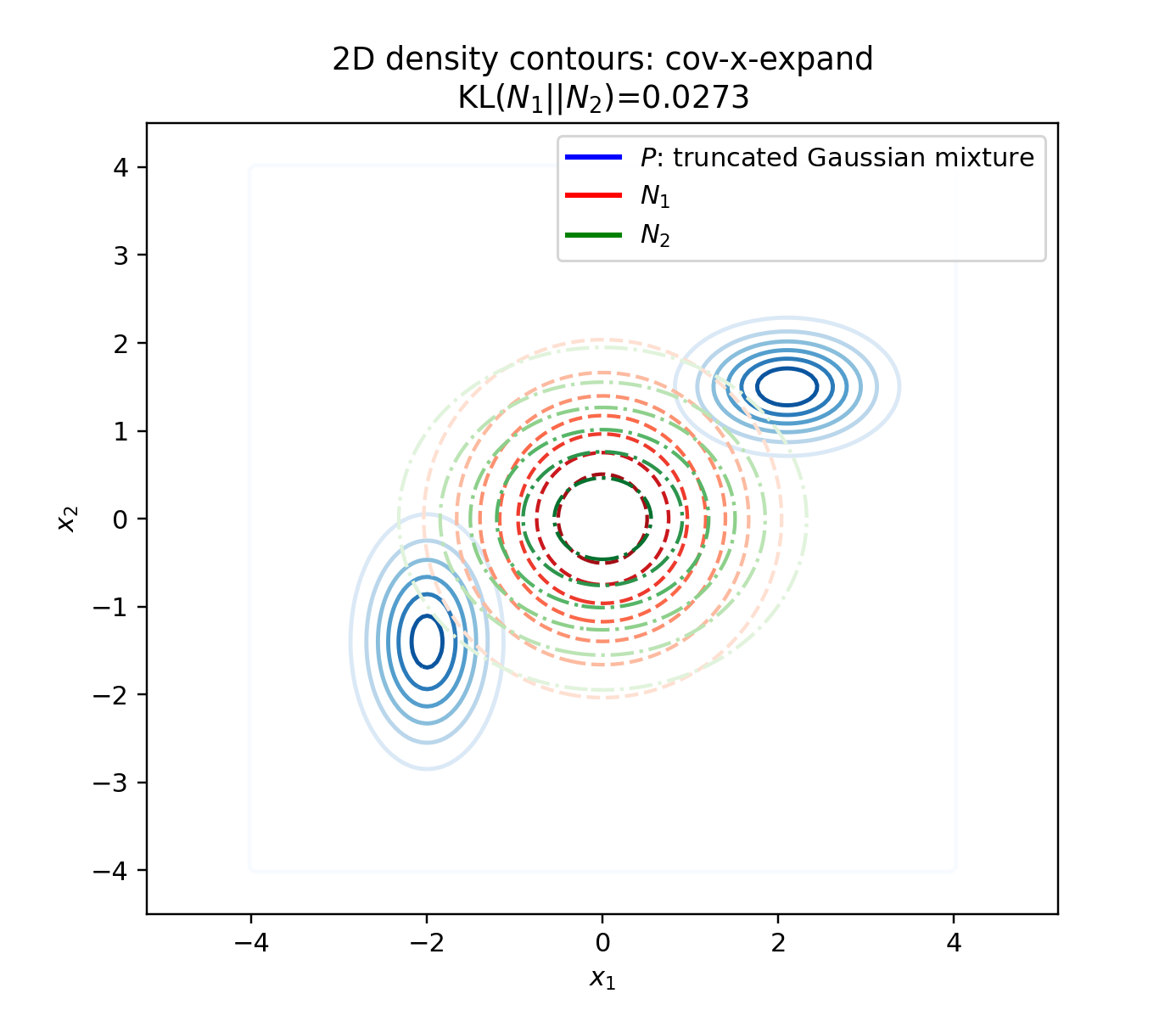}
			\label{fig:2D_density_cov-x-expand}
		\end{minipage}
		\caption{\(P\) is a two-component truncated Gaussian mixture, \(\N_1\) is the standard Gaussian distribution, and \(\N_2\) is obtained by perturbing \(\N_1\) with covariance expansion along the $x$-axis.}
		\label{fig:2D_P_N1_N2}
	\end{figure}
	
	\textbf{Additional Results}.
	\begin{figure}[t!]
		\centering
		\subfloat[]{
			\begin{minipage}[t]{6.5cm}
				\centering
				\includegraphics[width=6.5cm]{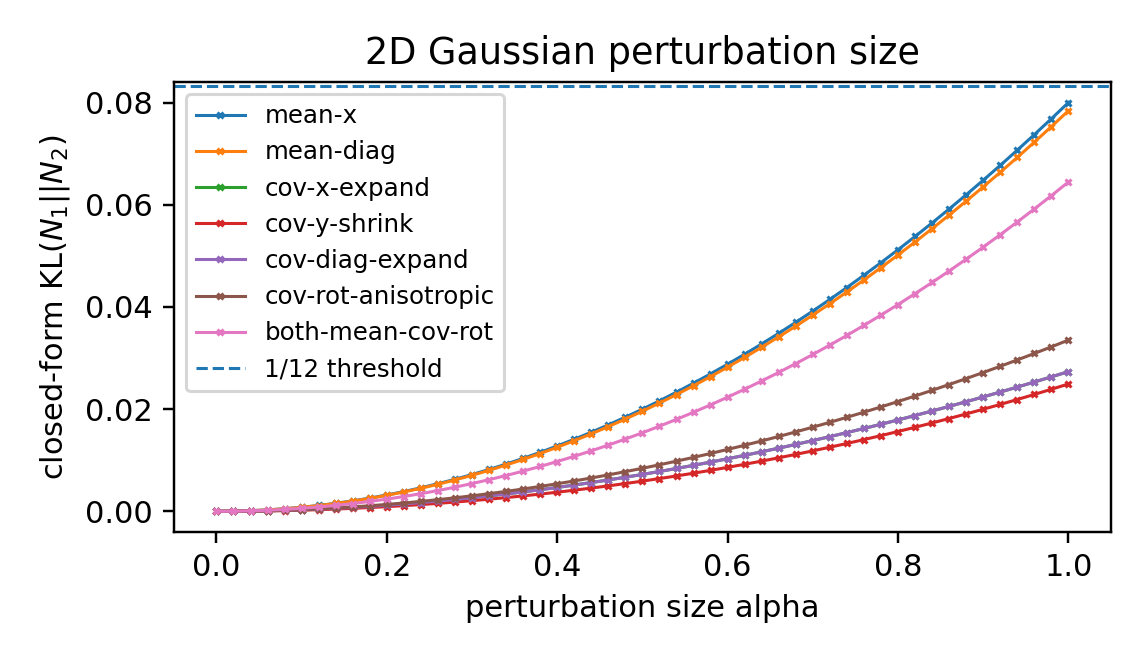}
				\vspace{-20pt}
				\label{fig:2d_kl_vs_perturbation}
			\end{minipage}
		}
		
		\subfloat[]{
			\begin{minipage}[t]{6.5cm}
				\centering
				\includegraphics[width=6.5cm]{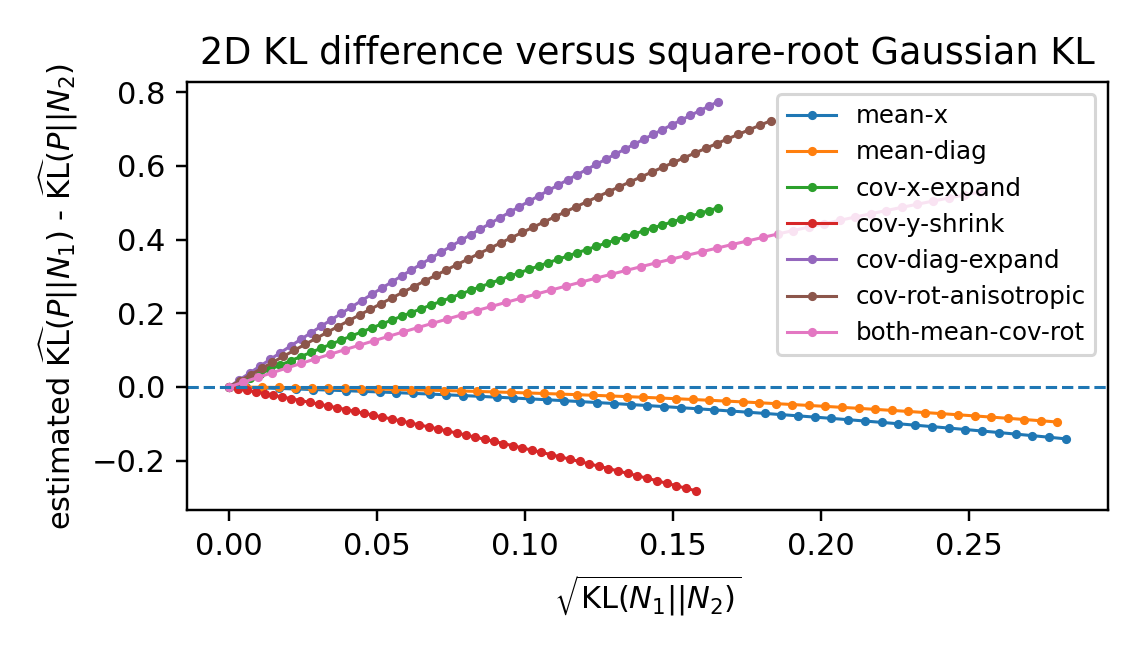}
				\vspace{-20pt}
				\label{fig:2D_kl_difference_vs_sqrt_epsilon}
			\end{minipage}
		}
		\caption{KL behavior under two-dimensional Gaussian perturbations. 
			(a) \(\mathrm{KL}(\N_1\|\N_2)\) versus the perturbation strength \(\alpha\). 
			The \textit{cov-x-expand} and \textit{cov-diag-expand} curves overlap. 
			(b) \(\widehat{\mathrm{KL}}(P\|\N_1)-\widehat{\mathrm{KL}}(P\|\N_2)\) versus \(\sqrt{\mathrm{KL}(\N_1\|\N_2)}\). 
			Each dot on a curve corresponds to one value of \(\alpha\). }
		\label{fig:2D_KL_eps}
	\end{figure}
	Fig. \ref{fig:2D_KL_eps} plots \(\mathrm{KL}(\N_1\|\N_2)\) versus the perturbation strength \(\alpha\) and \(\widehat{\mathrm{KL}}(P\|\N_1)-\widehat{\mathrm{KL}}(P\|\N_2)\) versus \(\sqrt{\mathrm{KL}(\N_1\|\N_2)}\). 
	
	Fig. \ref{fig:bound-against-e-2d} shows the estimated theoretical bound $\widehat{B}(\e)$ against $\e$ in the two-dimensional synthetic experiment.
	\begin{figure}[t!]
		\centering
		\begin{minipage}[t]{6.5cm}
			\centering
			\includegraphics[width=6.5cm]{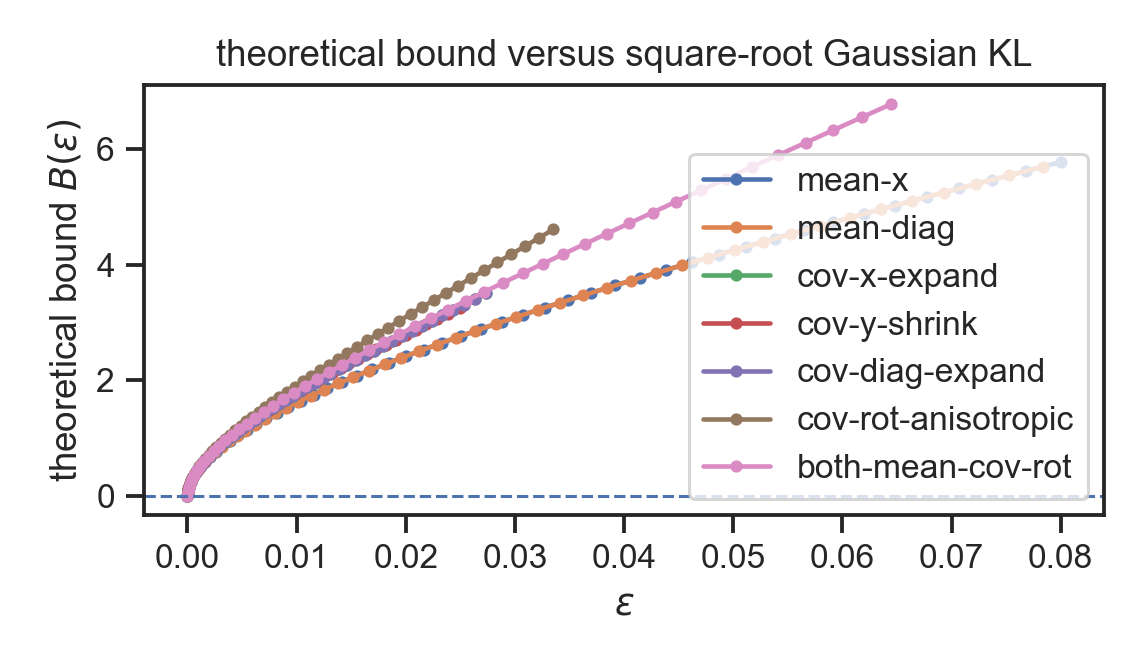}
		\end{minipage}
		\caption{The estimated theoretical bound $\widehat{B}(\e)$ in the two-dimensional synthetic experiment. Each marker corresponds to a value of $\alpha$.}
		\label{fig:bound-against-e-2d}
	\end{figure}

	\subsection{Details of Experiments on Two-Dimensional Flow-Based Models}\label{sec:appendix-exp-nsf-details}
	\textbf{Models}.
	The flow consists of two rational-quadratic spline coupling transforms with alternating binary masks. 
	Each transform uses $64$ spline bins and linear tails with a tail bound of $5$, 
	and its parameters are produced by a residual network with two residual blocks and a hidden dimension of $32$. 
	We train the model by maximum likelihood using the Adam optimizer with an initial learning rate of $5\times10^{-4}$, 
	a batch size of $512$, cosine learning-rate annealing, and $5\times10^4$ training steps. All models are trained on an NVIDIA GeForce RTX 2080 Ti GPU.
	
	\textbf{Additional Results}.
	{\color{black}
		Table~\ref{tbl:ID_2d_flow} reports the Henze--Zirkler normality test results for the ID representations, 
		together with $\mathrm{KL}(\widehat{P}_Z\Vert \N(\mathbf{0},\mathbf{I}_2))$, where $\widehat{P}_Z$ denotes the Gaussian distribution fitted to the ID representations.
	}
	
	\textbf{Figures}.
	Fig. \ref{fig:all-four-datasets} shows datasets used in the two-dimensional NSF experiments. Fig. \ref{fig:nsf_abs} shows the learned density and samples generated by NSF trained on the \textit{abs} dataset. Fig. \ref{fig:nsf_abs_vs_others_representations} shows the representations of ID (\textit{abs}) and OOD datasets. 
	\begin{figure*}[t!]
		\centering
		\includegraphics[width=17cm]{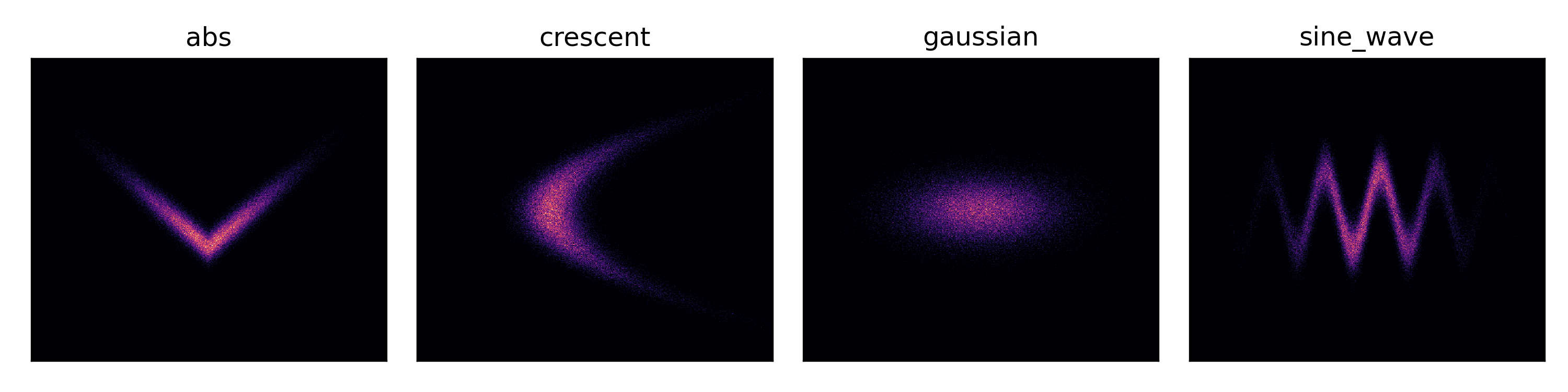}
		%\vspace{-10pt}
		\caption{Datasets used in the two-dimensional neural spline flow experiments.}
		\label{fig:all-four-datasets}
	\end{figure*}
	
	\begin{figure}[t!]
		\centering
		\includegraphics[width=8cm]{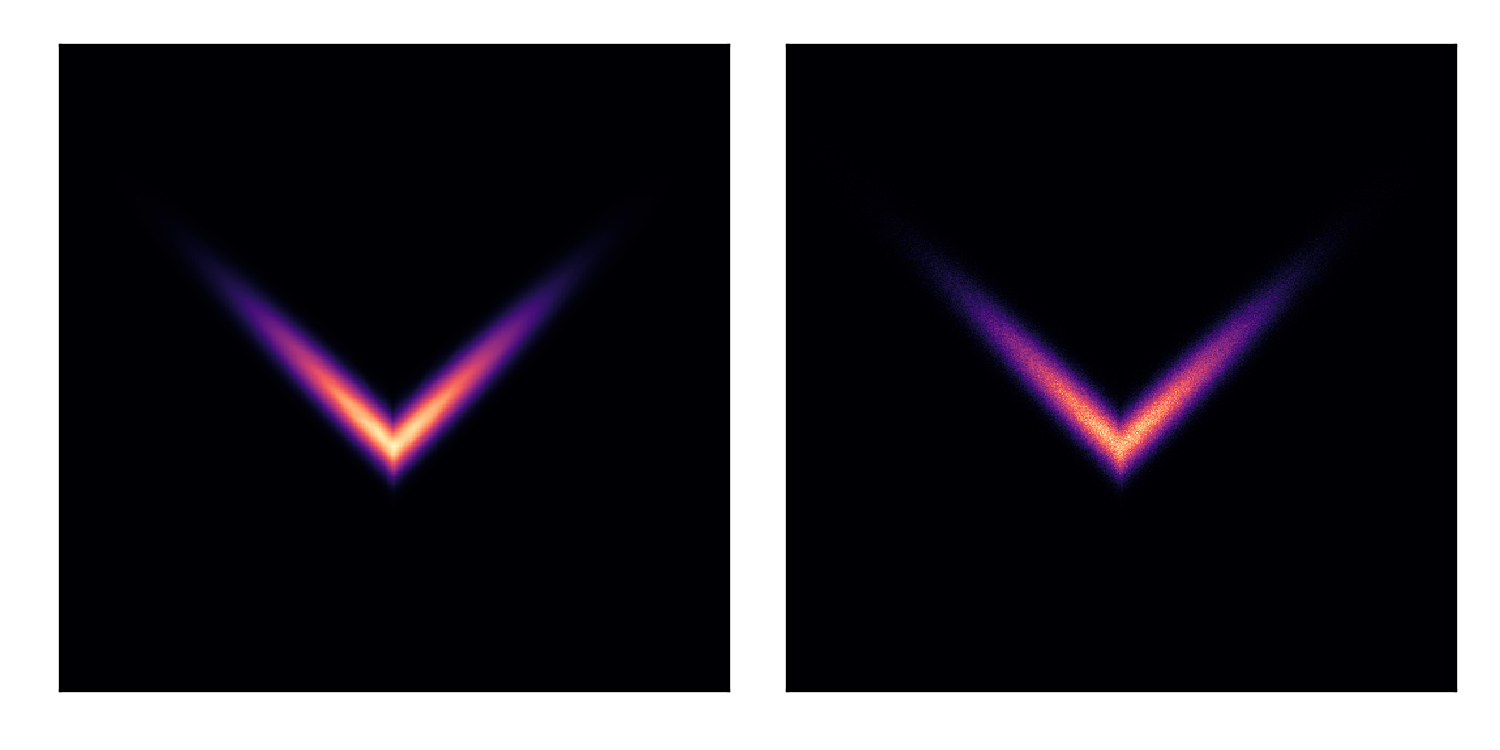}
		%\vspace{-10pt}
		\caption{Neural spline flow model trained on the \textit{abs} dataset. Left: density learned by the flow model; Right: samples generated by the flow-based model.}
		\label{fig:nsf_abs}
	\end{figure}

	\begin{figure*}[t!]
		\centering
		\includegraphics[width=17cm]{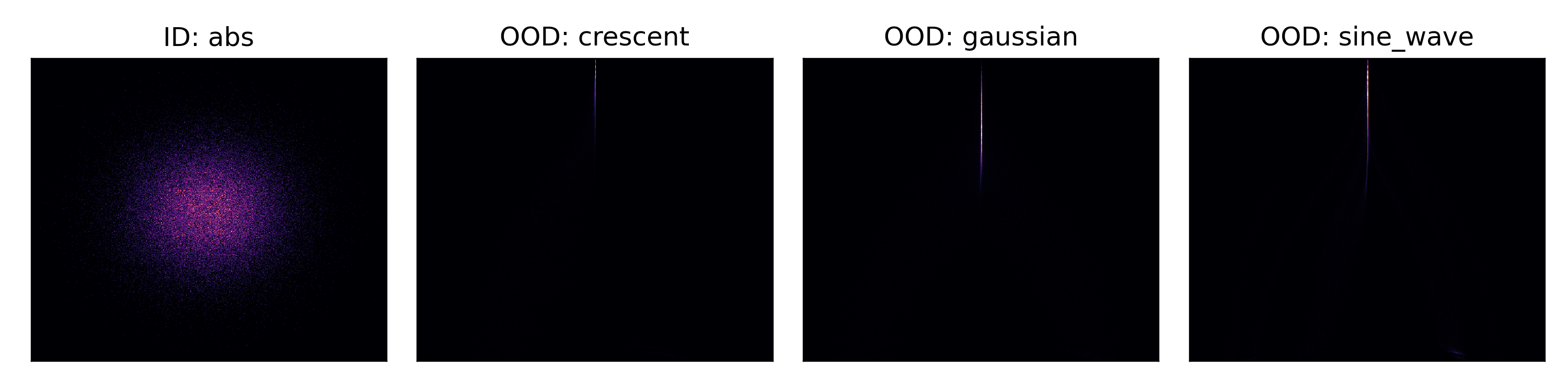}
		%\vspace{-10pt}
		\caption{Representations of ID (\textit{abs}) and OOD datasets.}
		\label{fig:nsf_abs_vs_others_representations}
	\end{figure*}
	
	\textbf{Discussion}.
	The directions of the KL divergences considered in this subsection differ from those in the motivating research problem in Section~\ref{problem}. 
	First, unlike Assumption~\ref{assumption0}, Assumption~\ref{assumption1} only requires  $\mathrm{KL}(Q_Z\Vert P_Z)$ 
	to be large. Second, previous work \cite{zhang2023outofdistribution} concludes that both $\mathrm{KL}(Q_Z\Vert P_Z^r)$ and $\mathrm{KL}(P_Z^r\Vert Q_Z)$ are large, 
	but the derivation in \cite{zhang2023outofdistribution} relies on the strong approximation $P_Z\approx P_Z^r$. 
	By contrast, our analysis establishes a guarantee only for $\mathrm{KL}(Q_Z\Vert P_Z^r)$. 
	This one-sided guarantee is sufficient for the KL-based OOD detection framework considered in~\cite{zhang2023outofdistribution}.

	\subsection{Details of Experiments on Higher-Dimensional Flow-Based Models}\label{sec:detail_higher_d_flow}
	\textbf{Models.} We use the official Glow architecture~\cite{glowopenai}. 
	The model has a multiscale architecture with three stages and 32 flow steps per stage. 
	At the end of each of the first two stages, the output is split into two parts: one is factored out as part of the final latent representation, and the other is passed to the next stage.
	For CelebA32 and SVHN, we use a standard Gaussian prior and train the models from scratch. For CIFAR-10, we use the pretrained checkpoint released by DeepMind~\cite{glowdeepmind}.
	We train the CelebA32 and SVHN models using Adamax with a batch size of 64. 
	The learning rate is linearly increased from 0 to 0.001 during the first 10 epochs and then held constant for the remaining epochs. 
	We train the SVHN and CelebA32 models for 390 and 2000 epochs on an NVIDIA GeForce RTX 2080 Ti GPU, respectively.
	
	\textbf{Additional Results}.
	\begin{table}[!t]
		\scriptsize
		\centering
		\caption{$p$-value: Henze-Zirkler normality test results on the ID representations from the two-dimensional NSF models.}
		\label{tbl:ID_2d_flow}
		\begin{tabular}{lcc}
			\toprule
			ID dataset & $p$-value & $\mathrm{KL}(\widehat{P_Z}\Vert \N(\mathbf{0},\mathbf{I}_2))$ \\ 
			\midrule
			\textit{abs} & $0.47$  & $7.48\times 10^{-6}$ \\ 
			\textit{crescent} & $0.37$  & $5.34\times 10^{-5}$ \\ 
			\textit{gaussian} & $0.39$  & $2.21\times 10^{-5}$ \\ 
			\textit{sine\_wave} & $0.39$ & $1.84\times 10^{-5}$ \\ 
			\bottomrule
		\end{tabular}
	\end{table}
	Table~\ref{tbl:ratio_group} reports the proportions of groups for which $\mathrm{KL}(\widehat{P}_Z^{(g)}\Vert \N(\mathbf{0},\mathbf{I}_4))<1/12$ and the proportions for which the Henze--Zirkler test yields $p>0.05$.
	
	\begin{table}[!t]
		\scriptsize
		\centering
		\caption{Proportions of dimension groups of ID representations.}
		\label{tbl:ratio_group}
		\begin{tabular}{lcc}
			\toprule
			Dataset & $\mathrm{KL}(\widehat{P}_Z\Vert \N(\mathbf{0},\mathbf{I}_4)) < 1/12$ & $p>0.05$ \\ \midrule
			CelebA32 & 98.6\% & 93.0\% \\ 
			SVHN & 99.3\% & 94.4\% \\ 
			CIFAR-10 & 99.1\% & 91.1\% \\ 
			\bottomrule	\end{tabular}
	\end{table}

	\begin{figure}[t!]
		\scriptsize
		\centering
		\subfloat[\textit{{\scriptsize ID: CelebA32}}]{
			\begin{minipage}[t]{8cm}
				\centering
				\includegraphics[width=8cm]{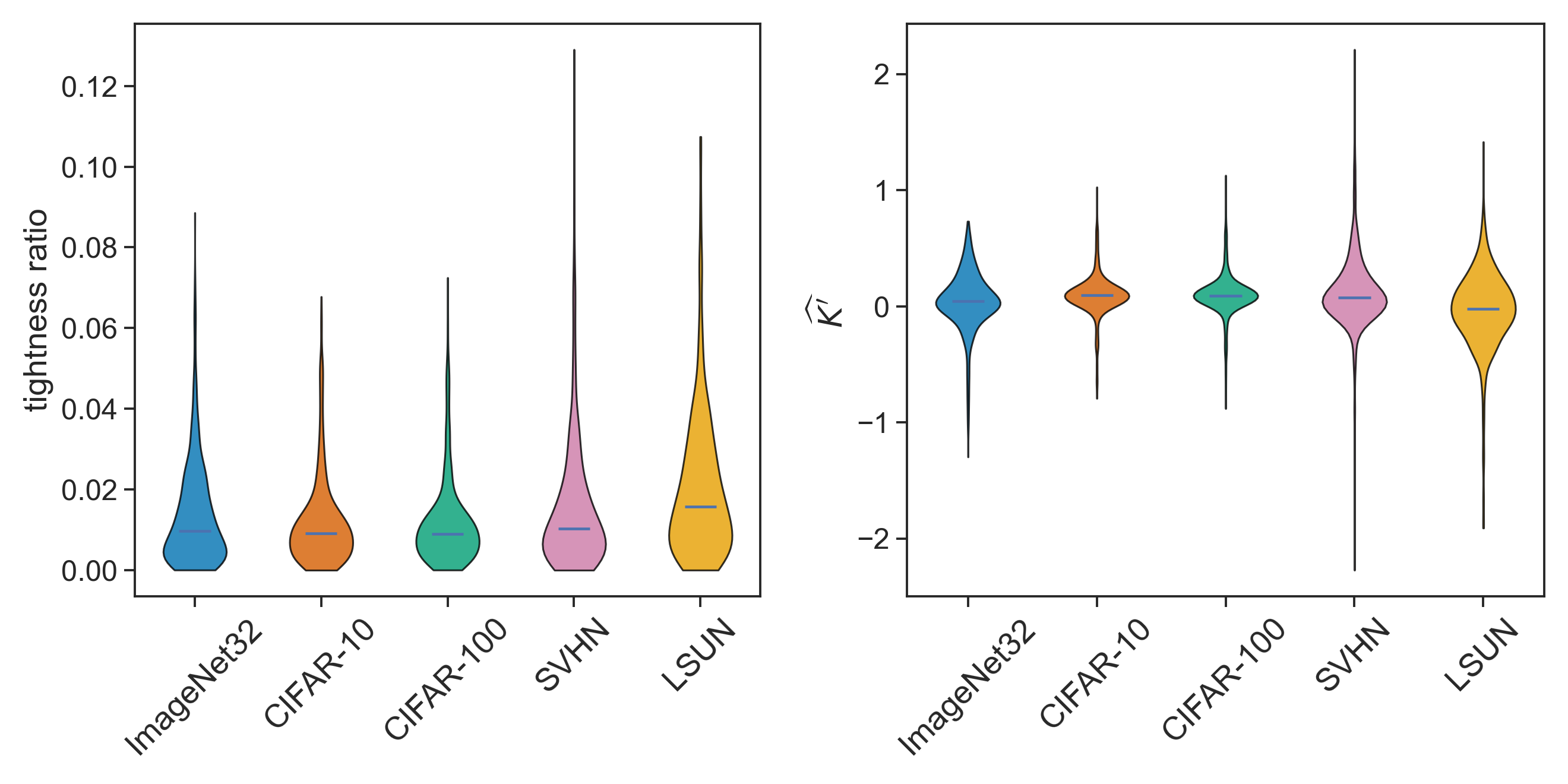}
				\vspace{-5pt}
				\label{fig:glow_celeba32_merged_K_tightness}
			\end{minipage}
		}	
		
		\subfloat[\textit{{\scriptsize ID: SVHN}}]{
			\begin{minipage}[t]{8cm}
				\centering
				\includegraphics[width=8cm]{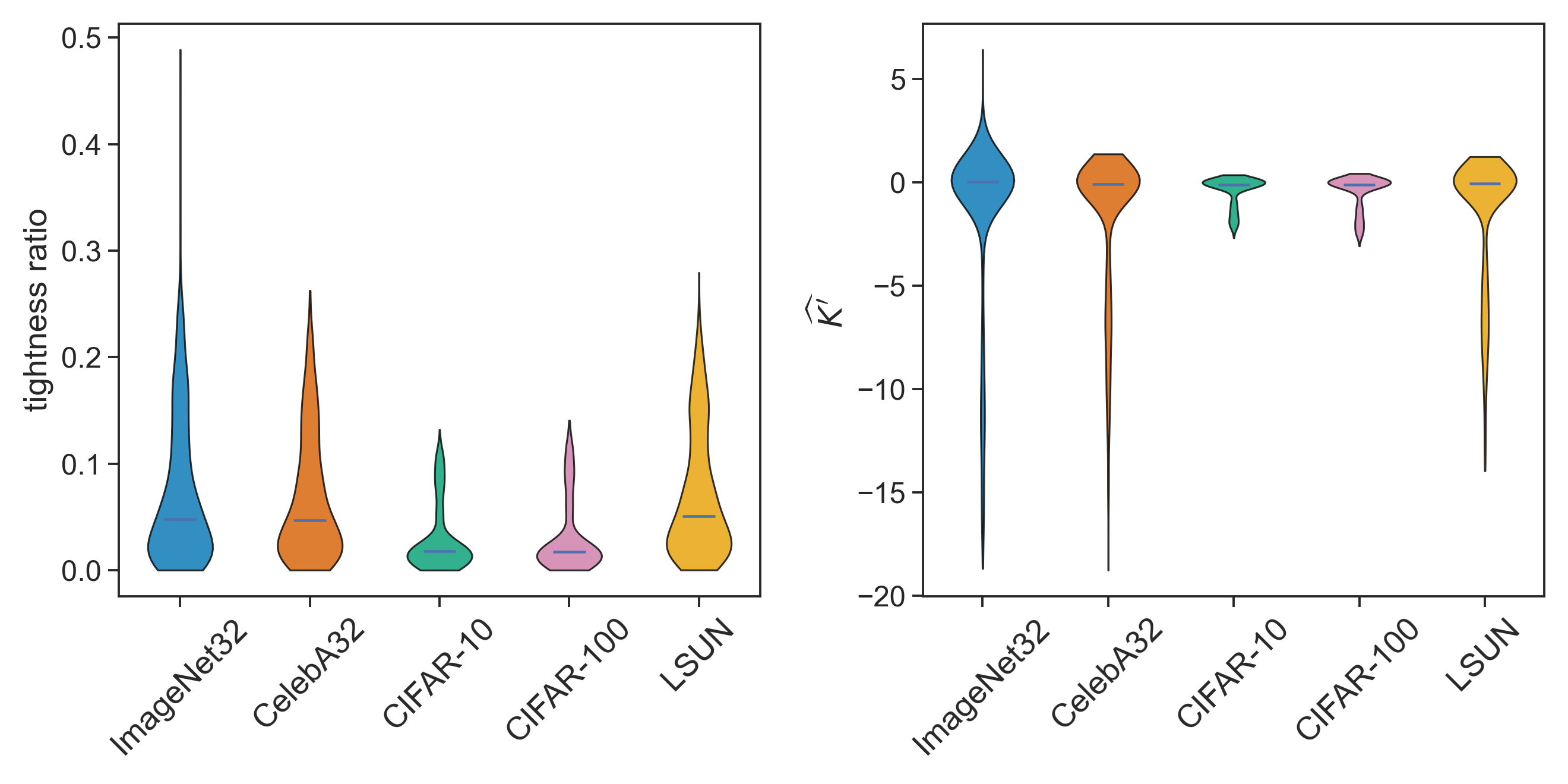}
				\vspace{-5pt}
				\label{fig:glow_svhn_merged_K_tightness}
			\end{minipage}
		}
		
		\subfloat[\textit{{\scriptsize ID: CIFAR-10}}]{
			\begin{minipage}[t]{8cm}
				\centering
				\includegraphics[width=8cm]{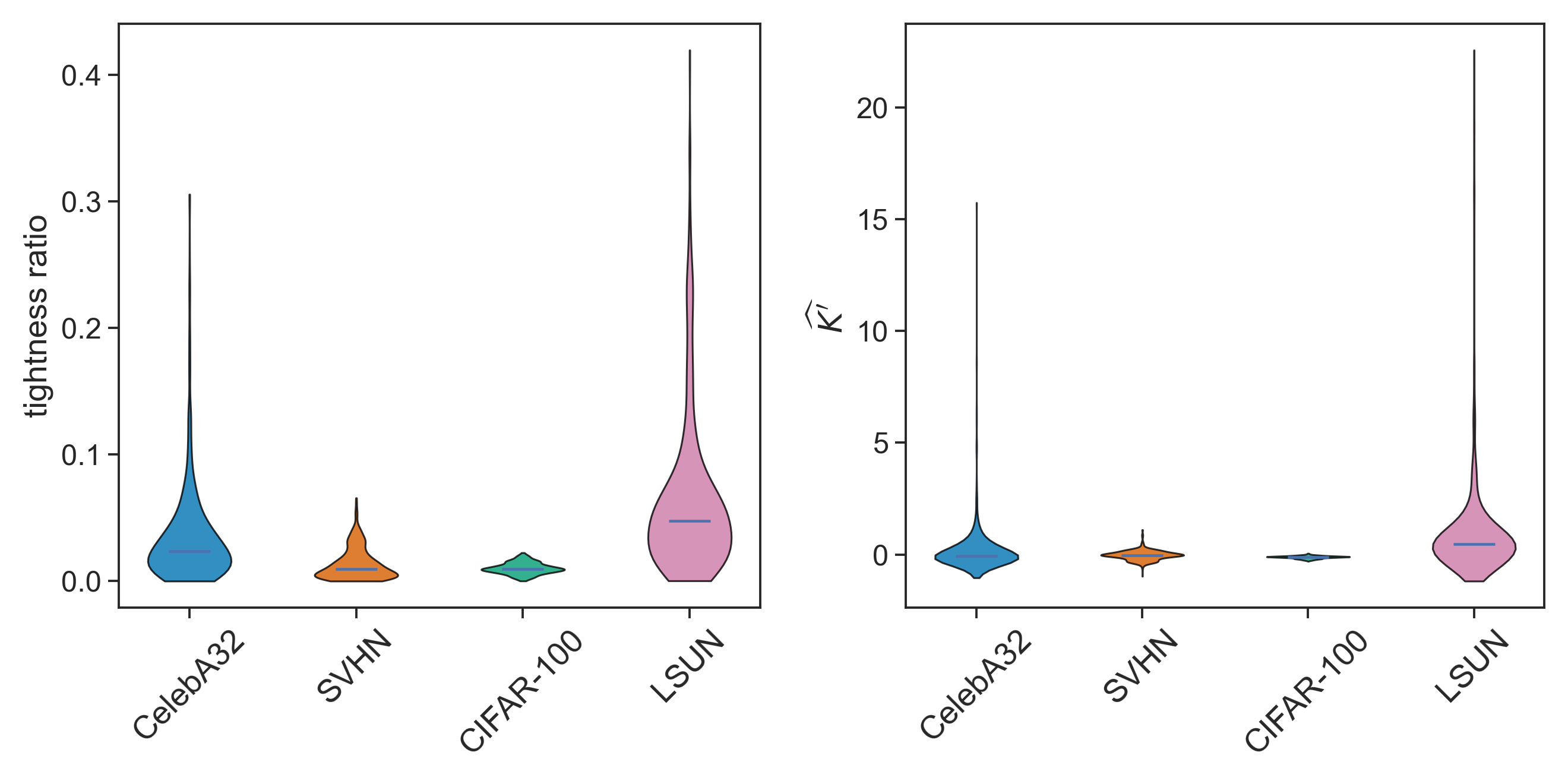}
				\vspace{-5pt}
				\label{fig:glow_cifar10_merged_K_tightness}
			\end{minipage}
		}
		\caption{Distributions of the tightness ratio $\mathrm{TR}$ and empirical coefficient $\widehat{K^{\prime}}$ across groups of Glow latent dimensions.
			Dimension groups satisfying $\mathbb{E}[\|Z\|^2]>50$ are excluded from the violin plots for readability.}
		\label{fig:glow_K_tightness}
	\end{figure}

	Fig.~\ref{fig:glow_Ex2_F3_K'_ratio} shows the estimated $\mathbb{E}[\|z\|^2]$ and the ratio $F_3/K'$ for each dimension group when testing OOD datasets on Glow models.
	
	Fig.~\ref{fig:glow_KL_on_large_Ex2_groups} compares 
	$\mathrm{KL}(\widehat{P}_Z^{(g)}\Vert\mathcal{N}(\mathbf{0},\mathbf{I}_4))$
	for all groups of ID representations and
	$\mathrm{KL}(\widehat{Q}_Z^{(g)}\Vert\mathcal{N}(\mathbf{0},\mathbf{I}_4))$
	for OOD groups satisfying $\mathbb{E}[\|Z\|^2]>50$.
	For these dimension groups, $\mathrm{KL}(\widehat{Q}_Z^{(g)}\Vert\mathcal{N}(\mathbf{0},\mathbf{I}_4))$ is consistently larger than $\mathrm{KL}(\widehat{P}_Z^{(g)}\Vert\mathcal{N}(\mathbf{0},\mathbf{I}_4))$.
	These results are consistent with the theoretical analysis (see Equation~\eqref{eq:oneside-KL-bound-flow} in the main text).	
	
	Table \ref{tbl:bigEX} shows the proportions of dimension groups where {\color{black}$\mathbb{E}[\|Z\|^2]>50$} %$\mathbb{E}[\|Z\|^2]\geq50$
	and the Pearson correlation coefficient between $\widehat{K'}$ and $\mathbb{E}[\|Z\|^2]$.
	
	\begin{table}[!ht]
		\scriptsize
		\centering
		\caption{Proportions of dimension groups where $ \mathbb{E}[\|Z\|^2] > 50$.
			ID--OOD pairs with no such dimension groups are omitted.
			The Pearson correlation coefficient $r$ between $\widehat{K'}$ and $\mathbb{E}[\|Z\|^2]$, together with the  $p$-value.}
		\label{tbl:bigEX}
		\begin{tabular}{llrrc}
			\toprule
			ID & OOD & Proportion & Pearson's $r$ & $p$-value \\ \midrule
			SVHN & ImageNet32 & 13.8\%  & -0.99  & $<10^{-10}$\\ 
			& LSUN & 0.5\%  & -0.95  & $<10^{-10}$ \\ \midrule
			CIFAR-10 & CelebA32 & 2.2\%  & 1.00  & $<10^{-10}$ \\ 
			& LSUN & 7.5\%  & 1.00  & $<10^{-10}$ \\ \bottomrule
		\end{tabular}
	\end{table}
	
	These results indicate that although dimension groups with large second moments may exhibit relatively large KL variations, such variations do not necessarily obscure the separation between ID and OOD latent representations. 
	Specifically, the OOD groups with large KL differences remain farther from the standard Gaussian prior than ID groups.

	\begin{figure}[t!]
		\scriptsize
		\centering
		\subfloat[\textit{{\scriptsize ID: CelebA32}}]{
			\begin{minipage}[t]{8cm}
				\centering
				\includegraphics[width=8cm]{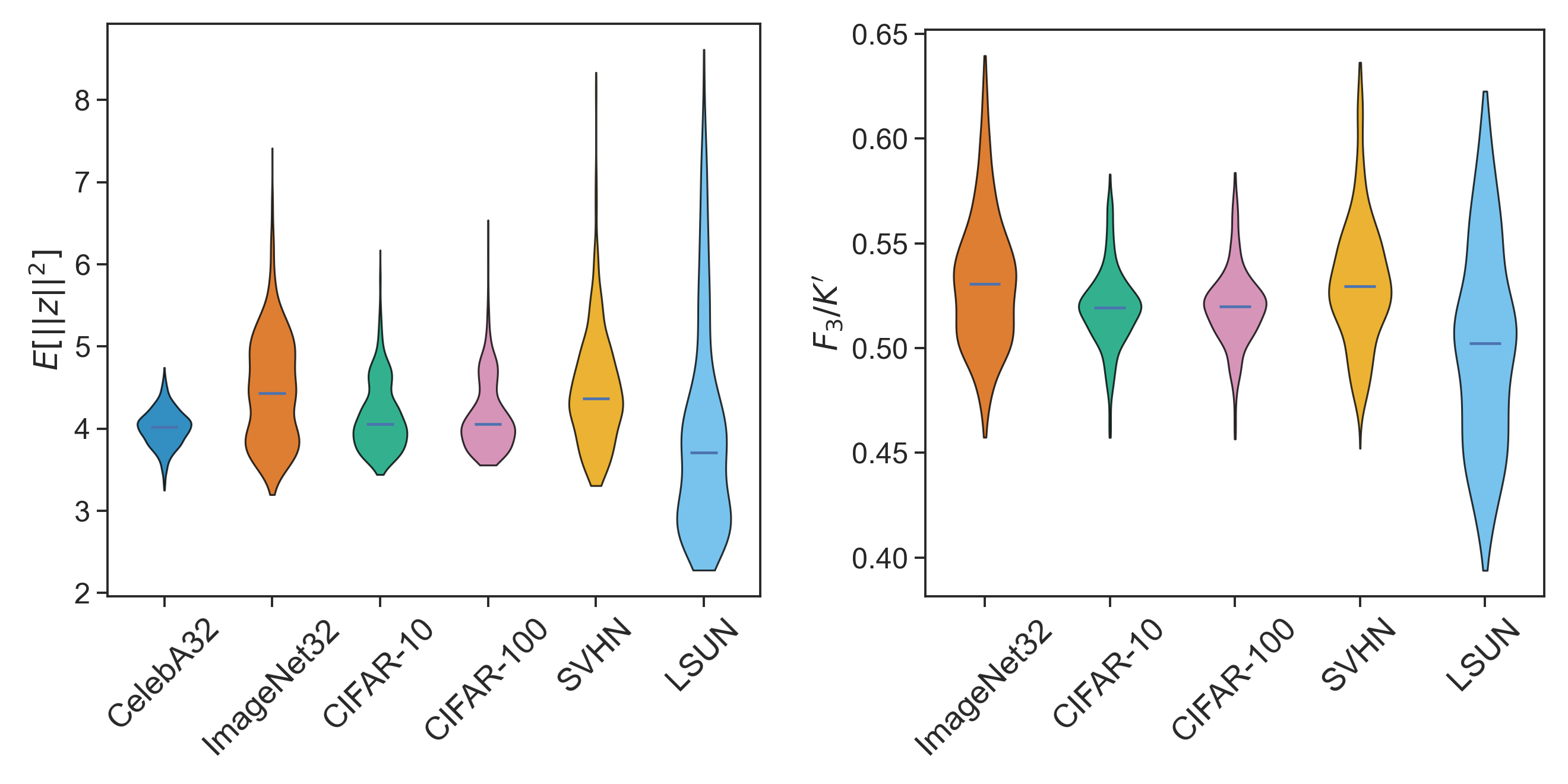}
				\vspace{-15pt}
				\label{fig:glow_celeba32_merged_Ex2_F3_K'_ratio}
			\end{minipage}
		}
		
		\vspace{-5pt}
		\subfloat[\textit{{\scriptsize ID: SVHN}}]{
			\begin{minipage}[t]{8cm}
				\centering
				\includegraphics[width=8cm]{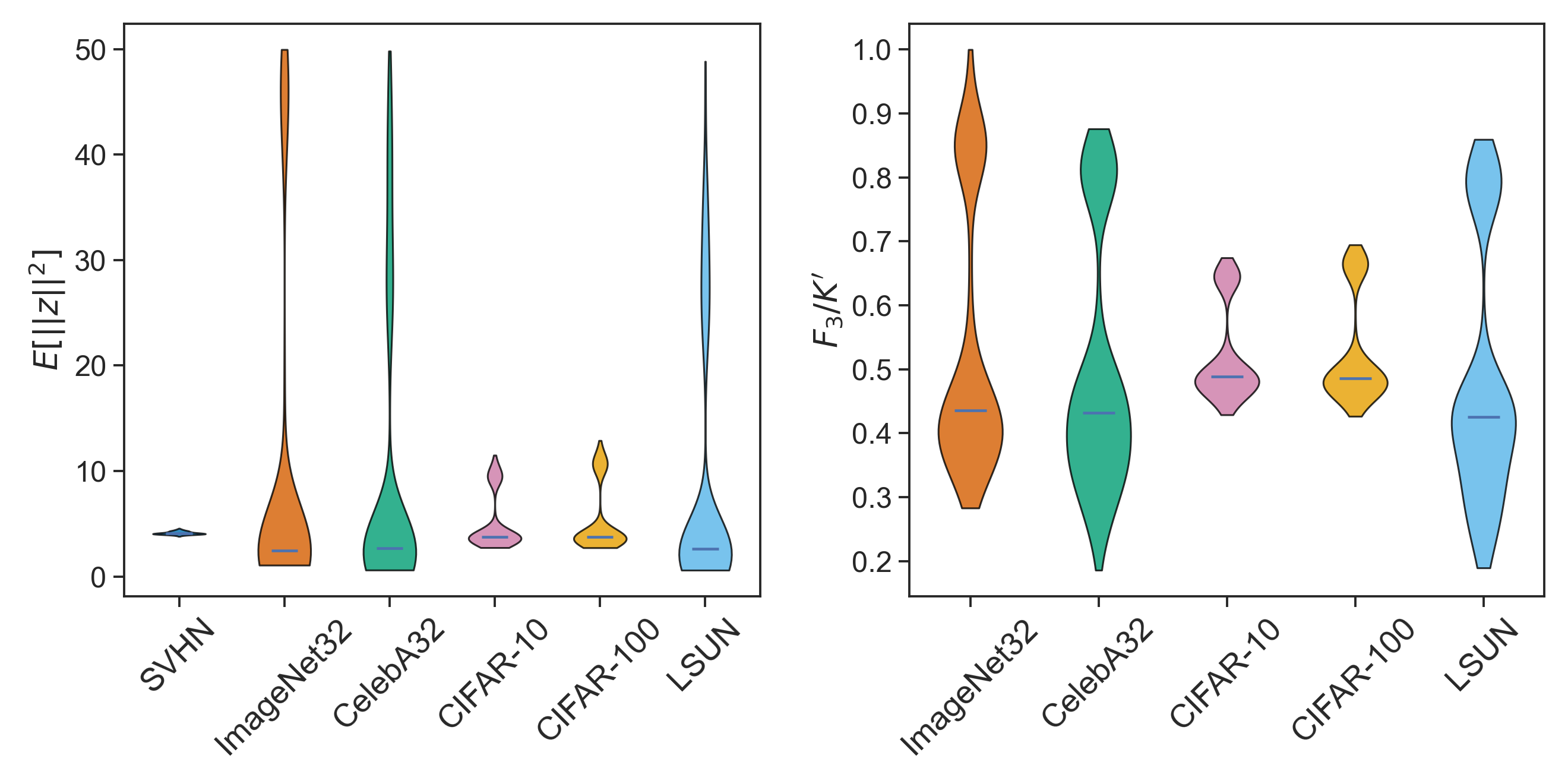}
				\vspace{-10pt}
				\label{fig:glow_svhn_merged_Ex2_F3_K'_ratio}
			\end{minipage}
		}
		
		\subfloat[\textit{{\scriptsize ID: CIFAR-10}}]{
			\begin{minipage}[t]{8cm}
				\centering
				\includegraphics[width=8cm]{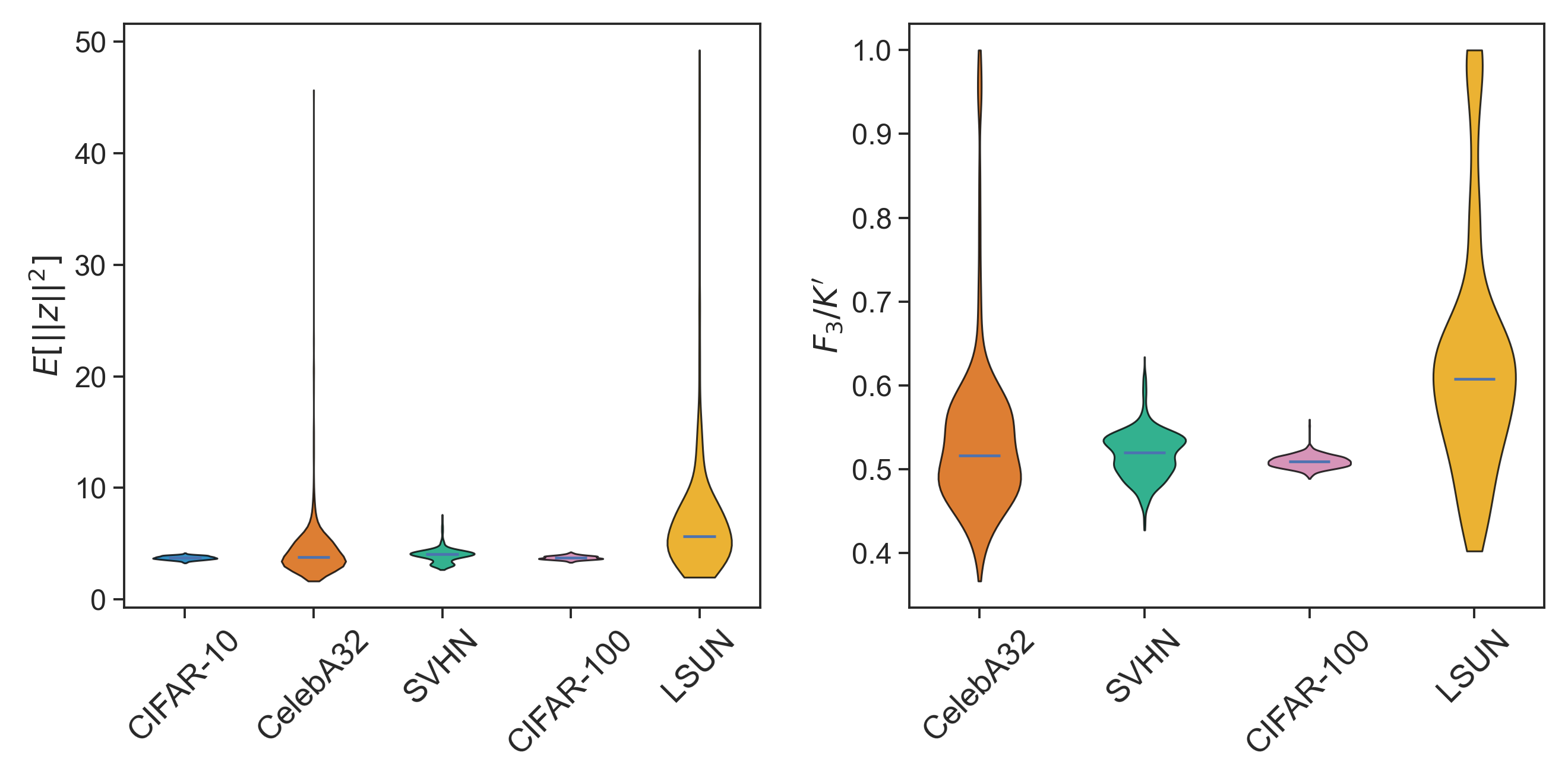}
				\vspace{-10pt}
				\label{fig:glow_cifar10_merged_Ex2_F3_K'_ratio}
			\end{minipage}
		}
		\caption{Estimated $\mathbb{E}[\|Z\|^2]$ and the {\color{black}ratio $F_3/K'$} on each dimension group when testing OOD datasets on Glow models.
			Groups with estimated second moments satisfying $ \mathbb{E}[\|Z\|^2]>50$ are excluded from the violin plots for readability.}
		\label{fig:glow_Ex2_F3_K'_ratio}
	\end{figure}

	\begin{figure}[t!]
		\scriptsize
		\centering
		\subfloat[\textit{{\scriptsize ID: SVHN}}]{
			\begin{minipage}[t]{5.5cm}
				\centering
				\includegraphics[width=5.5cm]{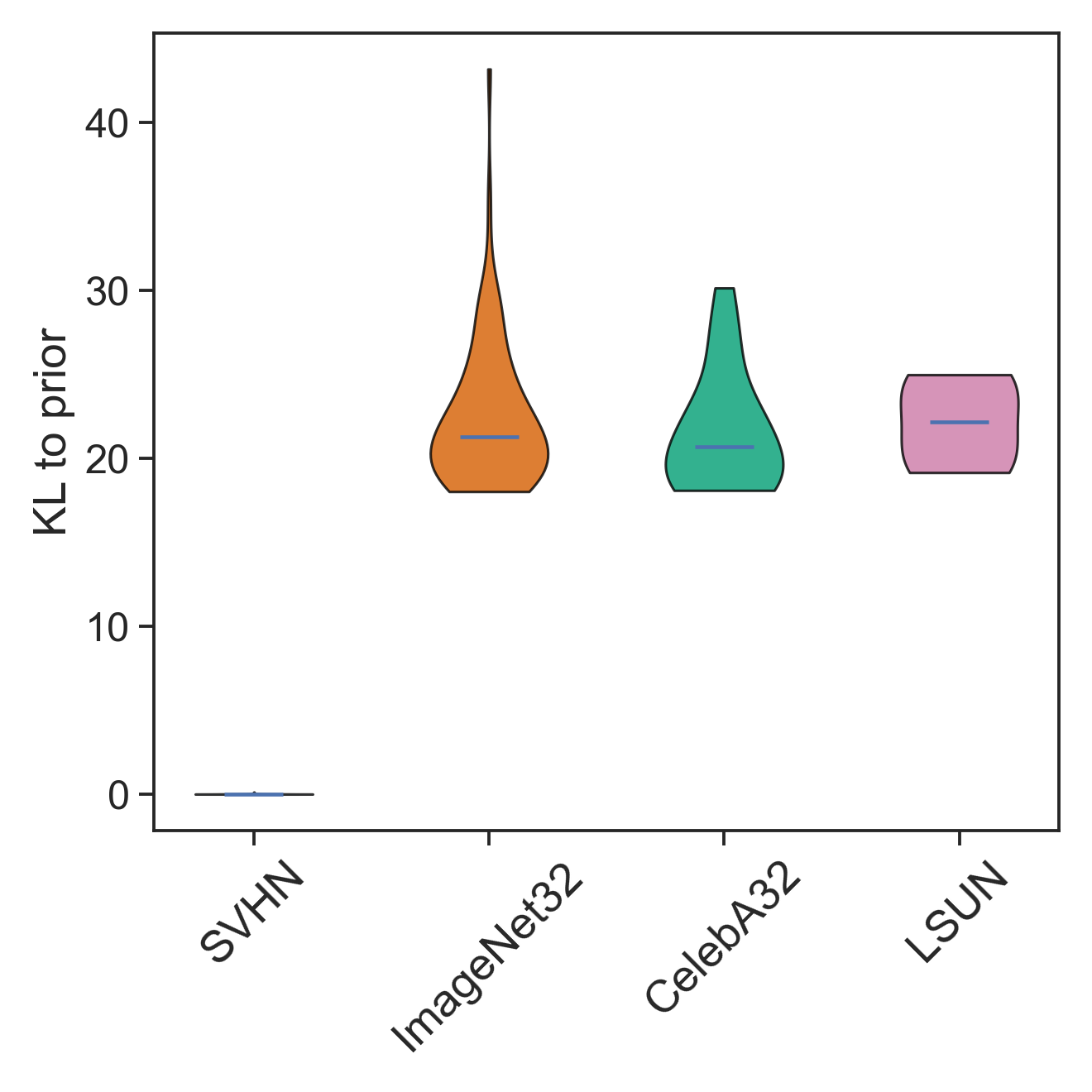}
				\vspace{-15pt}
				\label{fig:glow_celeba32_KL_OOD_Ez2_greater_than_50}
			\end{minipage}
		}
		\vspace{-5pt}
		\subfloat[\textit{{\scriptsize ID: CIFAR-10}}]{
			\begin{minipage}[t]{5.5cm}
				\centering
				\includegraphics[width=5.5cm]{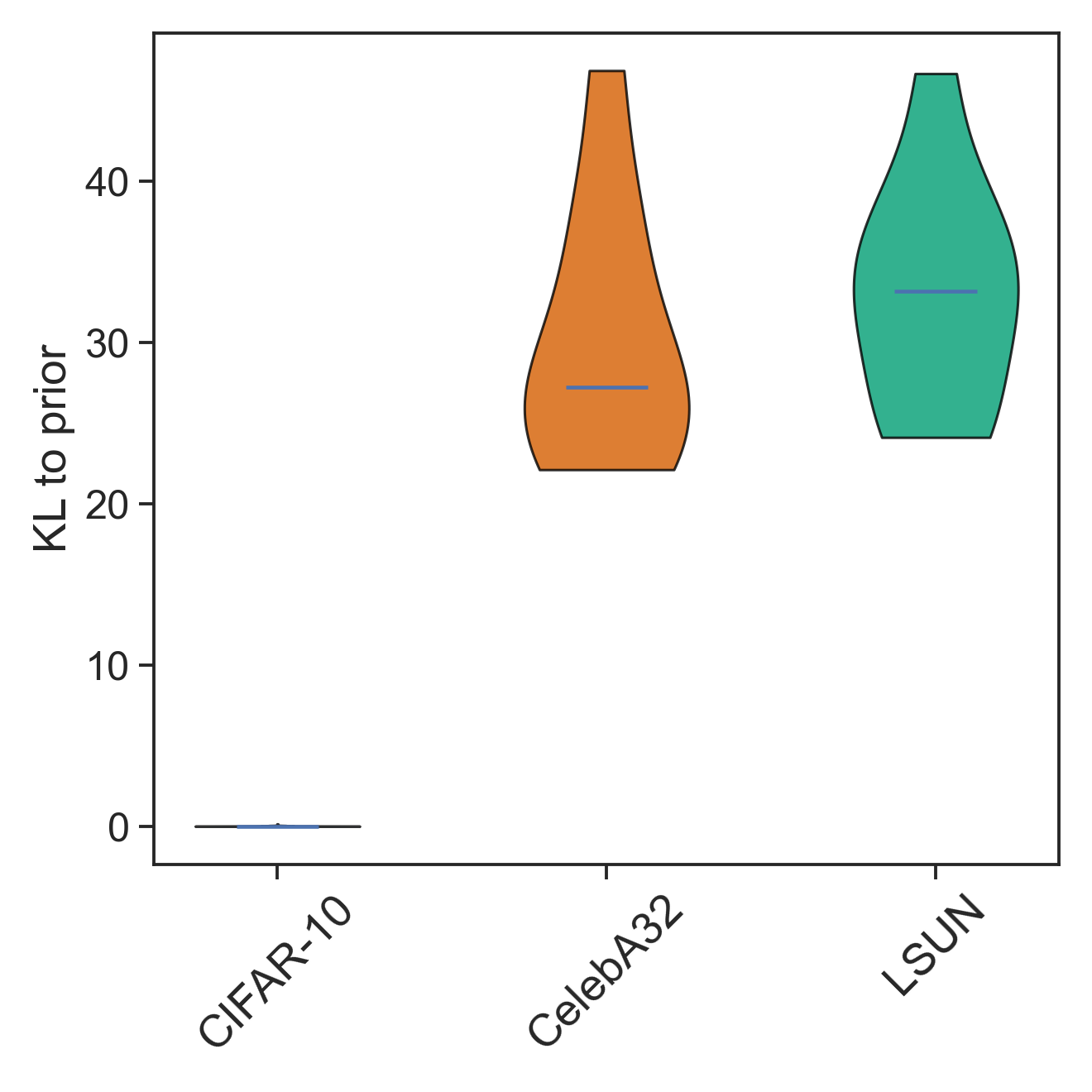}
				\vspace{-10pt}
				\label{fig:glow_svhn_KL_OOD_Ez2_greater_than_50}
			\end{minipage}
		}
		\caption{$\mathrm{KL}(\widehat{P}_Z^{(g)}\Vert \N(\mathbf{0},\mathbf{I}_4))$ for all groups of ID representations 
			and $\mathrm{KL}(\widehat{Q}_Z^{(g)}\Vert \N(\mathbf{0},\mathbf{I}_4))$ for groups of OOD representations satisfying $ \mathbb{E}[\|Z\|^2]>50$. 
			KL divergence values greater than $50$ are omitted to improve the scale of the visualization.}
		\label{fig:glow_KL_on_large_Ex2_groups}
	\end{figure}

	\textbf{Discussion on Dimension Partitioning}. 
	In high-dimensional flow-based models, the global divergence
	$\mathrm{KL}(P_Z\Vert\N(\mathbf{0},\mathbf{I}_d))$
	may exceed the threshold $1/12$ specified in Theorem~\ref{thm:triangle-P-N1-N2}. 
	We therefore examine the empirical applicability of the theorem at the group level by considering the marginal distribution of each group.
	In our experiments, we partition the latent dimensions into groups corresponding to $2\times 2$ blocks 
	because at least $98.6\%$ of these groups satisfy $\mathrm{KL}(\widehat{P}_Z^{(g)}\Vert \N(\mathbf{0},\mathbf{I}_4))<1/12$. 
	Other partitioning schemes, such as $2\times 3$ or $4\times 4$ blocks, may also be used.
	With $4\times 4$ blocks, the proportions of groups satisfying $\mathrm{KL}(\widehat{P}_Z^{(g)}\Vert \N(\mathbf{0},\mathbf{I}_{16}))<1/12$ 
	are $83.3\%$, $62.5\%$, and $54.7\%$ for models trained on SVHN, CelebA32, and CIFAR-10, respectively. 
	Despite these lower proportions, the grouping scheme can still be used in the previously developed OOD detection algorithm \cite{zhang2023outofdistribution}.
	Note that a more advanced flow model with a smaller model fitting error may improve the proportion of groups satisfying $\mathrm{KL}(\widehat{P}_Z^{(g)}\Vert \N(\mathbf{0},\mathbf{I}))<1/12$. We leave this exploration for future work.
	
	In general, a large global divergence $\mathrm{KL}(Q_Z\Vert P_Z)$ between the OOD and ID latent distributions 
	does not imply that their marginals are well separated in every group. Groupwise guarantees require sufficient ID--OOD separation 
	within the groups under consideration. For each such group that also satisfies the Gaussian-perturbation condition, 
	Theorem~\ref{thm:triangle-P-N1-N2} can be applied independently to obtain a lower bound on the KL divergence between 
	the corresponding OOD marginal and the groupwise prior.
	Thus, although the groupwise analysis does not establish a guarantee for the entire high-dimensional latent distribution, 
	the experimental results provide empirical support for the predicted stability behavior.

}%end color

\section{Discussion}\label{sec:appendix-discussion}

\subsection{Small Perturbation Condition}\label{sec:small-perturbation}
In high-dimensional flow-based models, $\mathrm{KL}(P_Z\Vert P_Z^r)$ is minimized 
but does not necessarily satisfy the perturbation condition $\mathrm{KL}(\widehat{P}_Z\Vert P_Z^r)\leq\e$ for $\e<1/12$.
Thus, we partition dimensions into groups and conduct experiments in each group separately.
This limitation may be mitigated by future advances in flow architectures and training methods that improve prior matching. 
Owing to computational resource constraints, we did not evaluate recently proposed large flow-based models~\cite{tarflow}.
Future work will investigate the empirical applicability of the theorem to such models and broader perturbation regimes.

\subsection{Computational Complexity}
The computational cost of evaluating the bound in Theorem~\ref{thm:triangle-P-N1-N2} (Corollary \ref{thm:cor-asymptotic-bound}) is dominated by the coefficient $K=F_1+F_2+F_3$, defined in Equation \eqref{eq:thmF123}.
Let $d$ denote the dimension, and let $M$ denote the number of samples drawn from $P$.
Computing the required operator norms, including $\|\Sigma_2\|_{\mathrm{op}}$, $\|\Sigma_2^{1/2}\|_{\mathrm{op}}$, and $\|\Sigma_i^{-1}\|_{\mathrm{op}}$ for $i=1,2$, requires the extreme eigenvalues of the covariance matrices.
Using standard eigendecomposition, these spectral computations require $O(d^3)$ time.
Computing the mean norms $\|\mu_1\|$ and $\|\mu_2\|$ requires $O(d)$ time.

The expectation terms in $F_2$ and $F_3$ are estimated using $M$ samples from $P$.
Because evaluating $\|x_i\|$ or $\|x_i\|^2$ for each sample requires $O(d)$ time, each empirical moment estimate has complexity $O(Md)$.
Thus, when the spectral computations are included, the costs associated with the three terms are
\[
T(F_1)=O(d^3), 
T(F_2)=O(d^3+Md), 
T(F_3)=O(d^3+Md).
\]
Thus, the total complexity of evaluating the bound is $O(d^3+Md)$.

For high-dimensional applications, exact eigendecomposition may be replaced by iterative eigensolvers, 
such as Lanczos methods~\cite{lanczos1950iteration}, to estimate the largest and smallest eigenvalues.
These methods can substantially reduce the computational cost.

\subsection{Other Potential Applications}\label{sec:potential-app}

\textbf{Implications for Robust Reinforcement Learning}.
Policy regularization is a central component in modern reinforcement learning (RL), especially in algorithms such as Trust Region Policy Optimization (TRPO) \cite{TRPO2015} and Proximal Policy Optimization (PPO) \cite{PPO2017}, where updates are constrained via KL divergence. In continuous control settings, policies are commonly modeled as multivariate Gaussian.
$$\pi(a|s)=\mathcal{N}(\mu_{\theta}(s),\Sigma_{\theta}(s)).$$
A typical update enforces
$\mathrm{KL}(\pi_{\mathrm{old}}\Vert \pi_{\mathrm{new}})\leq \e$,
where $\pi_{\mathrm{old}}$ denotes the old policy, and $\pi_{\mathrm{new}}$ denotes the updated policy.
Suppose $\pi_{\mathrm{old}}$ and $\pi_{\mathrm{new}}$ are both Gaussian policies. 
Such policy regularization ensures that the new policy does not deviate excessively from the previous one.

Let $P$ denote an arbitrary action distribution, \textit{e.g.}, induced by exploratory behavior, an off-policy dataset, or an adversarial perturbation. 
By Theorem \ref{thm:triangle-P-N1-N2}, we obtain $|\mathrm{KL}(P\Vert \pi_{\mathrm{old}})-\mathrm{KL}(P\Vert \pi_{\mathrm{new}})|\leq  \frac{\sqrt{6d\e}}{2} + K\sqrt{\e} + o(\sqrt{\e})$.

Suppose $\mathrm{KL}(P\Vert \pi_{\mathrm{old}})>C$, meaning $P$ is well-separated from the current policy. Then the above conclusion indicates that the separation is preserved up to a small degradation after the policy is updated with a small step.
Therefore, the KL separation remains after updates.

Furthermore, Corollary \ref{thm:multiple-steps-perturbation} can be used to extend the above guarantee to multiple update steps. 
Consider a sequence of Gaussian policies $\{\pi_i\}_{i=1}^k$, where each policy update satisfies
\[
\mathrm{KL}(\pi_t\Vert \pi_{t+1})\leq \epsilon, \quad t=1,\dots,k-1,
\]
where $\e\geq 0$ is sufficiently small such that $(k-1)^2\e+o(\e)<1/12$.
By applying Corollary \ref{thm:multiple-steps-perturbation}, we can obtain
\begin{align}
	&|\mathrm{KL}(P\Vert \pi_1)-\mathrm{KL}(P\Vert \pi_{k})|\\
	&\leq
	\frac{(k-1)\sqrt{6d\epsilon}}{2}+(k-1)K\sqrt{\epsilon}+o(\sqrt{\epsilon}).
\end{align}
This indicates that if an external distribution $P$ is initially well separated from the policy, \textit{i.e.},
$\mathrm{KL}(P\Vert \pi_1)>C$, 
then the KL separation is preserved after multiple constrained policy updates, as long as the accumulated degradation is sufficiently small. 
Therefore, KL-based policy regularization not only prevents large changes in a single update
but also provides a theoretical guarantee that repeated updates
preserve KL separation from the speciﬁed external action distribution.

{\color{black}
	\subsection{Connections with Stability Results for Related Divergences}\label{sec:appendix-stability-results}
	Stability properties have been studied for other divergences and discrepancy measures~\cite{endres2003new,villani2009optimal,ali1966general,cluster-bregman-div}. Existing results often relate KL divergence to such measures under additional assumptions, but our result does not follow directly from these inequalities. For example, Pinsker's inequality~\cite{DBLP:books/cu/CsiszarK11} bounds total variation distance in terms of KL divergence as
	\[
	\|P-Q\|_{\mathrm{TV}}
	\leq \sqrt{\frac{1}{2}\mathrm{KL}(P\Vert Q)}
	\]
	for any distributions $P$ and $Q$ on the same measurable space $(\mathcal{X},\mathcal{F})$. Here,
	$\|P-Q\|_{\mathrm{TV}}=\sup_{A\in\mathcal{F}}|P(A)-Q(A)|$.
	Theorem~\ref{thm:triangle-P-N1-N2} cannot be derived directly from Pinsker's inequality because total variation control alone does not generally bound expectations of unbounded log-density ratios. Instead, our proof exploits the quadratic structure of Gaussian log-likelihoods and controls perturbations of the means, covariance matrices, and log-determinant terms under the mild assumption that $P$ has a finite second moment.
	
	Transport inequalities, such as Talagrand's inequality, also relate KL divergence to Wasserstein distance when the reference distribution satisfies suitable regularity conditions~\cite{otto2000generalization}. These results connect KL divergence with other discrepancy measures, but they do not directly establish a triangle-type relation among KL divergences.
}

	\vfill
	
\end{document}